\documentclass[twoside,11pt]{article}

\usepackage{amsmath,amssymb,amsfonts,amsthm}
\usepackage{mathtools}
\usepackage{booktabs}
\usepackage{algorithm}
\usepackage{algorithmic}
\usepackage{empheq}
\usepackage{xcolor}

\usepackage{jmlr2e}
\usepackage{wrapfig}
\usepackage{needspace}
\usepackage{placeins}
\newsavebox{\admmAlgorithmBox}
\newsavebox{\rfsBranchesImageBox}

\newcommand{\R}{\mathbb{R}}
\newcommand{\N}{\mathbb{N}}
\renewcommand{\P}{\mathbb{P}}

\newcommand{\E}{\mathbb{E}}
\renewcommand{\d}{\mathrm{d}}

\newcommand{\rv}[1]{\mathbf{#1}}
\newcommand{\I}{\mathbb{I}}
\newcommand{\1}{\mathbf{1}}
\newcommand{\B}{\mathbb{B}}

\newcommand{\scriptB}{\mathcal{B}}
\newcommand{\scriptA}{\mathcal{A}}
\newcommand{\scriptC}{\mathcal{C}}
\newcommand{\scriptG}{\mathcal{G}}
\newcommand{\scriptT}{\mathcal{T}}
\newcommand{\scriptP}{\mathcal{P}}
\newcommand{\scriptM}{\mathcal{M}}
\newcommand{\scriptE}{\mathcal{E}}
\newcommand{\scriptF}{\mathcal{F}}

\newcommand{\scriptK}{\mathcal{K}}
\newcommand{\scriptN}{\mathcal{N}}
\newcommand{\scriptU}{\mathcal{U}}
\newcommand{\scriptI}{\mathcal{I}}
\newcommand{\scriptX}{\mathcal{X}}
\newcommand{\scriptL}{\mathcal{L}}
\newcommand{\scriptV}{\mathcal{V}}
\newcommand{\scriptJ}{\mathcal{J}}
\newcommand{\scriptH}{\mathcal{H}}

\newtheorem{prop}{Proposition}

\newtheorem{cor}{Corollary}

\newtheorem{defn}{Definition}
\newtheorem{rmk}{Remark}

\newcommand{\vct}[1]{\boldsymbol{#1}}

\newcommand{\josh}[1]{}

\colorlet{caseeditcolor}{red}
\DeclareRobustCommand{\caseedit}[1]{{\color{caseeditcolor}%
  \hypersetup{citecolor=caseeditcolor,linkcolor=caseeditcolor,urlcolor=caseeditcolor}#1}}

\DeclareRobustCommand{\joshrevision}[1]{{\color{red}%
  \hypersetup{citecolor=red,linkcolor=red,urlcolor=red}#1}}

\usepackage{lastpage}
\jmlrheading{[vol]}{2026}{1-\pageref{LastPage}}{xx/26; Revised xx/xx}{xx/xx}{21-0000}{Joshua Pilipovsky}

\ShortHeadings{Collaborative Autonomy for Heterogeneous UAV/UGV Teams}{Joshua Pilipovsky}
\firstpageno{1}

\makeatletter
\AtBeginDocument{%
  \colorlet{red}{black}%
  \colorlet{caseeditcolor}{black}%
  \hypersetup{allcolors=black}%
  \renewcommand{\josh}[1]{}%
  \g@addto@macro\ps@jmlrtps{%
    \let\@oddfoot\@empty
    \let\@evenfoot\@empty
  }%
}
\makeatother

\begin{document}

\title{Information-Driven Collaborative Autonomy for Heterogeneous UAV/UGV Teams}

\author{%
\name Joshua Y. Pilipovsky
\email joshua.y.pilipovsky@rtx.com \\
\addr Senior Research Engineer\\
RTX Technology Research Center (RTRC)\\
East Hartford, CT 06108, USA
}

\editor{xxx}

\maketitle

\begin{abstract}
\begingroup
This paper considers the problem of coordinating a heterogeneous team of unmanned aerial vehicles (UAVs) and unmanned ground vehicles (UGVs), where the ground vehicles must traverse a domain containing an unknown number of potentially hazardous targets.
The UAVs provide aerial sensing using onboard optical cameras, and their trajectories must adapt to both the ground routes and the information collected about the targets.
We formulate the aerial search problem using ergodic control, which distributes sensing effort according to a spatial information density to balance exploration and target localization.
The formulation incorporates dynamic camera field-of-view sensor footprints and probabilistic detection models that depend on vehicle position and attitude.
Angles-only measurements are processed through Bayesian inference and marginalization over data-association hypotheses to obtain a Poisson multi-Bernoulli (PMB) approximation of the target posterior, representing both detected and undetected targets.
This belief directs aerial sensing toward unexplored regions along the ground routes and viewpoints that improve target localization.
Ground planning combines chance-constrained target avoidance, enforced through confidence sets, with a void-probability objective that favors routes unlikely to contain threats.
We solve each team's optimization problem in a fully decentralized manner using a nested SCvx*--C-ADMM scheme, with outer sequential convex programming iterations and inner distributed consensus optimization, accounting for nonlinearities, nonseparable objectives, and continuous-time inter-agent separation constraints.
A closed-loop scheme updates both fleets' trajectories as information becomes available and permits ground motion when the remaining route risk falls below a prescribed threshold.
Numerical case studies examine the effects of sensing geometry and target uncertainty, compare decentralized and centralized solvers, and demonstrate full closed-loop collaboration between multiple UAVs and UGVs.
\endgroup
\end{abstract}


\section{Introduction}~\label{sec:intro}
\begingroup
\hypersetup{citecolor=red,linkcolor=red,urlcolor=red}
Collaborative autonomy deals with the problem of coordinating multiple autonomous vehicles to accomplish a common objective.
In a homogeneous team, the vehicles typically share similar motion and sensing capabilities, allowing coordination to be formulated around a common vehicle and sensor model.
Heterogeneous teams introduce an additional level of coupling, since the capabilities of one vehicle class may be needed for another to complete its task.
For example, unmanned aerial vehicles (UAVs) can observe regions that are inaccessible to unmanned ground vehicles (UGVs), while the ground vehicles can carry the payloads needed to execute a mission.
This complementary use of aerial and ground platforms has been studied in coordinated exploration and mapping~\citep{dewan2013coordinated}, collaborative localization~\citep{butzke2016collaborative}, and persistent sensing with mobile recharging~\citep{naveed2025adaptive}.

In this work, we consider a UAV/UGV team in which the ground vehicles must traverse a bounded domain containing an unknown number of potentially hazardous targets.
The UAVs provide aerial overwatch, using onboard optical cameras to discover, localize, and track these targets before the UGVs encounter them.
The difficulty is that neither fleet can plan effectively in isolation.
The value of a UAV observation depends on its relevance to the ground routes, while the admissibility of those routes depends on what the UAVs have observed.
Moreover, failing to detect a target does not imply that a region is empty, since the camera has a limited field of view and a non-unit probability of detection.
The planning problem must therefore account for both uncertainty in detected targets and the possibility of targets that have not yet been observed.

A natural approach to the aerial search problem is ergodic control, which seeks trajectories whose time-averaged spatial statistics match a prescribed distribution over the search domain.
Regions assigned greater probability mass receive proportionally more sensing effort, allowing a single objective to express both broad exploration and concentrated information gathering.
The work of \citet{mathew2011metrics} introduced multiscale metrics for measuring deviations from ergodicity and designing multi-agent coverage dynamics.
Subsequent works formulated continuous-time ergodic trajectory optimization~\citep{miller2013trajectory} and used expected information densities to guide active sensing with nonlinear measurement models~\citep{miller2016information}.
The commonly used spectral metric compares weighted Fourier coefficients of the trajectory occupation measure with those of the desired spatial distribution.
Alternatively, \citet{hughes2025ergodic} formulate ergodic trajectory optimization using maximum mean discrepancy between kernel mean embeddings in a reproducing kernel Hilbert space, allowing coverage objectives to be specified from samples on more general domains.

For a fleet of UAVs, the ergodic objective depends on the aggregate coverage of all agents and thus couples their trajectories through the team Fourier coefficients.
\citet{abraham2018measurement} address this coupling through consensus on the coverage statistics, obtaining decentralized ergodic control for agents with nonlinear dynamics.
The sensing model introduces a second difficulty: the point occupation measure of a vehicle does not describe the ground area observed by a camera whose footprint changes with altitude and attitude.
Finite sensor footprints and obstacle constraints have been incorporated into ergodic coverage by \citet{ayvali2017constrained}.
Here, we derive the elliptical ground footprint from the position and attitude of the camera and construct a smooth, normalized sensing kernel from this geometry.
Integrating this kernel along a trajectory gives an occupation measure that describes the spatial distribution of camera coverage and depends on both the translational and rotational motion of the UAV.
Its spectral coefficients enter the fleet ergodic objective, so that the trajectories are optimized with respect to the area observed by the cameras.
The same camera geometry also enters a spatially varying detection probability and an angles-only measurement model with pose-dependent covariance.
This connects the coverage model used in trajectory planning with the missed detections and measurement uncertainty accounted for in Bayesian inference.

The search distribution must, in turn, reflect the information collected by the UAVs.
Random finite set (RFS) theory provides a natural Bayesian description of this information, since both the number of targets and their states are uncertain~\citep{mahler2007statistical}.
The measurements are also a finite set, containing an unknown combination of target detections and clutter, with some targets producing no measurement.
Poisson multi-Bernoulli mixture (PMBM) filtering represents undetected targets by a Poisson point process and detected targets by Bernoulli components conditioned on global data-association hypotheses~\citep{garciafernandez2018pmbm}.
Marginalizing the association uncertainty provides tractable multi-Bernoulli approximations~\citep{williams2015marginal}.
These ideas have also been used for multi-agent planning that balances the discovery of new objects with the tracking of existing ones~\citep{nguyen2020multiobjective}.
Here, we specialize the measurement update to camera-based, angles-only observations with pose-dependent detection and measurement quality, and maintain a Poisson multi-Bernoulli (PMB) approximation of the target posterior.
This distinction between detected and undetected targets is particularly useful for ground planning, since avoiding known tracks alone does not account for threats in regions that have not been sufficiently explored.

Planning around the estimated target positions alone neglects this uncertainty.
Chance-constrained trajectory optimization instead limits the probability of violating a safety requirement~\citep{blackmore2011chance}, and confidence sets provide a geometric means of accounting for uncertain target locations.
For a multi-object belief, the corresponding question is whether an entire region contains any targets.
The RFS void probability, namely, the probability that a set contains no objects, has been used to impose safety constraints in sensor control~\citep{beard2017void}; more recently, \citet{han2026risk} have developed belief control barrier functions from a Poisson representation of multi-object uncertainty.
We use the PMB void probability over the corridor swept by a candidate UGV route, which accounts for both the residual undetected-target intensity and the existence and location uncertainty of detected tracks.
Its negative logarithm supplies a route-risk cost, while a separate threshold on the remaining fleet corridor determines when the UGVs may proceed under the current belief.
As sensing reduces uncertainty or reveals a threat, the ground routes and the aerial search distribution are updated accordingly.

Solving the resulting fleet trajectory problems at a centralized ground station requires aggregating all vehicle decision variables and constraints.
A decentralized formulation instead allows each vehicle to solve a local problem and exchange the quantities needed to coordinate with its neighbors.
Consensus alternating direction method of multipliers (C-ADMM) provides an established framework for this purpose when the local problems are convex~\citep{boyd2011admm}.
For trajectory planning, \citet{rey2018decentralized} combine decentralized ADMM with adaptive linearization of collision constraints for agents with linear dynamics.
For nonlinear model predictive control, \citet{stomberg2026realtime} develop a closely related nested scheme in which decentralized sequential quadratic programming uses ADMM to solve the coupled quadratic subproblems.
In our setting, nonlinear vehicle dynamics, camera footprints, and collision constraints likewise require an outer approximation method before the fleet problem can be distributed into convex local solves.

To this end, we use successive convexification (SCvx), which replaces the nonlinear optimal control problem by a sequence of convex subproblems, with virtual controls and trust regions to handle infeasible linearizations and control the approximation error~\citep{mao2016successive}.
The SCvx* extension embeds these iterations in an augmented Lagrangian framework to address feasibility of the original nonlinear problem~\citep{oguri2023scvxstar}.
We also adopt the continuous-time constraint formulation of \citet{elango2024ctcs}, in which integrated constraint violations are represented by auxiliary dynamics, to account for inter-agent separation between discretization nodes.
Within each SCvx* iteration, local copies and graph flows separate the shared collision constraints and the global ergodic coefficient sum, allowing the convex subproblem to be solved using neighbor-to-neighbor C-ADMM updates.

The main contributions of this work are as follows:
\begin{enumerate}
    \item We formulate multi-UAV ergodic trajectory optimization with optical sensor footprints that depend on the full camera pose.
    Starting from a conic field of view, we derive the position, orientation, and extent of the ground footprint and construct a smooth footprint kernel that defines the team occupation measure and its Fourier coefficients.
    The resulting ergodic objective measures coverage of the ground area observed by the cameras throughout their motion.
    \item We connect Bayesian inference from UAV angles-only measurements to both aerial and ground planning through a shared PMB target belief.
    The resulting formulations direct aerial sensing toward unexplored regions along the ground routes and informative views of detected targets.
    The same belief accounts for detected and undetected threats in ground planning through confidence-set avoidance and a route-risk cost based on the PMB void probability.
    \item \begingroup\emergencystretch=1em
    We develop a novel decentralized method for non-convex fleet trajectory optimization problems by combining outer SCvx iterations with inner C-ADMM solves.
    The method accommodates nonlinear dynamics, continuous-time inter-agent collision avoidance, and nonseparable team objectives.
    For the UAV problem, we additionally introduce a new graph-flow formulation of C-ADMM that replaces the global coefficient-sum constraint with local balance equations and neighbor agreement.
    We establish its equivalence at consensus to the centralized convex subproblem, enabling decentralized optimization of the coupled ergodic objective without a central coefficient aggregator.
    \par\endgroup
    \item We construct a closed-loop coordination scheme that updates aerial sensing plans as the target belief evolves, replans ground routes when newly confirmed threats affect them, and authorizes ground motion according to the remaining route risk.
    Numerical case studies assess UAV search through spatial coverage and ergodic mismatch, and UGV planning through route risk, control effort, and progress toward the goals.
    We also evaluate solver convergence, computational scaling, and solution quality relative to centralized optimization.
\end{enumerate}

The paper is organized as follows.
Sections~\ref{sec:notation} and~\ref{sec:prelim} introduce the notation and mathematical preliminaries.
Section~\ref{sec:UAV-traj-opt} formulates the UAV sensing and trajectory optimization problem and its decentralized solution.
Section~\ref{sec:rfs-inference} derives the RFS measurement update and PMB posterior approximation.
Section~\ref{sec:UGV-traj-opt} presents the UGV risk-aware trajectory optimization problem.
Section~\ref{sec:coordination} describes the feedback between the two fleets, followed by numerical case studies in Section~\ref{sec:examples} and concluding remarks in Section~\ref{sec:conclusion}.
\endgroup

\section{Notation}\label{sec:notation}
Real-valued vectors are lowercase $x \in \R^{d}$ and matrices are uppercase $X \in \R^{c\times d}$.
{The operator $\operatorname{col}(B_1,\ldots,B_q)\triangleq[B_1^\intercal,\ldots,B_q^\intercal]^\intercal$ denotes vertical stacking of vectors or matrix blocks with the same number of columns.
Indexed families are stacked in a fixed order.
For vector-valued functions, the operator acts pointwise.}
Random vectors are denoted by boldface $\vct x$.
{The notation $\scriptN(m,P)$ denotes a Gaussian distribution with mean $m$ and covariance $P$.
For $P\succ0$, $\scriptN(x;m,P)$ denotes its probability density evaluated at $x$.}
Denote by $\scriptM_{+}(\Omega)$ the cone of non-negative Borel measures over the support $\Omega\subset\R^{d}$ and $\scriptP(\Omega) = \{\mu \in \scriptM_{+}(\Omega) : \mu(\Omega) = 1\}$ the space of all probability measures.
We denote by $\delta_{x}$ the Dirac delta distribution concentrating unit mass at the atom $x \in \R^{d}$.
The disjoint union of two sets $\scriptA,\scriptB\subseteq\R^{d}$ is denoted by $\scriptA \sqcup \scriptB$.
Given a set $\scriptX \subseteq \R^{d}$, we denote by $2^{\scriptX} = \{S : S \subseteq \scriptX\}$ the power set of $\scriptX$.
The set $\mathcal{F}(\mathcal{X})$ represents the space of all finite subsets of $\mathcal{X}$.
More explicitly, $\scriptF(\scriptX) = \{S \subseteq \scriptX : |S| < \infty\}$, or, equivalently, $\scriptF(\scriptX) = \sqcup_{n=0}^{\infty} \scriptF_{n}(\scriptX)$, where $\scriptF_{n}(\scriptX) = \{S \subseteq \scriptX : |S| = n\}$.
Given an absolutely continuous probability distribution $\P$ supported on a set $\scriptX$, its density is denoted by $\rho(x)$ for $x \in \scriptX$.
{For a nonnegative function $g$ with $0<\int_\Omega g(z)\,\d z<\infty$, define the normalization operator $\operatorname{Norm}_\Omega[g](\xi)\triangleq g(\xi)/\int_\Omega g(z)\,\d z$.}
A Bernoulli random variable, denoted by $\vct E \sim \mathrm{Bernoulli}(p)$, satisfies $\P(\vct E = 1) = p$ and $\P(\vct E = 0) = 1-p$.
A Poisson random variable, denoted by $\vct E \sim \mathrm{Poisson}(\Lambda)$, satisfies $\P(\vct E = n) = e^{-\Lambda}\Lambda^{n} / n!$.
Given a domain $\Omega\subset\R^{d}$ and time horizon $[0,T]\subset\R_{+}$, denote by $L^p([0,T];\Omega)$ the (signal) space of functions with finite $L^p$ norm $\int_{0}^{T} |f(\tau)|^{p} \, \mathrm{d}\tau < \infty$. For an integer $s\geq0$, $W^{s,p}$ denotes the usual weak-derivative Sobolev space; fractional and negative orders below are understood through the standard spectral/interpolation definition and duality.
{For a scalar $a\in\R$, define its positive part by $[a]_+\triangleq\max\{0,a\}$.
For a vector $x\in\R^d$, $[x]_+\triangleq\max\{0,x\}\in\R^d$ applies this operation componentwise.}
An undirected communication graph is denoted by $\scriptG = (\scriptV, \scriptE)$, where $\scriptV = \{1,\ldots,M\}$ denotes the set of agents and $\scriptE \subseteq \{\{i,j\}\in\scriptV : i \neq j\}$ denotes the set of undirected communication links.
The characteristic function of a set $\scriptA\subset\R^d$ is $\1_{\scriptA}(a)=1$ if $a\in\scriptA$ and $0$ otherwise. Its extended-valued constraint indicator is $\I_{\scriptA}(a)=0$ if $a\in\scriptA$ and $+\infty$ otherwise.
For $r>0$, $\B(0,r)\triangleq\{x:\|x\|<r\}$ is the open Euclidean ball; $\scriptA\oplus\scriptB\triangleq\{a+b:a\in\scriptA,\ b\in\scriptB\}$ is the Minkowski sum, and $\scriptA\ominus\scriptB\triangleq\{x:x+\scriptB\subseteq\scriptA\}$ is the Pontryagin difference.
The cardinality of any finite set $\Pi$ is denoted by $|\Pi|$.
{For a nonempty set $\scriptA\subseteq\R^d$, the distance from $x\in\R^d$ to $\scriptA$ is $\operatorname{dist}(x,\scriptA)\triangleq\inf_{a\in\scriptA}\|x-a\|_2$.
When $\scriptA$ is also closed and convex, $\Pi_{\scriptA}(x)\triangleq\operatorname*{argmin}_{a\in\scriptA}\|x-a\|_2$ denotes its unique Euclidean projection.}
{For any positive integer $d$, $I_d$ and $0_d$ denote the $d\times d$ identity and zero matrices, respectively.
More generally, $0_{r\times c}$ is the $r\times c$ zero matrix; in particular, $0_{d\times 1}$ and $1_{d\times 1}$ denote the zero and one column vectors of length $d$.}

\section{Preliminaries}~\label{sec:prelim}
{This section introduces random finite set models for target inference in Section~\ref{sec:rfs-inference} and the spectral ergodic metric for UAV coverage in Section~\ref{sec:UAV-traj-opt}.
Sequential convex programming addresses the nonlinear UAV and UGV trajectory problems, an isoperimetric reformulation enforces continuous-time pairwise separation, and consensus ADMM distributes the resulting convex subproblems across each fleet.}
\subsection{Random Finite Set Theory}~\label{subsec:RFS}
When detecting objects%
\footnote{We will use the words `objects' and `targets' interchangeably throughout the text.}
in an unknown environment, there are often two sources of ambiguity regarding the objects, namely, both the \textit{estimated states} of detected targets and the \textit{number} (or cardinality) of the targets themselves.
Indeed, it is rare that one would ever know the true number of targets in a domain, aside from some exceptional prior reconnaissance or intelligence.
In this context, a Random Finite Set (RFS) is precisely the mathematical tool that can model uncertainty in both the cardinality of a set and the elements making up that set~\citep{mahler2007statistical}.
Below we give an overview of two classes of RFS which will be used in the sequel, namely the i) Multi-Bernoulli (MB) RFS, and the ii) Poisson Point Process (PPP).
\paragraph{Multi-Bernoulli RFS}
First, we start with a Bernoulli RFS which is a generalization of a standard Bernoulli random variable.
\begin{defn}[Bernoulli RFS]
    Let $\mathcal{X}\subset\mathbb{R}^{d}$ be the state space and let $\Upsilon$ be the sample space.
    An RFS $\rv X:\Upsilon\rightarrow\mathcal{F}(\mathcal{X})$ is a Bernoulli RFS with parameters $(p,\rho)$, where $p\in[0,1]$ is the \textit{existence} probability and $\rho$ is a probability density on $\mathcal{X}$, if $\rv X$ has the structure
    \begin{equation}
        \rv X = 
        \begin{cases}
            \varnothing, &\mathrm{w.p.} \ 1 - p, \\
            \{\rv x\},\, \rv x \sim \rho,  &\mathrm{w.p.} \ p.
        \end{cases}
    \end{equation}
\end{defn}
For brevity, we will often use the shorthand $\rv X \sim \mathrm{BernoulliRFS}(p,\rho)$.
In essence, a Bernoulli RFS generalizes a Bernoulli random variable by replacing the positive binary outcome (1 or true) with a random singleton $\{x\}$, thereby jointly modeling uncertainty in both existence and state.
As a result, the cardinality always satisfies $|\rv X| \in \{0,1\}$ and we can write the properties $\P(|\rv X| = 1) = p$ and $\rv x \mid |\rv X| = 1 \sim \rho$.
Furthermore, the complete density of the Bernoulli RFS is given by
\begin{equation}
    \rho(X) = 
    \begin{cases}
        1 - p, & X = \varnothing, \\
        p \rho(x), &X = \{x\}, \\
        0, &|X| \geq 2,
    \end{cases}
\end{equation}
where, in this context, a random finite set density $\rho(\cdot)$ is defined such that the set integral
\begin{equation*}
    \int \rho(X) \delta X = \rho(\varnothing) + \sum_{n=1}^{\infty} \frac{1}{n!} \int_{\mathcal{X}^{n}} \rho(\{x_1,\ldots,x_n\})\, \d x_1 \cdots \d x_n = 1.
\end{equation*}
This would suffice to characterize a single target in the domain.
However, to account for multiple detected targets, we generalize the Bernoulli construction to multiple independent Bernoulli events.
\begin{defn}[Multi-Bernoulli RFS]
    Let the single-object state space be $\mathcal{X}$ and consider $M$ independent Bernoulli RFSs $\rv X^{(i)} \sim \mathrm{BernoulliRFS}(r_i, \rho_i)$ for $i = 1,\ldots, M$, where $r_i \in [0,1]$ and $\rho_i$ is a probability density on $\mathcal X$.
    The random finite set 
    \begin{equation}
        \rv X = \bigcup_{i=1}^{M} \rv X^{(i)}
    \end{equation}
    is a multi-Bernoulli (MB) RFS with parameters $\{(r_i, \rho_i)\}_{i=1}^{M}$.
\end{defn}
Thus, it follows that the cardinality $|\rv X| \in \{0,\ldots, M\}$ may be at most $M$ and the discrete cardinality distribution is
\begin{equation}
    \P(|\rv X| = n) = \sum_{\scriptI \subseteq \{1,\ldots,M\} : |\scriptI| = n} \left[ \prod_{i\in \scriptI} r_i \prod_{j \notin \scriptI} (1 - r_j)\right].
\end{equation}
Importantly, we have to sum over all index subsets of cardinality $n$ in the above expression.
The corresponding density is given by
\begin{equation}
    \rho(X) = \sum_{\bigsqcup_{i=1}^{M} X^{(i)} = X} \, \, \prod_{i=1}^{M} \rho_i\big(X^{(i)}\big),
\end{equation}
where $\rho_i(\cdot)$ represent the individual Bernoulli RFS densities.
In the sequel, the MB RFS will be used to represent the distribution of \textit{detected} targets in the domain and will be updated whenever a probabilistic positive-detection event takes place.
\paragraph{Poisson Point Process}
A Poisson Point Process (PPP) is an RFS in which the number of objects is Poisson distributed, and conditional on that cardinality, the object states are i.i.d. according to a density obtained by normalizing an intensity function.
In particular, we first consider the general Poisson RFS below.
\begin{defn}
    Let $\scriptX\subset\R^{d}$ be the state space and let $\Upsilon$ be the sample space.
    An RFS $\rv X : \Upsilon \rightarrow \scriptF(\scriptX)$ is a Poisson RFS with parameters $(\Lambda, \P)$, where $\Lambda \geq 0$ is an intensity parameter and $\P$ is a probability distribution on $\scriptX$, if $\rv X$ satisfies the following properties:
    \begin{itemize}
        \item The cardinality follows the discrete Poisson distribution $|\rv X| \sim \mathrm{Poisson}(\Lambda)$.
        \item Conditional on a cardinality $N = n$, the individual elements are distributed according to $\rv x_1, \ldots, \rv x_n \overset{\text{i.i.d.}}{\sim} \P$.
    \end{itemize}
\end{defn}
Thus, the number of elements in the set is governed by the rate parameter intensity $\Lambda$, from which it follows that $\E[|\rv X|] = \Lambda$; and each individual element is distributed according to $\P$.
The PPP specializes the above definition to a structure where $\Lambda$ and $\P$ are fundamentally related by a \textit{rate parameter} as shown below.
\begin{defn}
    Let $\lambda: \scriptX \rightarrow \R_{+}$ denote the Poisson rate parameter.
    If $\Lambda\triangleq\int_{\scriptX}\lambda(x)\,\d x>0$, a Poisson Point Process is a Poisson RFS with intensity and conditional distributions given by
    \begin{equation}
        \Lambda = \int_{\scriptX} \lambda(x)\, \d x, \quad \rho(x) = \frac{\lambda(x)}{\Lambda}.
    \end{equation}
\end{defn}
As a shorthand, we write $\rv X \sim \mathrm{PPP}(\lambda)$; if $\Lambda=0$, then $\rv X=\varnothing$ almost surely.
It follows that, given a subset $\scriptA \subseteq \scriptX$, the number of targets in that domain is given by $|\rv X \cap \scriptA| \sim \mathrm{Poisson}\big(\Lambda(\scriptA)\big)$, where $\Lambda(\scriptA) \triangleq \int_{\scriptA} \lambda(x)\,\d x$.
The density of a PPP is given by
\begin{equation}
    \rho(X) = \exp(-\Lambda) \prod_{x \in X} \lambda(x).
\end{equation}
In the sequel, the PPP will be used to represent the distribution of \textit{unknown} targets in the domain and will be updated at every planning stage according to the previous stage's aerial sensing of the domain.

In summary, Table~\ref{tab:rfs-comparison} illustrates the three RFS objects used in this work in terms of their number, or cardinality, distribution, and individual (conditional) component distributions.

\begin{table}[t]
    \centering
    \caption{Common random finite set models on a state space $\scriptX$.}
    \label{tab:rfs-comparison}
    \renewcommand{\arraystretch}{1.4}
    \setlength{\tabcolsep}{10pt}
    \begin{tabular}{lll}
        \toprule
        \textbf{RFS Model}
        & \textbf{Cardinality Distribution}
        & \textbf{Spatial Distribution} \\
        \midrule

        Bernoulli
        &
        $\rv N \sim \operatorname{Bernoulli}(p)$
        &
        $\rv x \mid \rv N=1 \sim \rho(x)$
        \\

        Multi-Bernoulli
        &
        $\rv N = \displaystyle\sum_{i=1}^{M} \rv E_i,
        \ \rv E_i \sim \operatorname{Bernoulli}(r_i)$
        &
        $\rv x_i \mid \rv E_i=1 \sim \rho_i(x)$
        \\

        Poisson
        &
        $\rv N \sim \operatorname{Poisson}(\Lambda)$
        &
        $\rv x_1,\ldots,\rv x_N \mid \rv N
        \overset{\mathrm{i.i.d.}}{\sim}
        \rho(x)=\dfrac{\lambda(x)}{\Lambda}$
        \\

        \bottomrule
    \end{tabular}
\end{table}

\subsection{Spectral Ergodic Control}~\label{subsec:EC}
Ergodic control refers to the formulation and solution of optimization problems that drive a dynamical system towards ergodicity.
In the context of multi-agent collaborative search, there will be one shared base distribution over the domain which the UAVs will jointly cover.
Asymptotically, the fleet should spend more time in regions of higher (prescribed) probability density than in regions of lower density.
Having a perfect balance between the time-varying statistics of the trajectory distribution, and the underlying spatial distribution of which to search, is in essence an ergodic system.
Below we formalize this intuition.
\begin{defn}[Empirical Trajectory Distribution]
    Let the state space be $\scriptX\subset\R^n$, let $\Omega \subset \R^{d}$ be compact, and let $g:\scriptX\to\Omega$ be the sensing-location map.
    For an admissible input set $\scriptU\subset\R^m$, a dynamical system defined by the trajectory manifold
    \begin{equation}
        \scriptT = \left\{(x,u)\in W^{1,p}([0,T];\R^n)\times L^p([0,T];\R^m):
        \begin{array}{l}
        x(t)\in\scriptX,\ u(t)\in\scriptU\ \text{a.e.},\\[-1mm]
        \dot x(t)=f\big(x(t),u(t)\big)\ \text{a.e. on }[0,T]
        \end{array}\right\}
    \end{equation}
    admits the time-averaged, empirical occupation (trajectory) measure
    \begin{equation}
        \mu_{T}^{\scriptT} = \frac{1}{T}\int_{0}^{T} \delta_{g(x(\tau))} \, \d\tau \in \scriptP(\Omega). \label{eq:occupation-measure}
    \end{equation}
    For multiple independent dynamical systems $\scriptT_i$ over $i = 1,\ldots, M$, the occupation measure is 
    \begin{equation}
        \mu_{T}^{\scriptT} = \frac{1}{MT}\sum_{i=1}^{M} \int_{0}^{T} \delta_{g(x_i(\tau))} \ \d\tau,
    \end{equation}
    where $\scriptT \triangleq \scriptT_1 \times \cdots \times \scriptT_{M}$ is the multi-agent trajectory manifold.
\end{defn}
\begin{defn}
    Let $\Omega \subset \R^{d}$ be compact.
    Let $\hat\mu \in \scriptP(\Omega)$ define a spatial measure on $\Omega$.
    For $M > 0$ trajectories $x_i : [0,T] \rightarrow \scriptX$, let $\mu_{T}^{\scriptT}$ denote the joint occupation measure on $\Omega$.
    Then, the multi-agent dynamical system is ergodic with respect to $\hat\mu$ if the occupation measure converges weakly to $\hat\mu$, that is,
    \begin{equation}
        \mu_{T}^{\scriptT} \rightharpoonup \hat\mu \quad \mathrm{as} \quad T \rightarrow\infty.
    \end{equation}
\end{defn}
To quantify a ``distance" to ergodicity in the form of an objective function to minimize, it is natural to consider ergodic metrics of the form 
\begin{equation}
    \scriptE_{\hat\mu} = \left\| \psi[\mu_T^{\scriptT}] - \psi[\hat\mu] \right\|_{\psi},
\end{equation}
where $\psi$ represents a measure transform, and $\|\cdot\|_{\psi}$ an associated norm in the transformed space.
In this work, we utilize the \textit{spectral}, or Fourier, ergodic metric; for other formulations and their benefits and drawbacks, we refer the interested reader to~\citep{mathew2011metrics,miller2013trajectory,abraham2018measurement}.
The spectral ergodic metric takes the squared norm between the trajectory measure and the spatial measure in the dual Sobolev space $H^{-s}\equiv W^{-s,2}$ of fractional order $s=(d+1)/2$, and is given by
\begin{equation}
    \scriptE_{\hat\mu} = \|\mu_T^{\scriptT} - \hat\mu\|_{H^{-s}}^2 = \sum_{k \in \N_0^{d}} \Lambda_k (c_k - \phi_k)^2,
    \quad \Lambda_k \triangleq (1 + \lambda_k)^{-s},
    \quad \lambda_k\triangleq\sum_{i=1}^d\left(\frac{k_i\pi}{L_i}\right)^2,
\end{equation}
where
\begin{equation}
    c_k = \E_{\mu_{T}^{\scriptT}}[F_k] = \frac{1}{MT}\sum_{i=1}^{M}\int_{0}^{T} F_k\Big(g\big(x_i(\tau)\big)\Big)\, \d\tau, \quad \phi_k = \E_{\hat\mu}[F_k] = \int_{\Omega} F_k(\xi)\, \hat\mu(\d\xi).
\end{equation}
are the Fourier coefficients corresponding to $\mu_{T}^{\scriptT}$ and $\hat\mu$, respectively, and $F_k$ are the corresponding Fourier basis functions on the domain $\Omega$.
Over general manifolds $\Omega \subset \R^{d}$, it is a non-trivial task to determine the basis functions $F_k$, which are generally eigenfunctions of the negative Laplace--Beltrami operator~\citep{chavel1984eigenvalues} and satisfy $-\Delta F_k=\lambda_k F_k$ with $\lambda_k\geq0$.
In our context, we assume a rectangular domain of the form $\Omega = [0, L_1] \times \cdots \times [0, L_d]$, where each $L_i > 0$.
In this Euclidean space, the eigenfunctions correspond to the standard Laplacian, and the Fourier basis functions are exactly represented as the orthonormal set
\begin{equation}\label{eq:fourier-basis}
    F_k(x) = \frac{1}{h_k} \prod_{i=1}^{d} \cos\left(\frac{k_i \pi x_i}{L_i} \right), \quad k\in\N_0^d, \quad h_k^2 = \prod_{i=1}^{d}
    \begin{cases}
        L_i, &k_i = 0, \\
        L_i/2, &k_i > 0
    \end{cases}.
\end{equation}
Intuitively, the space $H^{-s}$ represents generalized spatial distributions whose high-frequency Fourier content is increasingly discounted.
As evidenced by the form of $\Lambda_k$, higher frequencies have less associated weight.
In turn, minimizing the spectral ergodic metric will prioritize driving the large-scale, low-frequency structure of the two distributions together, which results in smoother and more global coverage.

\subsection{Sequential Convex Programming}~\label{subsec:SCvx*}
Successive convexification, or sequential convex programming (SCP), refers to an optimization method to find local minima of nonlinear and non-convex programs, by iteratively linearizing the problem around a given reference, solving the linearized (convex) sub-problem, using that solution as the next reference, and iterating until convergence.
SCP is a rich and well-studied topic and has been successfully applied to numerous problems in guidance and control.
For a good overview, we refer the reader to~\citep{mao2016successive,oguri2023scvxstar}.
In this particular work, we use a variant of SCP, called SCvx*, to solve the nonlinear UAV and UGV optimization problems. Under the regularity, penalty-update, and subproblem-accuracy assumptions stated in~\citep{oguri2023scvxstar}, SCvx* provides feasibility and local convergence guarantees; whether those assumptions hold for each specialized team problem must be verified separately, and a finite numerical run yields only a candidate solution subject to the chosen stopping tolerances.
Below we outline the high-level elements of the algorithm, and defer the full treatment of SCvx* to~\citep{oguri2023scvxstar}.
Consider the general nonlinear program
\begin{equation}
    \begin{aligned}
        &\min_{z} &&J(z), \\ 
        &\ \, \mathrm{s.t.} &&g_i(z) = 0, \quad i = 1,\ldots, p, \\ 
        &&&h_j(z) \leq 0, \quad j = 1,\ldots,q, \\ 
        &&&g_{\mathrm{aff}}(z) = 0, \quad h_{\mathrm{cvx}}(z) \leq 0.
    \end{aligned}
\end{equation}
where $J \in C^1$ is assumed convex, $g_{\mathrm{aff}}$ and $g_i$ denote the affine and non-convex equality constraints, and $h_{\mathrm{cvx}}$ and $h_j$ denote the convex and non-convex inequality constraints.
The Lagrangian associated with the optimization problem (above) is
\begin{equation}
    \mathcal{L}(z, \lambda, \mu) = J(z) + \lambda^\intercal g(z) + \mu^\intercal h(z).
\end{equation}
Furthermore, the \textit{augmented} Lagrangian additionally adds the terms
\begin{equation}
    \mathcal{L}_{w}(z,\lambda,\mu) = J(z)+\lambda^\intercal g(z)+\frac{w}{2}\|g(z)\|^2+\mu^\intercal[h(z)]_+ + \frac{w}{2}\|[h(z)]_+\|^2,
\end{equation}
{Here $\mu\geq0$.}
By defining the total \textit{penalty} function as
\begin{equation}
    P(\cdot) = \lambda^\intercal g + \frac{w}{2}\|g\|^2 + \mu^\intercal \max\{0,h\} + \frac{w}{2}\|\max\{0,h\}\|^2,    
\end{equation}
the (non-convex) penalty problem that SCvx* repeatedly solves is given by
\begin{equation}
    \begin{aligned}
        &\min_{z} &&J(z) + P\big(g(z), h(z); w, \lambda, \mu), \\ 
        &\ \, \mathrm{s.t.} &&g_{\mathrm{aff}}(z) = 0, \quad h_{\mathrm{cvx}}(z) \leq 0.
    \end{aligned}
\end{equation}
Each iteration linearizes the non-convex equality and inequality constraints around a reference $\hat{z}$ as
\begin{equation*}
    \begin{aligned}
        \tilde{g}(z) &= g(\hat{z}) + \nabla g(\hat{z})^\intercal (z - \hat{z}), \\ 
        \tilde{h}(z) &= h(\hat{z}) + \nabla h(\hat{z})^\intercal (z - \hat{z}).
    \end{aligned}
\end{equation*}
One common defect of SCP methods is that their convex sub-problems suffer from artificial \textit{infeasibility} and \textit{unboundedness} due to the linearizations (and in the latter case, constraint relaxations).
To this end, instead of imposing $\tilde{g} = 0$ and $\tilde{h} \leq 0$, we relax the convex constraints to $\tilde{g}(z) = \xi$ and $\tilde{h}(z) \leq \zeta$ for some slack variables, $\xi \in \R^{p}, \zeta \in \R_{+}^{q}$, and penalize these in the objective function.
Furthermore, we bound the deviation of the decision variable $z$ from its reference through the trust region $\|z - \hat{z}\|_{\infty} \leq \Delta$, for some trust region size $\Delta > 0$.
In total, at iteration $\ell$, the convex sub-problem that is solved is given by
\begin{equation}
    \begin{aligned}
        &\min_{z,\xi,\zeta} &&L(z,\xi,\zeta) \triangleq J(z) + (\lambda^\ell)^\intercal\xi+\frac{w^\ell}{2}\|\xi\|^2+(\mu^\ell)^\intercal\zeta+\frac{w^\ell}{2}\|\zeta\|^2, \\ 
        &\ \, \mathrm{s.t.} &&g_{\mathrm{aff}}(z) = 0, \quad h_{\mathrm{cvx}}(z) \leq 0, \\[-1mm]
        &&&\tilde g(z)=\xi, \quad \tilde h(z)\leq\zeta, \quad \zeta\geq0, \quad \|z-\hat z\|_\infty\leq\Delta.
    \end{aligned}
\end{equation}
Solutions are accepted according to the ratio between their nonlinear-to-linear cost improvements, and the Lagrange multipliers $(\lambda,\mu)$, trust region radius $\Delta$, and penalty parameter $w$ are updated according to SCvx*.
For details on how these parameters are updated and the resulting convergence guarantees, please see~\citep{oguri2023scvxstar}; also see~\citep{mao2016successive,elango2024ctcs} for related SCvx methods.

\subsection{Continuous-Time Pairwise Separation}\label{subsec:ctcs-separation}
Consider a pair of agents with continuous position trajectories $s_i(t)$ and $s_j(t)$ and required separation $d_{ij}>0$.
The geometric path constraint and its equivalent normalized residual are
\begin{equation}\label{eq:pair-separation}
    \|s_i(t)-s_j(t)\|_2\geq d_{ij}
    \quad\Longleftrightarrow\quad
    g_{ij}(t)\triangleq 1-\frac{\|s_i(t)-s_j(t)\|_2^2}{d_{ij}^2}\leq0,
    \qquad t\in[0,T].
\end{equation}
Following the isoperimetric reformulation of~\citet{elango2024ctcs}, introduce the scalar violation accumulator
\begin{equation}\label{eq:pair-violation-accumulator}
    \dot y_{ij}(t)=[g_{ij}(t)]_+^2,
    \qquad y_{ij}(0)=0,
    \qquad y_{ij}(T)=0,
\end{equation}
Integrating the auxiliary dynamics and using both endpoint conditions gives
\begin{equation*}
    0=y_{ij}(T)-y_{ij}(0)=\int_0^T [g_{ij}(t)]_+^2\,\d t.
\end{equation*}
The integrand $[g_{ij}(t)]_+^2$ is nonnegative and continuous; if it were positive at any time, continuity would make its integral positive over a neighboring interval, contradicting the zero integral.
The integral therefore vanishes if and only if $g_{ij}(t)\leq0$ throughout the horizon.
Thus, it follows that \eqref{eq:pair-violation-accumulator} is equivalent to \eqref{eq:pair-separation}, and allows for continuous-time constraint satisfaction through enforcement of a corresponding (boundary-constrained) ODE.

\subsection{Decentralized Consensus-ADMM}~\label{subsec:C-ADMM}
In this work, we consider trajectory optimization problems over \textit{teams}, or fleets, of agents within a particular class (e.g., UAV or UGV).
The resulting discretized mathematical programs will necessarily involve the joint decision variables of all agents, say $z \triangleq [z_1^\intercal, \ldots, z_M^\intercal]^\intercal$ for $M$ total agents in the fleet. 
Furthermore, there may be equality, or inequality, constraints linking the decision variables of the agents together; for example, as in the context of inter-agent collision avoidance, or in the formulation of the ergodic metric.
As such, it is not possible to independently solve for each agent's decision vector, and hence one must either choose an architecture which uses a centralized node to solve the entire program, or otherwise distribute the optimization problem so that each agent solves its own decentralized \textit{local} problem, and communicates with its neighbors in order to reach \textit{consensus} over successive iterations.
We will explore this latter method presently, and use the consensus ADMM algorithm for efficient and convergent distributed optimization at each outer SCvx* iteration.
Below we give an overview of the method, and defer readers to a more thorough treatment in~\citep{boyd2011admm}.
\paragraph{General C-ADMM}
The Alternating Direction Method of Multipliers (ADMM) is a first-order optimization method for solving linearly constrained problems of the form
\begin{equation}\label{eq:admm-problem}
    \min_{z,w}\ f(z) + g(w) \quad \mathrm{s.t.} \quad Az + Bw = c,
\end{equation}
where $z$ and $w$ are blocks of decision variables, often introduced so that the terms $f$ and $g$ can be treated separately.
ADMM solves \eqref{eq:admm-problem} by seeking a saddle point of the augmented Lagrangian (AL)
\begin{equation}\label{eq:admm-augmented-lagrangian}
    \mathcal{L}_{\rho}(z,w,\lambda) = f(z) + g(w) + \lambda^\intercal (Az + Bw - c) + \frac{\rho}{2}\|Az + Bw - c\|_{2}^{2},
\end{equation}
where $\lambda$ has the same dimension as the residual $Az+Bw-c$, and $\rho > 0$ is a user-chosen penalty parameter.
In the consensus form considered here, the constraint is simply $z - w = 0$, which corresponds to the special case $A = I, B = -I$, and $c = 0$ in \eqref{eq:admm-problem}.
Defining the scaled dual variable $u \triangleq \rho^{-1} \lambda$, the AL becomes
\begin{equation}\label{eq:admm-scaled-lagrangian}
    \scriptL_{\rho}(z,w,u) = f(z) + g(w) + \frac{\rho}{2}\|z - w + u\|_{2}^{2} - \frac{\rho}{2}\|u\|_{2}^{2}.
\end{equation}
This scaled form yields particularly simple ADMM iterations, since the quadratic penalty acts directly on the disagreement between the local variable $z$ and the consensus/copy variable $w$~\citep{boyd2011admm}.

\begin{lrbox}{\admmAlgorithmBox}
\begin{minipage}{0.45\textwidth}
\begin{algorithm}[H]
\small
\setlength{\abovedisplayskip}{4pt}
\setlength{\belowdisplayskip}{4pt}
\setlength{\abovedisplayshortskip}{2pt}
\setlength{\belowdisplayshortskip}{2pt}
\caption{{Consensus ADMM}}
\label{alg:admm}
\begin{algorithmic}[1]
\STATE \textbf{Input:} \(\epsilon_{\mathrm{p}},\, \epsilon_{\mathrm{d}},\, \rho,\, k^{\max},\, z^{(0)},\, w^{(0)},\, u^{(0)}\)
\STATE \textbf{Output:} \(z^\star\)
\STATE \(k \gets 0,\quad z^\star\gets z^{(0)}\)
\WHILE{$k < k^{\max}$}
    \STATE Primal update:
    \begin{equation}~\label{eq:primal-update}
        \begin{aligned}
            &z^{(k+1)} \gets \arg\min_{z}\; \Bigl\{f(z)\\
            &\quad\qquad {}+ \frac{\rho}{2}\|z - w^{(k)} + u^{(k)}\|_{2}^{2}\Bigr\}
        \end{aligned}
    \end{equation}
    \STATE Consensus update:
    \begin{equation}~\label{eq:consensus-update}
        \begin{aligned}
            &w^{(k+1)} \gets \arg\min_{w}\; \Bigl\{g(w)\\
            &\quad\quad {}+ \frac{\rho}{2}\|z^{(k+1)} - w + u^{(k)}\|_{2}^{2}\Bigr\}
        \end{aligned}
    \end{equation}
    \STATE Dual update:
    \begin{equation}~\label{eq:dual-update}
        u^{(k+1)} \gets u^{(k)} + z^{(k+1)} - w^{(k+1)}
    \end{equation}
    \STATE Compute primal and dual residuals:
    \begin{equation}~\label{residual-update}
        \begin{aligned}
            r^{(k+1)} &\gets z^{(k+1)} - w^{(k+1)}, \\
            s^{(k+1)} &\gets \rho\bigl(w^{(k+1)} - w^{(k)}\bigr).
        \end{aligned}
    \end{equation}
    \IF{$\|r^{(k+1)}\|_{2} \le \epsilon_{\mathrm{p}}$ \textbf{and} $\|s^{(k+1)}\|_{2} \le \epsilon_{\mathrm{d}}$}
        \STATE \(z^\star \gets z^{(k+1)}\)
        \STATE \textbf{break}
    \ENDIF
    \STATE \(z^\star\gets z^{(k+1)}\)
    \STATE \(k \gets k + 1\)
\ENDWHILE
\end{algorithmic}
\end{algorithm}
\end{minipage}
\end{lrbox}
\Needspace{\dimexpr\ht\admmAlgorithmBox+\dp\admmAlgorithmBox+12pt\relax}
\begin{wrapfigure}{r}{0.45\textwidth}
\setlength{\intextsep}{6pt}
\usebox{\admmAlgorithmBox}
\end{wrapfigure}
As shown in Algorithm~\ref{alg:admm}, ADMM proceeds by alternating minimization of the AL \eqref{eq:admm-scaled-lagrangian} with respect to $z$ and $w$, followed by an ascent step on the dual variable.
In the consensus setting, the \textit{primal residual}
\begin{equation}
    r^{(k+1)} = z^{(k+1)} - w^{(k+1)},
\end{equation}
measures violation of the agreement constraint $z = w$, while the \textit{dual residual}
\begin{equation}
    s^{(k+1)} = \rho\big(w^{(k+1)} - w^{(k)}\big),
\end{equation}
measures the change in the consensus variable between successive iterations.
Thus, $r^{(k+1)}$ quantifies primal infeasibility, whereas $s^{(k+1)}$ quantifies whether the dual iterates have stabilized.

For convex programs, ADMM enjoys strong convergence guarantees under mild assumptions.
In particular, if $f$ and $g$ are closed, proper, and convex, and if a saddle point exists for \eqref{eq:admm-problem}, then the ADMM iterates generated by Algorithm~\ref{alg:admm} satisfy $r^{(k)} \rightarrow 0$ and $s^{(k)} \rightarrow 0$ as $k \rightarrow\infty$~\citep{boyd2011admm}.

\paragraph{Graph C-ADMM}
Now we specialize the preceding C-ADMM algorithm to the context of a distributed multi-agent convex optimization.
To this end, let us fix a connected, undirected communication graph $\scriptG = (\scriptV, \scriptE)$, where $\scriptV = \{1,\ldots,M\}$ denotes the set of agents and $\scriptE \subseteq \{\{i,j\}\in\scriptV : i \neq j\}$ denotes the set of undirected communication links.
The neighbor set of agent $i\in\scriptV$ is defined%
\footnote{The assumption of an undirected graph implies that $j \in \scriptN_i \iff i \in \scriptN_j$.}
as $\scriptN_i \triangleq \{j\in\scriptV\setminus\{i\} : \{i,j\} \in \scriptE\}$.
Consider the generic coupled (and centralized) convex program of the form
\begin{equation}~\label{eq:graph-optimization}
    \begin{aligned}
        \min_{\forall i: \theta_i}\quad &\sum_{i=1}^{M}\, J_i(\theta_i), \\
        \mathrm{s.t.}\quad &\theta_i\in\Theta_i, \\
        &\hat{h}_{ij}(\theta_i,\theta_j) = 0, \\
        &\forall i = 1,\ldots,M,\quad \forall j\in\scriptN_i.
    \end{aligned}
\end{equation}
where $J_i(\cdot)$ is a convex function, $\Theta_i$ is a convex set, for all $i = 1,\ldots, M$, and $\hat{h}_{ij}$ is affine, for all $\{i,j\}\in\scriptE$.
The main difficulty in distributing the program \eqref{eq:graph-optimization} is in the constraints $\hat{h}_{ij}$ coupling agent $i$'s decision vector to agent $j$'s decision vector.
To circumvent this, we introduce the consensus variables $w_i$ and $w_{ij}$ to obtain the equivalent program
\begin{equation}
    \begin{aligned}
        &\min_{\forall i,\forall j \in\scriptN_i: \theta_i, w_i, w_{ij}} &&\sum_{i=1}^{M}\, J_i(\theta_i), \\ 
        &\quad\quad\quad \mathrm{s.t.} &&\theta_i\in\Theta_i, \quad i = 1,\ldots, M, \\ 
        &&&\hat{h}_{ij}(w_i,w_{ij}) = 0, \quad \forall j \in \scriptN_i,\, \forall i = 1,\ldots, M, \\
        &&&\theta_i = w_i, \quad \theta_j = w_{ij}, \quad \forall j \in \scriptN_i,\, \forall i = 1,\ldots, M.
    \end{aligned}
\end{equation}
The variables $w_i$ can be interpreted as a copy of the local decision variables corresponding to agent $i$ and the variables $w_{ij}$ as agent $i$'s \textit{belief}, or guess, of what agent $j$'s local decision variables are. 
By stacking the variables $\{\theta_i\}_{i=1}^{M}$ into the earlier decision vector $z$ and the local and incident copies $\{w_i,w_{ji}\}_{i=1,j\in\scriptN_i}^{M}$ into $w$, the consensus equations take the general form $Dz-w=0$, where $D$ is the appropriate duplication/incidence operator. This gives the form needed for C-ADMM, with the added benefit that the primal and dual updates can be solved in parallel among the fleet.
To see this, construct the AL over all nodes $\scriptV$ in the graph as 
\begin{equation}
    \scriptL_{\rho} = \sum_{i=1}^{M} \Big(\underbrace{J_i(\theta_i) + \eta_i^\intercal (\theta_i - w_i) + \frac{\rho}{2}\|\theta_i - w_i\|_{2}^{2}}_{\scriptL_{\rho,i}(\theta_i,w_i,\eta_i)} + \sum_{j\in\scriptN_i}\big(\underbrace{\eta_{ij}^\intercal (\theta_j - w_{ij}) + \frac{\rho}{2}\|\theta_j - w_{ij}\|_{2}^{2}}_{\scriptL_{\rho,ij}(\theta_j, w_{ij}, \eta_{ij})}\big)\Big), \label{eq:graph-AL}
\end{equation}
where we identify $\scriptL_{\rho,i}$ as the per-agent contribution to the AL and $\scriptL_{\rho,ij}$ as the neighboring contribution to the AL.
The resulting C-ADMM algorithm follows the same flow as in Algorithm~\ref{alg:admm}, so it is not repeated for brevity.
Instead, we walk through the high-level sequential updates: first, and in parallel, the agents solve the primal minimization problem
\begin{equation}
    \theta_i^{k+1} = \mathrm{arg}\min_{\theta_i}\, \scriptL_{\rho,i}(\theta_i, w_i^k,\eta_i^k) + \sum_{j\in\scriptN_i}\scriptL_{\rho,ji}(\theta_i, w_{ji}^k,\eta_{ji}^k) \quad \mathrm{s.t.} \quad \theta_i \in \Theta_i.
\end{equation}
After this step, agent $i$ communicates its decision variables $\theta_i^\star = \theta_i^{(k+1)}$ to its neighbors $\scriptN_i$, and the neighbors communicate their decision variables $\{\theta_j^\star\}_{j\in\scriptN_i}$ back to agent $i$.
The next step is for the agents to solve, in parallel, the consensus minimization problem
\begin{align}
    \begin{bmatrix}
        w_i^{k+1} \\ 
        \{w_{ij}^{k+1}\}_{j\in\scriptN_i}
    \end{bmatrix} = \ &\underset{w_i,\{w_{ij}\}_{j\in\scriptN_i}}{\operatorname{argmin}}\, &&\scriptL_{\rho,i}(\theta_i^{k+1}, w_i, \eta_i^k) + \sum_{j\in\scriptN_i} \scriptL_{\rho,ij}(\theta_j^{k+1}, w_{ij}, \eta_{ij}^k) \nonumber \\ 
    &\quad \ \mathrm{s.t.} &&\hat h_{ij}(w_i, w_{ij}) = 0, \quad \forall j \in \scriptN_i.
\end{align}
Finally, the agents perform the dual ascent step, which is given by
\begin{equation}
    \begin{aligned}
        \eta_i^{k+1} &= \eta_i^k + \rho(\theta_i^{k+1} - w_i^{k+1}), \\ 
        \eta_{ij}^{k+1} &= \eta_{ij}^{k} + \rho(\theta_j^{k+1} - w_{ij}^{k+1}), \quad \forall j \in \scriptN_i.
    \end{aligned}
\end{equation}
At the end of this step, the agents perform one last round of communication, where agent $i$ communicates to its neighbors its consensus and dual variables $(w_{ij}, \eta_{ij})$ and receives from its neighbors the variables $\{w_{ji}, \eta_{ji}\}_{j\in\scriptN_i}$.
This completes one full decentralized graph C-ADMM iteration.
The process is then repeated until convergence.

\section{UAV Ergodic Search}~\label{sec:UAV-traj-opt}
At each planning instant, the aerial fleet is tasked with covering a prescribed base distribution and performing vision-based sensing to detect new targets and localize existing targets.
To rigorously formulate and solve this collaborative search and sensing problem, we proceed in three steps.
First, we present a probabilistic model for the vision-based sensors responsible for target detections, and the corresponding RFS-based measurement models used for data collection and Bayesian inference.
Then, we incorporate the resulting sensor footprints into the ergodic metric and formulate the joint centralized trajectory optimization problem, taking into account inter-agent collision avoidance.
Finally, we present a fully decentralized optimization method based on SCP and C-ADMM to seek candidate local solutions.
Figure~\ref{fig:collaborative-autonomy-architecture} summarizes the complete sequential architecture and the information exchanged between the UAV and UGV fleets.
\begin{figure}[!htb]
    \centering
    \includegraphics[width=\linewidth]{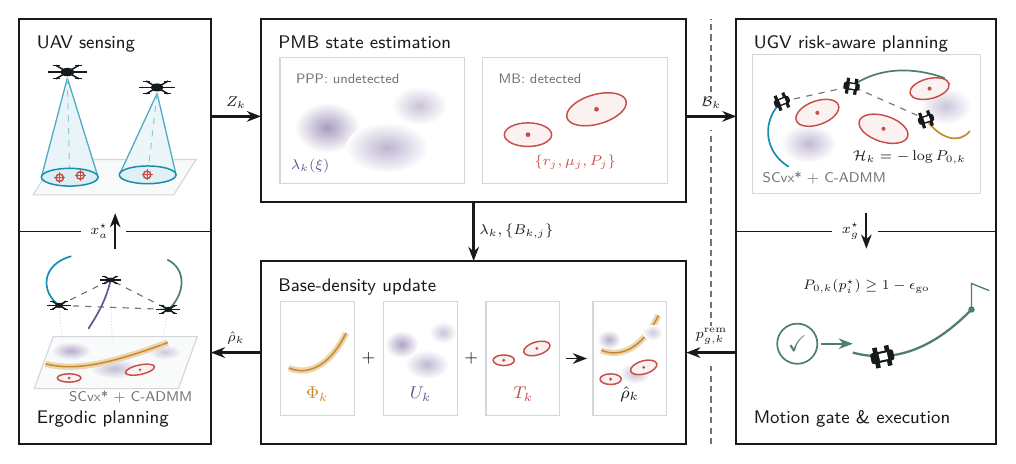}
    \caption{High-level UAV--UGV collaborative autonomy architecture. UAV measurements update the PMB posterior, which informs risk-aware UGV planning, while the committed UGV routes and posterior components update the UAV search density; the PMB void probability gates UGV execution. \josh{}}
    \label{fig:collaborative-autonomy-architecture}
\end{figure}

\subsection{Sensor and Observation Modeling}\label{subsec:sensor-model}
For each UAV $i = 1,\ldots, N_a$, we define its state by $x_i = [r_i^\intercal, v_i^\intercal, q_i^\intercal]^\intercal\in\R^{9}$ and control input by $u_i \triangleq [T_i, \omega_i^\intercal]^\intercal \in \R^{4}$.
Here $r_i \in \R^{3}$ is the position, $v_i\in\R^{3}$ is the inertial velocity, $q_i = [\phi_i, \theta_i, \psi_i]^\intercal\in\R^{3}$ is the Euler-angle parameterized (XYZ) attitude, $T_i \geq 0$ is the total thrust, and $\omega_i \in \R^{3}$ is the body angular rate.
In this setup, we assume that the total thrust and angular commands can be sufficiently tracked by a lower-level flight controller.
Denoting by $R(q_i)$ the body-to-world frame rotation, by $S(q_i)$ the matrix transforming body angular rate to Euler angle rates, and $e_3 = [0,0,1]^\intercal$, the quadrotor dynamics are given by
\begin{equation}~\label{eq:uav-dynamics}
    \begin{aligned}
        \dot{r}_i &= v_i, \\
        \dot{v}_i &= -g e_3 + \frac{T_i}{m_i} R(q_i) e_3, \\ 
        \dot{q}_i &= S(q_i)\omega_i.
    \end{aligned}
\end{equation}
Let $f_i : \R^{n}\times\R^{m}\rightarrow\R^{n}$ denote the UAV dynamics in \eqref{eq:uav-dynamics} and define $(\scriptX_i,\scriptU_i) \subset \R^{n}\times\R^{m}$ as the admissible convex state and input sets chosen to enforce actuator and flight-envelope limits.

\begin{wrapfigure}{r}{0.42\textwidth}
    \centering
    \includegraphics[width=\linewidth]{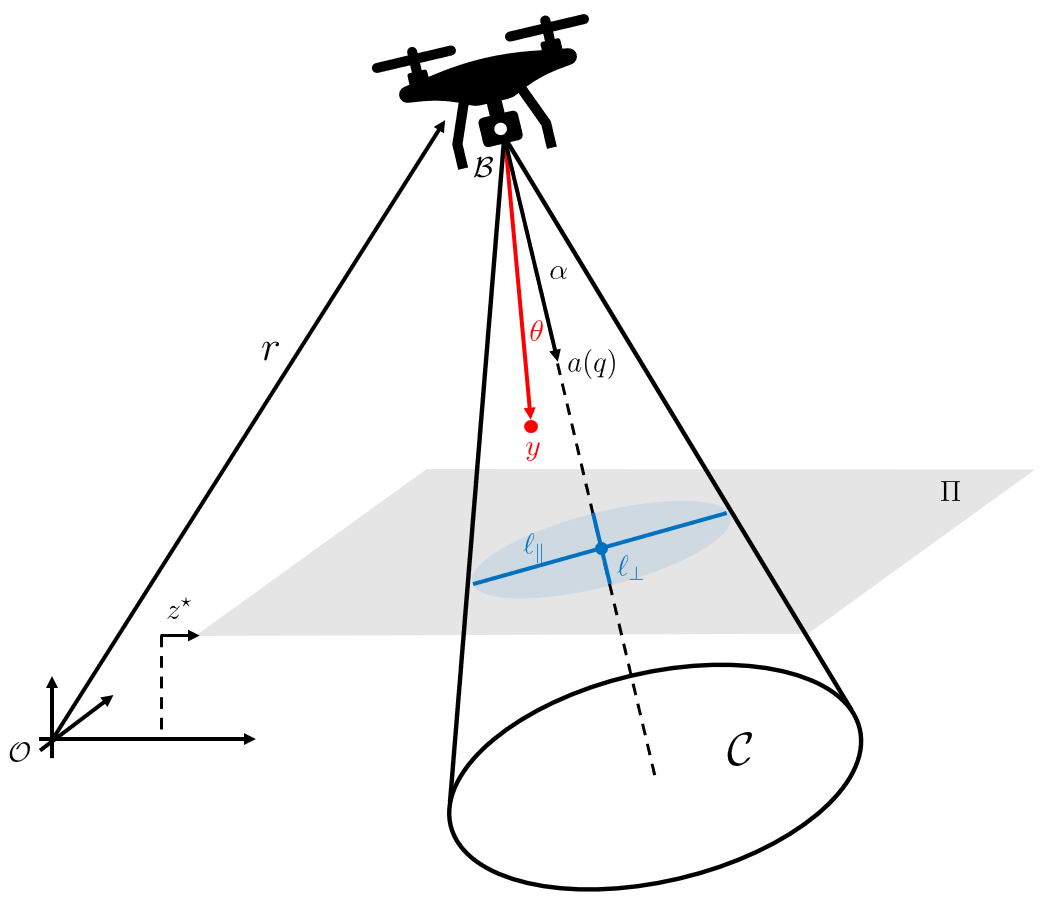}
    \makeatletter
    \long\def\@makecaption#1#2{\vskip 10pt {\raggedright #1: #2\par}}
    \makeatother
    \caption{Geometry of quadrotor body-mounted optical sensor field of view, and its intersection with the target plane. \josh{}}
    \label{fig:drone-fov}
\end{wrapfigure}
When performing collaborative sensing and target detection, we assume the targets are distributed on a two-dimensional domain $\Omega = [0,L_1] \times [0,L_2] \subset \R^{2}$ for some $L_1,L_2 > 0$, and are located on the plane $z = 0$ for simplicity.
Furthermore, we assume optical-based measurements in the form of a body-fixed camera looking along $-e_3$.
Let $C_b^w(q_i)\triangleq R(q_i)$ map body-frame coordinates to world-frame coordinates, and let the fixed camera-to-body rotation be $C_c^b=R_x(\pi)$.
Then $C_c^w(q_i)=C_b^w(q_i)C_c^b$ and $o_i\triangleq C_c^w(q_i)e_3=-R(q_i)e_3$ is the associated unit optical axis of UAV agent $i$, expressed in world coordinates; let $\alpha_i$ denote its half-angle.

Write $p_i \triangleq [r_{x,i}, r_{y,i}]^\intercal$ for the UAV's planar position and $h_i \triangleq r_{z,i} > 0$ for its altitude.
Similarly, decompose the sensor optical axis as $b_i \triangleq [o_{x,i}, o_{y,i}]^\intercal$ so that $o_i = [b_i^\intercal, o_{z,i}]^\intercal, \, o_{z,i} < 0$, and $\|b_i\|^2 + o_{z,i}^2 = 1$.
At a given position and attitude $(r_i, q_i)$, the pose-aware sensor Field-of-View (FoV) cone is defined as the set
\begin{equation}
    \scriptC_i \triangleq \left\{y \in \R^3\setminus\{r_i\} : \frac{o_i^\intercal (y - r_i)}{\|y - r_i\|} \geq \cos\alpha_i, \quad o_i^\intercal (y - r_i) > 0\right\}.
\end{equation}
Geometrically, when the sensor is nadir pointed, the FoV cone intersects the target plane to form a circle. 
In general, given some attitude $q_i$, the sensor FoV cone intersects the target plane to form an elliptical cross-section; see Figure~\ref{fig:drone-fov} for a geometric illustration.
The next result formalizes this notion.
\begin{prop}~\label{prop:FoV-ellipse}
    Let $o\in\R^{3}$ denote the UAVs sensor optical axis with half-angle $\alpha\in(0,\pi/2)$, and let $r\in\R^{3}$ denote its current position.
    Assume $o_{z}<-\sin\alpha$, or equivalently, that the downward-facing cone has a bounded intersection with the target plane, and define $D \triangleq o_{z}^2 - \sin^2\alpha>0$.
    The target-plane intersection traced out by the FoV cone is an ellipse; its restriction to the search domain is the footprint
    \begin{equation}
        \scriptE_{\mathrm{f}}(x) = \left\{\xi \in \Omega : \frac{\big(u_{\mathrm f}^\intercal (\xi - p_{\mathrm{ell}})\big)^2}{\ell_{\parallel}^2} + \frac{\big(v_{\mathrm f}^\intercal (\xi - p_{\mathrm{ell}})\big)^2}{\ell_{\perp}^2} \leq 1 \right\}, \label{eq:fov-footprint}
    \end{equation}
    where 
    \begin{equation*}
        u_{\mathrm f} \triangleq \frac{b}{\|b\|}, \quad v_{\mathrm f} \triangleq 
        \begin{bmatrix}
            0 & -1 \\
            1 & 0
        \end{bmatrix} u_{\mathrm f},
    \end{equation*}
    denote the principal sensor directions, and 
    \begin{equation*}
        D=o_{z}^2-\sin^2\alpha, \quad p_{\mathrm{ell}} = p - \frac{o_{z} h}{D}b, \quad \ell_{\parallel} = \frac{h\sin(2\alpha)}{2D}, \quad \ell_{\perp} = \frac{h\sin\alpha}{\sqrt{D}}.
    \end{equation*}
    At exact nadir ($b=0$), any orthonormal planar basis may be used and the footprint is the circle of radius $h\tan\alpha$ centered at $p$.
\end{prop}
\begin{proof}
    See Appendix~\ref{app:A}.
\end{proof}

Thus, at every observation instant, the target-plane intersection traces out a finite ellipse, whose restriction to $\Omega$ is the sensor footprint (assuming a sufficiently downward-facing camera).
The next layer to detecting an object on that target plane is the \textit{quality} of the sensor, which is in general a function of both its altitude and its attitude.
In particular, for UAV $i$ at measurement time $t_k$, we define the dimensionless pointwise sensing-quality field, and its value with respect to target $\ell$ at position $p_{\ell}$, by
\begin{equation}
    \chi_{i,k}(\xi)\triangleq\beta_s(x_i(t_k))G_i(\xi;x_i(t_k)), \qquad
    \chi_{i\ell}(x_i(t_k))\triangleq\chi_{i,k}(p_\ell),
\end{equation}
where
\begin{equation}
    G_i(p; x_i) \triangleq \exp\left[-\frac{1}{2}(p - p_{\mathrm{ell},i})^\intercal \Sigma_{\mathrm{f},i}^{-1} (p - p_{\mathrm{ell},i})\right], \label{eq:footprint-spread}
\end{equation}
is an ellipse-associated Gaussian, with regularized covariance $\Sigma_{\mathrm{f},i} \triangleq \kappa_{\mathrm{f}}^2 Q_i + \varepsilon_{\Sigma} I_2$, given some $\kappa_{\mathrm{f}} > 0,\, \varepsilon_{\Sigma} > 0$, and geometric shape matrix $Q_i \triangleq \ell_{\parallel,i}^2 u_{\mathrm f,i}u_{\mathrm f,i}^\intercal + \ell_{\perp,i}^2 v_{\mathrm f,i}v_{\mathrm f,i}^\intercal$.
The second factor
\begin{equation}
    \beta_s(x) = \frac{(-o_z)^{\gamma}}{1 + (h/\bar{h})^2} \in [0,1],
\end{equation}
represents altitude and off-nadir sensing quality, where $\gamma>0$ and $\bar h>0$ are design parameters.
Here, the Gaussian spread $G(p,x)$ models the proportional quality drop-off from the sensor's elliptical center, where targets in the center of the frame are more likely to be detected than those at the boundary of the frame; and the factor $\beta_s(x)$ models quality drop-off in altitude due to the target occupying fewer image pixels, as well as off-nadir degradation; c.f. Figure~\ref{fig:sensor-quality} for representative examples of these quality factors.
In total, at any given measurement time $t_k$ with UAV state $x_i$, the probability of detection \textit{field} is given by
\begin{equation}
    p^D_{i,k}(\xi) = \1_{\scriptE_{\mathrm{f}}(x_i)}(\xi)\, \chi_{i,k}(\xi), \quad 0 \leq \chi_{i,k}(\xi) \leq 1.
\end{equation}
\begin{figure}[!htb]
    \centering
    \includegraphics[width=\linewidth]{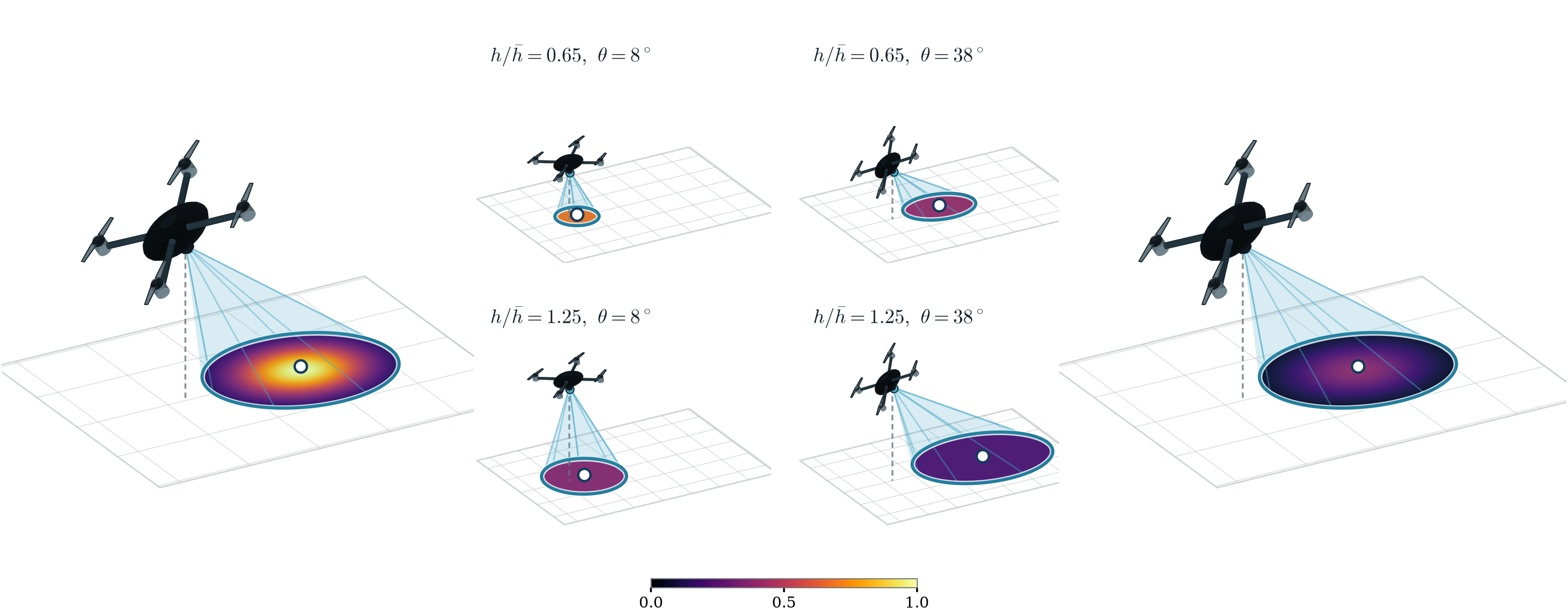}
    \caption{Pose-dependent sensing quality examples: (left) Gaussian spread factor $G(p;x)$ with respect to footprint center; (middle) altitude/attitude trade-off $\beta_s(x)$ for a 2x2 sweep of representative heights and angles; (right) the combined field $\chi(p;x) = \beta_s(x) G(p;x)$ for a fixed vehicle pose. \josh{}}
    \label{fig:sensor-quality}
\end{figure}

Conditional on a positive measurement, we assume a camera-frame angles-only observation model given by
\begin{equation}~\label{eq:meas-model}
    \begin{aligned}
    y &= h(\xi;x_i)+\nu(\xi,x_i)
      = \begin{bmatrix}
          \operatorname{atan2}(\ell_1,\ell_3) \\
          \operatorname{atan2}(\ell_2,\ell_3)
        \end{bmatrix}+\nu(\xi,x_i), \\
    \ell &= \big(C_c^w(q_i)\big)^\intercal
        \left(\begin{bmatrix}\xi\\0\end{bmatrix}-r_i\right),
        \qquad \nu\sim\scriptN\big(0,R(x_i\mid\xi)\big).
    \end{aligned}
\end{equation}
We use the heteroscedastic model $R(x_i\mid\xi)=\sigma^2(\chi_{i,k}(\xi))I_2$, with
\begin{equation}\label{eq:angular-noise-law}
    \sigma(\chi)=\Pi_{[\sigma_{\min},\sigma_{\max}]}\!\left(k_\sigma(\chi^{-1/2}-1)\right), \qquad \chi>0,
\end{equation}
where $k_\sigma>0$, $0<\sigma_{\min}\leq\sigma_{\max}$, and $\sigma(0)\triangleq\sigma_{\max}$.
The measurement covariance captures the dependence of (angular) localization accuracy on viewing quality.
Namely, when $\chi(\xi)$ is large, the target is observed under favorable sensing conditions and the corresponding measurement has smaller (angular) uncertainty.
Conversely, as $\chi(\xi)$ decreases, the target becomes more difficult to resolve and the uncertainty increases.
{For the linearized Gaussian measurement model with covariance $R$ held fixed, the Fisher information is $\scriptI=H^\intercal R^{-1}H$, where $H=\partial h/\partial\xi$~\citep{wilson2014trajectory}.
This inverse-covariance weighting motivates the $\chi^{-1/2}$ dependence in $\sigma(\chi)$.}

\subsection{Kernelized Ergodic Coverage}\label{subsec:kernel-coverage}
The starting point of the UAV fleet optimization problem is that of the spatial distribution, $\P$, which is sought to track.
Details on the construction and update of this base distribution, or base density, as we will refer to it henceforth, will be presented in Section~\ref{sec:coordination}.
In general, it will be a mixture of a UGV-path centered density along with a Gaussian mixture density with respect to existing target tracks, along with a residual undetected PPP density from the previous sensing.
In any case, suppose the base distribution $\P$ is prescribed to the UAV fleet over the planar rectangular domain $\Omega = [0,L_1]\times [0,L_2]$, and let $\phi_k$ denote the corresponding spatial Fourier coefficients.

In the traditional setup, over a horizon $T > 0$, the ordinary empirical occupation measure (cf. \eqref{eq:occupation-measure}) is given by
\begin{equation}
    \mu_T^{\mathrm{pt}} = \frac{1}{T} \int_{0}^{T} \delta_{g\big(x(\tau)\big)}\, \d\tau \iff \mu_{T}^{\mathrm{pt}}(\scriptA) = \frac{1}{T}\int_{0}^{T} \1_{\scriptA}\Big(g\big(x(\tau)\big)\Big)\, \d\tau,
\end{equation}
where $x: [0,T] \rightarrow \scriptX$ is the UAV trajectory, $g: \scriptX \rightarrow \Omega$ extracts the planar sensor location from the state, and the right-hand side equivalently shows the action of the trajectory distribution on any measurable set $\scriptA \subseteq \Omega$.
Thus, $\mu_T(\scriptA)$ tells us exactly what fraction of time the sensor \textit{point} lies in the set $\scriptA$.
In the present context, however, we have shown in Section~\ref{subsec:sensor-model} that the UAV state $x$ induces the finite sensor footprint $\scriptE_{\mathrm f}(x) \subseteq\Omega$ from \eqref{eq:fov-footprint}.
Thus, we generalize the point sensor map $x \mapsto g(x) \in \Omega$ to a finite sensor map $x \mapsto \nu_{x} \in \scriptP(\Omega)$, where $\nu_x$ can be thought of as the instantaneous spatial sensing measure induced by state $x$.

For the simplest hard, uniformly sensitive footprint, define the sensor footprint distribution as
\begin{equation}
    \nu_x(\d\xi) = \frac{\1_{\scriptE_{\mathrm f}(x)}(\xi)}{|\scriptE_{\mathrm f}(x)|}\, \d\xi \iff \nu_x(\scriptA) = \frac{|\scriptA \cap \scriptE_{\mathrm f}(x)|}{|\scriptE_{\mathrm f}(x)|},
\end{equation}
which is just a uniform distribution over the closed, convex footprint.
Now we simply average these instantaneous measures over time to obtain the finite-time occupation measure
\begin{equation}
    \mu_{T}^{f} = \frac{1}{T}\int_{0}^{T} \nu_{x(\tau)}\, \d\tau \iff \mu_{T}^{f}(\scriptA) = \frac{1}{T}\int_{0}^{T} \frac{|\scriptA\cap\scriptE_{\mathrm f}\big(x(\tau)\big)|}{|\scriptE_{\mathrm f}\big(x(\tau)\big)|}\, \d\tau.
\end{equation}
Although the above construction is a natural generalization of the point-mass sensor footprint to a pose-aware, spatially varying FoV sensor, it is rather cumbersome to work with in the overall optimization problem due to the discontinuity of the distribution at the footprint boundary $\partial\scriptE_{\mathrm f}$.
In the most general case, we can abstract away the sensor footprint measure entirely by defining a state-dependent spatial kernel
\begin{equation}
    K(\cdot;x) \in \scriptP(\Omega), \quad K(\xi;x) \geq 0,\, \forall\xi \in \Omega, \quad \int_{\Omega} K(\xi;x)\, \d\xi = 1.
\end{equation}
Then, define the associated time-averaged sensor occupancy measure as 
\begin{equation}
    \mu_T^K = \frac{1}{T}\int_{0}^{T} K\big(\cdot; x(\tau)\big)\, \d\tau \iff \mu_T^K(\scriptA) = \frac{1}{T}\int_{0}^{T} \int_{\scriptA} K\big(\xi; x(\tau)\big)\, \d\xi\, \d\tau.
\end{equation}
Given the orthonormal Fourier basis functions $\{F_k\}$ over the rectangular domain $\Omega$ as defined in \eqref{eq:fourier-basis}, the corresponding coefficients of the generalized occupation measure immediately follow as
\begin{equation}\label{eq:kernel-coefficients}
    c_k = \frac{1}{T}\int_{0}^{T} \Psi_k\big(x(\tau)\big)\, \d\tau, \quad \Psi_k(x) \triangleq \int_{\Omega} K(\xi; x) F_k(\xi)\, \d\xi,
\end{equation}
where we regard $\Psi_k$ as the \textit{spectral response} corresponding to mode $k$.
To aid in tractability when computing derivatives for SCP (as discussed in Section~\ref{subsec:SCvx*}), we therefore relax the hard sensor footprint to the Gaussian-smoothed kernel
\begin{equation}
    K(\xi;x_i) = \frac{G_i(\xi;x_i)}{\int_{\Omega} G_i(\sigma;x_i)\, \d\sigma},
\end{equation}
with $G_i(\cdot;x_i)$ as defined in \eqref{eq:footprint-spread}.
\begin{table}[t]
    \centering
    \caption{Comparison of sensor models used to construct the ergodic occupation measure on the search domain $\Omega$.}
    \label{tab:sensor-model-comparison}
    \renewcommand{\arraystretch}{1.5}
    \setlength{\tabcolsep}{3.5pt}
    \begin{tabular}{lll}
        \toprule
        \textbf{Sensor Model}
        & \textbf{Instantaneous Sensing Measure $\nu_x$}
        & \textbf{Spectral Response $\Psi_k(x)$} \\
        \midrule

        Point Mass
        &
        $\displaystyle
        \nu_x = \delta_{g(x)}
        $
        &
        $\displaystyle
        \Psi_k(x) = F_k\!\left(g(x)\right)
        $
        \\[1mm]

        Hard FoV
        &
        $\displaystyle
        \nu_x(\mathrm{d}\xi)
        =
        \frac{\1_{\mathcal{E}_f(x)}(\xi)}
        {|\mathcal{E}_f(x)|}
        \,\mathrm{d}\xi
        $
        &
        $\displaystyle
        \Psi_k(x)
        =
        \frac{1}{|\mathcal{E}_f(x)|}
        \int_{\mathcal{E}_f(x)}
        F_k(\xi)\,\mathrm{d}\xi
        $
        \\[3mm]

        Smooth FoV
        &
        $\displaystyle
        \nu_x(\mathrm{d}\xi)
        =
        K(\xi;x)\,\mathrm{d}\xi,
        \int_{\Omega}K(\xi;x)\,\mathrm{d}\xi=1
        $
        &
        $\displaystyle
        \Psi_k(x)
        =
        \int_{\Omega}
        K(\xi;x)F_k(\xi)\,\mathrm{d}\xi
        $
        \\

        \bottomrule
    \end{tabular}
\end{table}
Table~\ref{tab:sensor-model-comparison} summarizes the preceding discussion on the various sensor models and shows the progression from the canonical point-mass model to the proposed Gaussian kernel model, i.e.,
\begin{equation*}
    \delta_{g(x)} \longrightarrow \frac{\1_{\scriptE_{\mathrm f}(x)}}{|\scriptE_{\mathrm f}(x)|} \longrightarrow  \frac{G(\cdot;x)}{\int_{\Omega} G}.
\end{equation*}

\subsection{UAV Problem Formulation}\label{subsec:uav-problem}
{Let $N_a$ denote the number of UAVs, indexed by $i=1,\ldots,N_a$.}
The sensor footprint occupancy Fourier coefficients as defined in \eqref{eq:kernel-coefficients} can be readily extended to the entire multi-agent UAV fleet by averaging over the set of agents, giving the global occupancy coefficients
\begin{equation}
    C_k(x) \triangleq \frac{1}{\textcolor{black}{N_a}}\sum_{i=1}^{\textcolor{black}{N_a}} c_{k,i} = \frac{1}{\textcolor{black}{N_a}}\sum_{i=1}^{\textcolor{black}{N_a}} \frac{1}{T}\int_{0}^{T} \Psi_k\big(x_i(\tau)\big)\, \d\tau.
\end{equation}
The resulting fleet-level ergodic metric is then given by
\begin{equation}\label{eq:fleet-ergodic-metric}
    \scriptE_{\hat\mu} = \sum_{k\in\N_0^{2}} \Lambda_k \big(C_k(x) - \phi_k\big)^2 \approx \sum_{k\in\scriptK} \Lambda_k \big(C_k(x) - \phi_k\big)^2,
\end{equation}
where in practice, we truncate to some maximum number of modes, over the set $\scriptK \triangleq \{(k_1,k_2) \in \N_0^2 : 0 \leq k_i \leq K_i,\, i = 1,2\}$ for some $K_1, K_2 \in \N$.
Since $C_k(x)$ involves the integral of each UAV state over time through the individual coefficients $c_{k,i}$, it is a bit cumbersome to work with.
Instead, we can define the Fourier coefficient \textit{accumulator} state by
\begin{equation}
    \dot{a}_{k,i}(t) = \frac{1}{T}\Psi_k\big(x_i(t)\big), \quad a_{k,i}(0) = 0, \, t \in [0,T] \implies c_{k,i} = a_{k,i}(T).
\end{equation}
This state augmentation follows the auxiliary-state construction in \citet{delatorre2016stochastic} and the extended-state formulation used in RAnGE~\citep{berger2024range}. 
As a result, by stacking the mode-indexed accumulator states as $a_i \triangleq [a_{k,i}]_{k\in\scriptK} \in \R^{|\scriptK|}$ and defining the fleet-average terminal accumulator $\tau\triangleq \textcolor{black}{N_a}^{-1}\sum_{i=1}^{\textcolor{black}{N_a}}a_i(T)$, the ergodic metric can be equivalently written as the Mayer cost
\begin{equation}
    \scriptE_{\hat\mu}(\tau) = \left\|\tau-\phi\right\|_{\Lambda}^{2}, \quad
    \dot{a}_i(t) \triangleq g_a\big(x_i(t)\big) = \frac{1}{T}
    \begin{bmatrix}
        \Psi_1\big(x_i(t)\big), \\ 
        \vdots \\
        \Psi_{|\scriptK|}\big(x_i(t)\big)
    \end{bmatrix}, \quad t \in [0,T],
\end{equation}
with initial conditions $a_i(0) = 0_{|\scriptK|\times 1}$.
As a result, notice that the objective function has been reformulated as a convex function of the decision variables $\{a_i\}_{i=1}^{\textcolor{black}{N_a}}$ and the nonlinearities are relegated into the accumulator dynamics through $g_a(\cdot)$.

In addition to minimizing the fleet-level ergodic metric as in \eqref{eq:fleet-ergodic-metric}, it is also necessary to maintain safety between neighboring pairs of UAVs.
To this end, at each planning instant, we assume an undirected, connected graph $\scriptG = (\scriptV, \scriptE)$, where two agents $\{i,j\} \in \scriptE$ are connected if they are within each other's communication radius%
\footnote{We assume that the communication radius is large enough so that the resulting graph is connected, and that agents can detect each other well before they need to coordinate.}%
, that is, $\|r_i - r_j\| \leq D_{\mathrm{com}}$, for some $D_{\mathrm{com}} > 0$.
We impose geometric separation between UAV pairs as
\begin{equation}~\label{eq:uav-separation}
    \|r_i(t)-r_j(t)\|_2\geq d_{ij},
    \qquad \{i,j\}\in\scriptE,
    \qquad t\in[0,T],
\end{equation}
where $d_{ij}=R_i^a+R_j^a+\Delta>0$, $R_i^a$ is UAV $i$'s physical safety radius, and $\Delta>0$ is an additional margin.
\begingroup\color{black}
\hypersetup{citecolor=black,linkcolor=black,urlcolor=black}
Using the normalized residual in Section~\ref{subsec:ctcs-separation}, define
\begin{equation*}
    g_{ij}(x_i,x_j)\triangleq 1-\frac{\|r_i-r_j\|_2^2}{d_{ij}^2},
    \qquad g_{c,ij}(x_i,x_j)\triangleq[g_{ij}(x_i,x_j)]_+^2.
\end{equation*}
The pairwise accumulator satisfies $\dot y_{ij}=g_{c,ij}$ with $y_{ij}(0)=y_{ij}(T)=0$.
Define the augmented agent state, dynamics, and initial condition by
\begin{equation*}
    z_i\triangleq\begin{bmatrix}x_i\\a_i\end{bmatrix},\qquad
    \tilde f_i(z_i,u_i)\triangleq
    \begin{bmatrix}f_i(\Pi_xz_i,u_i)\\g_a(\Pi_xz_i)\end{bmatrix},\qquad
    z_{0,i}\triangleq\begin{bmatrix}x_{0,i}\\0_{|\scriptK|\times1}\end{bmatrix},
\end{equation*}
where $\Pi_x=[I_n\ \ 0_{n\times|\scriptK|}]$ and
$\Pi_a=[0_{|\scriptK|\times n}\ \ I_{|\scriptK|}]$ select the physical and ergodic states.
{Let $\ell_i^a(x_i,u_i)$ denote the UAV running cost.
The control-effort choice used in the transcription below is $\ell_i^a=\frac{\beta}{2}u_i^\intercal R_i u_i$, with $\beta>0$ and $R_i\succ0$; Section~\ref{subsec:case-setup} specifies the hover reference and state regularization used in the simulations.}
{With $\scriptV=\{1,\ldots,N_a\}$, the centralized problem is}
\begin{subequations}\label{prob:uav-centralized}
\begin{empheq}[box=\fbox]{align}
    \min_{\tau,z,u,y}\quad &J_a(z,u,\tau)
    ={\sum_{i=1}^{N_a}\int_0^T \ell_i^a(x_i,u_i)\,\d t}
      +\scriptE_{\hat\mu}(\tau),\\
    \mathrm{s.t.}\quad &\dot z_i=\tilde f_i(z_i,u_i),\quad z_i(0)=z_{0,i},\quad\forall i\in\scriptV,\ \forall t\in[0,T],\\
    &\dot y_{ij}=g_{c,ij}(\Pi_xz_i,\Pi_xz_j),\quad y_{ij}(0)=y_{ij}(T)=0,\quad\forall\{i,j\}\in\scriptE,\ \forall t\in[0,T],\\
    &(\Pi_xz_i,u_i)\in\scriptX_i\times\scriptU_i,\quad\forall i\in\scriptV,\ \forall t\in[0,T],\\
    &\tau=\frac{1}{N_a}\sum_{i=1}^{N_a}\Pi_a z_i(T).
\end{empheq}
\end{subequations}
Here $z$, $u$, and $y$ collect the agent and pair trajectories.
\endgroup

\subsection{UAV Problem Solution}\label{subsec:uav-solution}
\begingroup\color{black}
\hypersetup{citecolor=black,linkcolor=black,urlcolor=black}
The objective in \eqref{prob:uav-centralized} is convex in the controls and terminal ergodic accumulators.
The nonlinear agent dynamics and pairwise separation dynamics require convexification, while separation and the fleet-average constraint couple the agents.
We first transcribe the problem, then identify its local and coupled terms.

Fix $N\geq1$, $\Delta t=T/N$, and the partition $t_n=n\Delta t$, so that
$0=t_0<\cdots<t_N=T$.
On each interval $[t_n,t_{n+1})$, use zero-order-held controls $u_i(t)=u_{i,n}$,
with $u_i(T)=u_{i,N-1}$, and let $u_n=\operatorname{col}(\{u_{i,n}\}_{i\in\scriptV})$.
Stack the agent and pair states as
\begin{equation*}
    z^a=\operatorname{col}(\{z_i\}_{i\in\scriptV},\{y_{ij}\}_{\{i,j\}\in\scriptE}),\qquad
    f_{\mathrm{aug}}(z^a,u)=\operatorname{col}(\{\tilde f_i(z_i,u_i)\}_i,
    \{g_{c,ij}(\Pi_xz_i,\Pi_xz_j)\}_{\{i,j\}\in\scriptE}).
\end{equation*}
Assume the augmented dynamics have a unique solution on each interval for the chosen initial state and constant control.
The exact flow map gives the state at the end of the interval as
\begin{equation*}
    F_n^a(z_n^a,u_n)\triangleq\zeta_n(t_{n+1}),\qquad
    \dot\zeta_n(t)=f_{\mathrm{aug}}(\zeta_n(t),u_n),\quad\zeta_n(t_n)=z_n^a.
\end{equation*}
We use multiple shooting~\citep{bock1984multiple} to transcribe the dynamics, treating the node states as independent variables and linking adjacent intervals through the matching equations
$z_{n+1}^a=F_n^a(z_n^a,u_n)$.
In practice, we approximate the flow by numerical integration (e.g. RK45).

Following \citet[Section~3.1]{elango2024ctcs}, we allow each accumulator increment to be at most a prescribed tolerance $\varepsilon_{ij,n}>0$.
These interval bounds replace $y_{ij}(T)=0$ and limit the integrated squared separation violation on each interval.
Positive tolerances therefore permit residual violation between nodes.
The resulting finite-dimensional program is
\begin{subequations}\label{prob:uav-transcription}
\begin{empheq}[box=\fbox]{align}
    \min_{\tau,z^a,u}\quad &J_a^{\mathrm d}(u,\tau)
    =\frac{\beta\Delta t}{2}\sum_{i=1}^{N_a}\sum_{n=0}^{N-1}
    u_{i,n}^\intercal R_i u_{i,n}+\scriptE_{\hat\mu}(\tau),\\
    \mathrm{s.t.}\quad &z^a_{n+1}=F_n^a(z^a_n,u_n),\quad\forall n\in\{0,\ldots,N-1\},\label{eq:ctcs-augmented-flow}\\
    &z_{i,0}=z_{0,i},\ \forall i\in\scriptV,\qquad y_{ij,0}=0,\ \forall\{i,j\}\in\scriptE,\\
    &\Pi_xz_{i,n}\in\scriptX_i,\quad\forall i\in\scriptV,\quad\forall n\in\{0,\ldots,N\},\\
    &u_{i,n}\in\scriptU_i,\quad\forall i\in\scriptV,\quad\forall n\in\{0,\ldots,N-1\},\\
    &y_{ij,n+1}-y_{ij,n}\leq\varepsilon_{ij,n},\quad\forall\{i,j\}\in\scriptE,\quad\forall n\in\{0,\ldots,N-1\},\label{eq:pair-ctcs-relaxed}\\
    &\tau=\frac{1}{N_a}\sum_{i=1}^{N_a}\Pi_a z_{i,N}.
\end{empheq}
\end{subequations}

To expose the structure of \eqref{prob:uav-transcription} symmetrically, introduce two copies of each pair's accumulator sequence, that is,
\begin{equation*}
    Y_{ij}^{i}\triangleq(y_{ij,0}^{i},\ldots,y_{ij,N}^{i}),\qquad
    Y_{ij}^{j}\triangleq(y_{ij,0}^{j},\ldots,y_{ij,N}^{j}),\qquad
    Y_{ij}^{i}=Y_{ij}^{j},\quad\forall\{i,j\}\in\scriptE.
\end{equation*}
As a result, agent $i$'s decision vector stacks its state and control nodes and its pair-accumulator copies as
\begin{equation*}
    \vartheta_i\triangleq\operatorname{col}\!\left(
    z_{i,0},\ldots,z_{i,N},\;
    u_{i,0},\ldots,u_{i,N-1},\;
    \{Y_{ij}^{i}\}_{j\in\scriptN_i}\right)\in\R^{d_{\vartheta,i}},
\end{equation*}
where $d_{\vartheta,i}=(N+1)(n+|\scriptK|+|\scriptN_i|)+Nm$.

Define the full decision vector $\theta=\operatorname{col}(\tau,\vartheta_1,\ldots,\vartheta_{N_a})$, the convex local cost
$J_i(\vartheta_i)=\frac{\beta\Delta t}{2}\sum_{n=0}^{N-1}u_{i,n}^\intercal R_i u_{i,n}$,
and the terminal selector $A_i\vartheta_i=\Pi_a z_{i,N}$.
Let $\Pi(\vartheta_i)$ stack the local physical and ergodic shooting defects, and let the convex set $\Theta_i$ collect the node bounds and initial conditions, including $y_{ij,0}^{i}=0$ and $y_{ij,n+1}^{i}-y_{ij,n}^{i}\leq\varepsilon_{ij,n}$ for every $j\in\scriptN_i$ and $n=0,\ldots,N-1$.
To define one shared shooting defect per pair, set $\bar y_{ij,n}=(y_{ij,n}^{i}+y_{ij,n}^{j})/2$ and let $\bar z_n^a$ stack the agent states and these averages.
Then
\begin{equation*}
    \pi_{ij,n}(\vartheta_i,\vartheta_j)
    =\bar y_{ij,n+1}-[F_n^a(\bar z_n^a,u_n)]_{y_{ij}},\qquad n=0,\ldots,N-1,
\end{equation*}
which depends only on the two agents and their accumulator copies.
When the copies agree, their average equals the original accumulator, so the shared defect recovers the shooting equation in \eqref{prob:uav-transcription}.
Stacking these defects into $\pi_{ij}$ gives the equivalent program
\begin{subequations}
    \begin{align}
        \min_{\theta}\quad &\sum_{i=1}^{N_a}J_i(\vartheta_i)+\scriptE_{\hat\mu}(\tau),\\
        \mathrm{s.t.}\quad &\Pi(\vartheta_i)=0,\quad\vartheta_i\in\Theta_i,\quad i\in\scriptV,\\
        &\pi_{ij}(\vartheta_i,\vartheta_j)=0,\quad\{i,j\}\in\scriptE,\\
        &Y_{ij}^{i}=Y_{ij}^{j},\quad\forall\{i,j\}\in\scriptE,\\
        &\tau=\frac{1}{N_a}\sum_{i=1}^{N_a}A_i\vartheta_i.
    \end{align}
\end{subequations}
SCvx* linearizes the nonlinear local and pairwise shooting defects while retaining the convex costs, node sets, and copy-agreement constraints.
\endgroup
{We obtain the required flow Jacobians by integrating the state and control variational equations alongside the reference trajectory on each interval and evaluating them at its endpoint (see Appendix~\ref{app:flow-sensitivities}).}
Let $\hat\vartheta$ denote the current reference decision vector.
{We introduce virtual controls for artificial infeasibility and trust regions for artificial unboundedness, and minimize the augmented Lagrangian as in Section~\ref{subsec:SCvx*} to obtain the convex subproblem}
\begin{subequations}\label{prob:uav-scvx-subproblem}
    \begin{align}
        &\min_{\substack{\tau, \forall i\in\scriptV: \vartheta_i, \xi_i \\ \forall \{i,j\}\in\scriptE: \zeta_{ij}}} &&\sum_{i=1}^{\textcolor{black}{N_a}} J_i(\vartheta_i) + \scriptE_{\hat\mu}(\tau) \nonumber\\[-1mm]
        &&&\quad + \sum_{i=1}^{\textcolor{black}{N_a}} \Big(\lambda_i^\intercal \xi_i + \frac{w}{2}\|\xi_i\|^2\Big) + \sum_{\{i,j\}\in\scriptE} \Big(\mu_{ij}^\intercal \zeta_{ij} + \frac{w}{2}\|\zeta_{ij}\|^{2} \Big)\label{eq:joint-objective}\\
        &\ \mathrm{s.t.} && \Pi^{\mathrm{cvx}}(\vartheta_i) = \xi_i, \quad \vartheta_i\in\Theta_i, \quad i \in \scriptV, \label{eq:joint-local-constraints}\\
        &&&\pi_{ij}^{\mathrm{cvx}}(\vartheta_i, \vartheta_j) = \zeta_{ij}, \quad \{i,j\}\in\scriptE, \label{eq:joint-pair-shooting}\\
        &&&\textcolor{black}{Y_{ij}^{i}=Y_{ij}^{j},\quad\forall\{i,j\}\in\scriptE}, \label{eq:joint-y-consensus}\\
        &&&\|\vartheta_i - \hat{\vartheta}_i\|_{\infty} \leq \Delta_{\mathrm{tr}}, \quad i \in\scriptV, \label{eq:joint-trust-region}\\
        &&&\tau = \frac{1}{\textcolor{black}{N_a}}\sum_{i=1}^{\textcolor{black}{N_a}} \textcolor{black}{A_i}\vartheta_i,\label{eq:uav-scvx-global-sum}
    \end{align}
\end{subequations}
Here $\lambda_i$ and $\mu_{ij}$ are the unrestricted multipliers for the local and pairwise shooting equalities, respectively, and $\xi_i$ and $\zeta_{ij}$ are their signed virtual controls.
{Each endpoint copy retains its initial condition and affine increment bounds in $\Theta_i$, while copy agreement is enforced explicitly.
The shared defect has one virtual control $\zeta_{ij}$ per pair.}
\begingroup\color{black}
\hypersetup{citecolor=black,linkcolor=black,urlcolor=black}
The objective in \eqref{prob:uav-scvx-subproblem} depends on the global variable $\tau$, while \eqref{eq:uav-scvx-global-sum} couples all agents through a sum of their terminal ergodic accumulators.
To distribute the objective, introduce a local copy $\tau_i\in\R^{|\scriptK|}$ for each agent $i\in\scriptV$ and assign it the cost $N_a^{-1}\scriptE_{\hat\mu}(\tau_i)$, with the copies constrained to agree.
To distribute the global sum constraint, introduce directed flows $f_{i\rightarrow j}\in\R^{|\scriptK|}$ for each $j\in\scriptN_i$ on $\scriptG$.
These flows replace the global sum with local balance equations and antisymmetry between neighbors, through the following constraints
\begin{subequations}\label{eq:flow-lifting}
    \begin{align}
        &f_{i\rightarrow j} + f_{j\rightarrow i} = 0, \quad \forall \{i,j\}\in\scriptE \label{eq:flow-antisymmetry}\\
        &\sum_{j\in\scriptN_{i}}f_{i\rightarrow j} = \frac{1}{\textcolor{black}{N_a}}(\textcolor{black}{A_i}\vartheta_i - \tau_i), \quad \forall i \in\scriptV, \label{eq:flow-local-balance}\\
        &\tau_i=\tau_j, \quad \forall\{i,j\}\in\scriptE.\label{eq:flow-tau-consensus}
    \end{align}
\end{subequations}
For each pair $\{i,j\}\in\scriptE$, also introduce endpoint copies $\zeta_{ij}^{i}$ and $\zeta_{ij}^{j}$ of the shared virtual control $\zeta_{ij}$, with the same multiplier $\mu_{ij}$ at both endpoints.
The local objective and constraints for agent $i\in\scriptV$ are thus
\begin{subequations}\label{prob:uav-node-separable}
    \begin{align}
        &\min_{\substack{\vartheta_i, \xi_i, \tau_i \\ \forall j\in\scriptN_i: f_{i\rightarrow j},\, \zeta_{ij}^{i}}} &&J_i(\vartheta_i) + \frac{1}{\textcolor{black}{N_a}}\scriptE_{\hat\mu}(\tau_i) + \lambda_i^\intercal \xi_i + \frac{w}{2}\|\xi_i\|^2 + \frac{1}{2}\sum_{j \in \scriptN_i} \Big(\mu_{ij}^\intercal \zeta_{ij}^{i} + \frac{w}{2}\|\zeta_{ij}^{i}\|^{2} \Big)\label{eq:node-objective}\\
        &\quad\quad\ \ \mathrm{s.t.} && \Pi^{\mathrm{cvx}}(\vartheta_i) = \xi_i, \quad \vartheta_i\in\Theta_i, \label{eq:node-local-constraints}\\
        &&&\pi_{ij}^{\mathrm{cvx}}(\vartheta_i, \vartheta_j) = \zeta_{ij}^{i}, \quad \forall j\in\scriptN_i, \label{eq:node-pair-shooting}\\
        &&&\|\vartheta_i - \hat{\vartheta}_i\|_{\infty} \leq \Delta_{\mathrm{tr}}, \label{eq:node-trust-region}\\
        &&&\zeta_{ij}^{i} - \zeta_{ij}^{j} = 0, \quad \forall j \in \scriptN_i, \label{eq:node-zeta-consensus}\\
        &&&\textcolor{black}{Y_{ij}^{i}-Y_{ij}^{j}=0,\quad\forall j\in\scriptN_i}, \label{eq:node-y-consensus}\\
        &&&f_{i\rightarrow j} + f_{j\rightarrow i} = 0, \quad \forall j \in \scriptN_i, \label{eq:node-flow-antisymmetry}\\
        &&&\tau_i-\tau_j=0, \quad \forall j\in\scriptN_i, \label{eq:node-tau-consensus}\\
        &&&\sum_{j\in\scriptN_i} f_{i\rightarrow j} - \frac{1}{\textcolor{black}{N_a}}(\textcolor{black}{A_i}\vartheta_i - \tau_i) = 0.\label{eq:node-flow-balance}
    \end{align}
\end{subequations}
\begin{prop}[Equivalence at consensus]\label{prop:uav-flow-equivalence}
On a connected graph, minimizing the sum of the local objectives in \eqref{prob:uav-node-separable} subject to all agent constraints is equivalent to \eqref{prob:uav-scvx-subproblem}.
At consensus, identifying the shared copies and eliminating the flows recovers the original problem, with the same optimal value and corresponding minimizers.
\end{prop}
\begin{proof}
Take a point satisfying \eqref{prob:uav-node-separable} for every agent.
The coefficient agreement in \eqref{eq:flow-tau-consensus}, together with connectedness of $\scriptG$, gives $\tau_i=\tau$ for all $i\in\scriptV$.
Summing the local balance equations \eqref{eq:flow-local-balance} over all agents therefore gives
\[
    \sum_{i\in\scriptV}\sum_{j\in\scriptN_i}f_{i\rightarrow j}
    =\frac{1}{N_a}\sum_{i\in\scriptV}(A_i\vartheta_i-\tau_i)
    =\frac{1}{N_a}\sum_{i\in\scriptV}A_i\vartheta_i-\tau.
\]
Each unordered edge contributes both directed flows to the left-hand side.
Grouping these terms by edge and applying antisymmetry \eqref{eq:flow-antisymmetry} yields
\[
    \sum_{i\in\scriptV}\sum_{j\in\scriptN_i}f_{i\rightarrow j}
    =\sum_{\{i,j\}\in\scriptE}(f_{i\rightarrow j}+f_{j\rightarrow i})=0.
\]
Combining the two identities gives $\tau=N_a^{-1}\sum_{i\in\scriptV}A_i\vartheta_i$, which is exactly \eqref{eq:uav-scvx-global-sum}.

The virtual-control agreement in \eqref{eq:node-zeta-consensus} identifies a single $\zeta_{ij}$ for each pair, so \eqref{eq:node-pair-shooting} becomes \eqref{eq:joint-pair-shooting}.
The local constraints \eqref{eq:node-local-constraints}, trust regions \eqref{eq:node-trust-region}, and accumulator agreement \eqref{eq:node-y-consensus} reproduce \eqref{eq:joint-local-constraints}, \eqref{eq:joint-trust-region}, and \eqref{eq:joint-y-consensus}, respectively.
Thus all constraints of \eqref{prob:uav-scvx-subproblem} are recovered.
For the objective, coefficient consensus gives
\[
    \sum_{i\in\scriptV}\frac{1}{N_a}\scriptE_{\hat\mu}(\tau_i)
    =\scriptE_{\hat\mu}(\tau).
\]
Each unordered edge appears twice in the neighbor sums in \eqref{eq:node-objective}.
At virtual-control consensus, its two identical half-weighted terms sum to the single pair penalty in \eqref{eq:joint-objective}.
The local costs $J_i$ and penalties on $\xi_i$ each appear once, so the sum of the objectives in \eqref{eq:node-objective} equals \eqref{eq:joint-objective}.
\end{proof}
\endgroup
\begingroup\color{black}
\hypersetup{citecolor=black,linkcolor=black,urlcolor=black}
The remaining neighbor coupling lies in the pairwise shooting constraint \eqref{eq:node-pair-shooting} and the agreement and antisymmetry constraints \eqref{eq:node-zeta-consensus}--\eqref{eq:node-tau-consensus}.
To separate these constraints, first collect agent $i$'s decision variables as
\begin{equation*}
    \nu_i\triangleq\operatorname{col}\!\left(
    \vartheta_i,\xi_i,\tau_i,
    \{f_{i\rightarrow j}\}_{j\in\scriptN_i},
    \{\zeta_{ij}^{i}\}_{j\in\scriptN_i}\right)
    \in\R^{d_{\nu,i}},\qquad\forall i\in\scriptV,
\end{equation*}
where $d_{\nu,i}=d_{\vartheta,i}+N(n+|\scriptK|)+|\scriptK|+|\scriptN_i|(N+|\scriptK|)$.
Let $\scriptJ_i(\nu_i)$ denote the local objective in \eqref{eq:node-objective}.
The local constraints \eqref{eq:node-local-constraints}, \eqref{eq:node-trust-region}, and \eqref{eq:node-flow-balance} define the convex set
\begin{equation}\label{eq:uav-local-set}
    N_i\triangleq\left\{\nu_i:\begin{aligned}
        &\Pi^{\mathrm{cvx}}(\vartheta_i)=\xi_i,\quad\vartheta_i\in\Theta_i,\quad
        \|\vartheta_i-\hat\vartheta_i\|_\infty\leq\Delta_{\mathrm{tr}},\\
        &\sum_{j\in\scriptN_i}f_{i\rightarrow j}
        =\frac{1}{N_a}(A_i\vartheta_i-\tau_i)
    \end{aligned}\right\},\qquad\forall i\in\scriptV.
\end{equation}
Collect all pairwise equality constraints \eqref{eq:node-pair-shooting} and \eqref{eq:node-zeta-consensus}--\eqref{eq:node-tau-consensus} in a single affine map as
\begin{equation}\label{eq:uav-coupling-map}
    \widehat H_{ij}(\nu_i,\nu_j)\triangleq
    \begin{bmatrix}
        \pi_{ij}^{\mathrm{cvx}}(\vartheta_i,\vartheta_j)-\zeta_{ij}^{i}\\
        Y_{ij}^{i}-Y_{ij}^{j}\\
        \zeta_{ij}^{i}-\zeta_{ij}^{j}\\
        f_{i\rightarrow j}+f_{j\rightarrow i}\\
        \tau_i-\tau_j
    \end{bmatrix},\qquad\forall i\in\scriptV,\quad\forall j\in\scriptN_i.
\end{equation}
Summing the local objectives and imposing all agent constraints gives
\begin{equation}\label{prob:uav-node-coupled}
    \begin{aligned}
        \min_{\{\nu_i\}_{i\in\scriptV}}\quad
        &\sum_{i\in\scriptV}\scriptJ_i(\nu_i),\\
        \mathrm{s.t.}\quad&\nu_i\in N_i,\quad\forall i\in\scriptV,\\
        &\widehat H_{ij}(\nu_i,\nu_j)=0,
        \quad\forall i\in\scriptV,\quad\forall j\in\scriptN_i.
    \end{aligned}
\end{equation}
To apply graph C-ADMM from Section~\ref{subsec:C-ADMM}, introduce at agent $i$ a copy $w_i\in\R^{d_{\nu,i}}$ of its own vector $\nu_i$ and a copy $w_{ij}\in\R^{d_{\nu,j}}$ of each neighbor's vector $\nu_j$.
Evaluate the coupling map on these copies and enforce agreement with the agent vectors:
\begin{subequations}\label{prob:uav-consensus-copies}
    \begin{empheq}[box=\fbox]{align}
        \min_{\substack{\{\nu_i,w_i\}_{i\in\scriptV}\\
                        \{w_{ij}\}_{i\in\scriptV,\,j\in\scriptN_i}}}\quad
        &\sum_{i\in\scriptV}\scriptJ_i(\nu_i),\\
        \mathrm{s.t.}\quad&\nu_i\in N_i,\quad\forall i\in\scriptV,\\
        &\widehat H_{ij}(w_i,w_{ij})=0,
        \quad\forall i\in\scriptV,\quad\forall j\in\scriptN_i,\\
        &\nu_i=w_i,\quad\nu_j=w_{ij},
        \quad\forall i\in\scriptV,\quad\forall j\in\scriptN_i.
    \end{empheq}
\end{subequations}
At copy agreement, substitution of $w_i=\nu_i$ and $w_{ij}=\nu_j$ recovers \eqref{prob:uav-node-coupled}.
Within each outer SCvx* iteration, graph C-ADMM updates the agent vectors, the local copies, and the dual variables through neighbor communication, as described in Section~\ref{subsec:C-ADMM}.
The resulting agent vectors supply the candidate update for the outer SCvx* reference.
\begin{rmk}[Computational complexity]
For fixed state and control dimensions and a fixed number of spectral modes, agent $i$'s primal vector has $d_{\nu,i}=O\!\left(N(1+|\scriptN_i|)\right)$ entries.
Thus, with a uniformly bounded number of neighbors, the local primal problem size grows linearly with the horizon and is independent of the total fleet size.
Total runtime also depends on the solver and ADMM iteration counts, the consensus updates, and communication.
Runtime can be improved by exploiting temporal sparsity and warm-starting successive solves~\citep{wang2010fastmpc}, while graph-based decomposition enables parallel local computations~\citep{boyd2011admm}.
\end{rmk}
\endgroup

\section{RFS Bayesian Inference}\label{sec:rfs-inference}
At any given time $t_k$, we use a Poisson multi-Bernoulli (PMB) RFS to approximate the posterior distribution of targets in the domain~\citep{garciafernandez2018pmbm,williams2015marginal}.
In particular, the PPP RFS is used to capture \textit{undetected} objects in the domain and is given by $\rv X^u \sim \mathrm{PPP}(\lambda)$.
Repeated sensing of a particular area without any evidence of a target would naturally decay the posterior PPP field and encourage the UGV to plan its resulting trajectory in those areas.
Similarly, the MB RFS is used to capture already \textit{detected} targets, and each detected track $\ell$ admits the Bernoulli RFS $\rv B_{\ell} \sim \mathrm{BernoulliRFS}(r_{\ell}, \rho_{\ell})$.
Here, we take the single-target state space to be $\scriptX_t=\R^2$ and specialize the conditional target belief to a full-space Gaussian, i.e., $\rv x_{\ell} \mid |\rv B_{\ell}| = 1 \sim \scriptN(\mu_{\ell}, P_{\ell})$.
{The initial PPP intensity is supported on $\Omega$, and $p^D$ is extended by zero outside $\Omega$.}
The existence probability $r_{\ell}$ quantifies whether the current hypothesis corresponds to a real target, while $\scriptN(\mu_{\ell},P_{\ell})$ quantifies where that target is, if it exists.
Hence, at time step $k$, the posterior RFS is approximated by
\begin{equation}
    \rv X_k = \rv X_k^u \sqcup \bigsqcup_{\ell=1}^{L_k} \rv B_{k,\ell},
\end{equation}
where $L_k \geq 0$ denotes the number of detected (potential) target tracks.

Given this ansatz of the PMB posterior form, let us walk through a cycle of the RFS filter, consisting of both a time update and measurement update, just as with a traditional Kalman filter.
Starting at sensing epoch $k$, the UAVs have access to the belief snapshot
\begin{equation}
    \scriptB_k = \left(\lambda_k, \{r_{k,\ell}, \mu_{k,\ell}, P_{k,\ell}\}_{\ell=1}^{L_k}\right).
\end{equation}
\begingroup\color{black}
\hypersetup{citecolor=black,linkcolor=black,urlcolor=black}
For prediction, assume independent target motion, unit survival probability, and no births.
Let $p_k^{\mathrm{tr}}(\xi\mid\xi')$ denote the transition density from position $\xi'$ at $t_{k-1}$ to position $\xi$ at $t_k$.
We assume this motion model applies to both detected and undetected targets.
The Bernoulli existence parameters remain $r_{k,\ell}^{-}=r_{k-1,\ell}^{+}$, while the target densities and undetected-target intensity propagate as~\citep{garciafernandez2018pmbm}
\begin{equation*}
    \begin{aligned}
        \rho_{k,\ell}^{-}(\xi)
        &=\int_{\scriptX_t}p_k^{\mathrm{tr}}(\xi\mid\xi')\rho_{k-1,\ell}^{+}(\xi')\,\d\xi',
        \quad\forall\ell\in\{1,\ldots,L_{k-1}\},\\
        \lambda_k^{-}(\xi)
        &=\int_{\scriptX_t}p_k^{\mathrm{tr}}(\xi\mid\xi')\lambda_{k-1}^{+}(\xi')\,\d\xi',
        \quad\forall\xi\in\scriptX_t.
    \end{aligned}
\end{equation*}
Let $\rv x_{k,\ell}\in\R^2$ denote the random position of target $\ell$ at $t_k$, conditional on its existence.
For an affine Gaussian motion model, take
\begin{equation*}
    \rv x_{k,\ell}=F_k^{\mathrm t}\rv x_{k-1,\ell}+b_k^{\mathrm t}+\eta_{k,\ell},
    \qquad\eta_{k,\ell}\sim\scriptN(0,Q_k),
\end{equation*}
where $F_k^{\mathrm t}\in\R^{2\times2}$ is the transition matrix, $b_k^{\mathrm t}\in\R^2$ is a known displacement, and the process noise is independent of the preceding target state, with $Q_k\succeq0$.
For $Q_k\succ0$, this model gives the transition density
\begin{equation*}
    p_k^{\mathrm{tr}}(\xi\mid\xi')
    =\scriptN\!\left(\xi;F_k^{\mathrm t}\xi'+b_k^{\mathrm t},Q_k\right),
    \qquad\xi,\xi'\in\scriptX_t.
\end{equation*}
Substituting this transition density into the target-density prediction gives
\begin{equation}\label{eq:target-motion-prediction}
    \begin{aligned}
        \mu_{k,\ell}^{-}&=F_k^{\mathrm t}\mu_{k-1,\ell}^{+}+b_k^{\mathrm t},\\
        P_{k,\ell}^{-}&=F_k^{\mathrm t}P_{k-1,\ell}^{+}(F_k^{\mathrm t})^\intercal+Q_k,
        \qquad\forall\ell\in\{1,\ldots,L_{k-1}\}.
    \end{aligned}
\end{equation}
\begin{rmk}[Special cases of target motion]
Setting $F_k^{\mathrm t}=I_2$ and $b_k^{\mathrm t}=0$ gives the random-walk model with covariance inflation $Q_k$.
If also $Q_k=0$, the targets are static, the transition is a point mass at $\xi'$, and both $\rho_{k,\ell}^{-}=\rho_{k-1,\ell}^{+}$ and $\lambda_k^{-}=\lambda_{k-1}^{+}$.
For nonlinear motion, the same prediction integrals apply, but retaining Gaussian target beliefs generally requires a moment approximation.
\end{rmk}
\endgroup

Given the sensor probability of detection along with its measurement model from Section~\ref{subsec:sensor-model}, we can naturally encompass both these aspects into a unified observation RFS.
To see this, note that an object measurement only happens \textit{conditional} on its detection.
Hence, this phenomenon is precisely described by a Bernoulli RFS; letting $\rv Y_{i,k}(\xi)$ denote the target-generated measurement RFS for a target at position $\xi \in \scriptX_t$ (with $p^D_{i,k}(\xi)=0$ outside $\Omega$), we have 
\begin{equation}
    \rv Y_{i,k}(\xi) = 
    \begin{cases}
        \varnothing, &\mathrm{w.p.}\ 1 - p^D_{i,k}(\xi), \\
        \{\rv Z\},\, \rv Z \sim g_{i,k}(\cdot \mid \xi) &\mathrm{w.p.}\ p^D_{i,k}(\xi),
    \end{cases}
\end{equation}
\begingroup\color{black}
\hypersetup{citecolor=black,linkcolor=black,urlcolor=black}
Here $g_{i,k}(z\mid\xi)$ is the measurement density conditional on detecting a target at $\xi$.
From the angles-only model \eqref{eq:meas-model}, it is
\begin{equation*}
    g_{i,k}(z\mid\xi)
    =\scriptN\!\left(z;h(\xi;x_i(t_k)),R(x_i(t_k)\mid\xi)\right),\qquad z\in\R^2.
\end{equation*}
The detection probability $p^D_{i,k}(\xi)$ is accounted for separately in the Bernoulli observation model above.
\endgroup
In vision-based sensors, it is also possible to obtain spurious measurements (e.g., false positives), which we call clutter, and we model this phenomenon as $\rv C_{i,k} \sim \mathrm{PPP}(\kappa_{i,k})$ in measurement space, with a finite intensity function $\kappa_{i,k}$.
Assuming independent clutter and target-generated measurement RFSs, the complete sensor set at sensing epoch $k$ is their RFS superposition
\begin{equation}\label{eq:measurement-rfs}
    \rv Z_{i,k} = \rv C_{i,k} \sqcup \bigsqcup_{\xi \in \rv X_k^-} \rv Y_{i,k}(\xi).
\end{equation}
Notice that due to this construction, \eqref{eq:measurement-rfs} is necessarily a random union; conditioned on one target-set realization $X = \{x_1,\ldots, x_n\}$, the random union contains exactly $n$ target-generated Bernoulli RFSs.
For clarity, the update below is written for one sensor and suppresses the indices $(i,k)$, so $p^D\equiv p^D_{i,k}$, $g\equiv g_{i,k}$, and $\kappa\equiv\kappa_{i,k}$; multiple conditionally independent sensors may be incorporated through sequential updates or an equivalent joint likelihood. Let the corresponding observed measurement set be $Z_k = \{z_1, \ldots, z_m\}$, write $m=m_k$, and let $L$ denote the number of predicted existing tracks at this update.
The local update construction consists of three ingredients, followed by association and marginalization.

\subsection{Posterior component updates}
First, we have to update the existing parameters of the PPP and MB RFSs with the information of the new measurement $Z_k$.
For the PPP, we perform an operation called ``thinning", whereby we subtract intensity from the PPP in exactly the places where detections did not occur.
Explicitly, define $q^D(\xi) \triangleq 1 - p^D(\xi)$ as the probability of \textit{not} detecting a target at $\xi\in\scriptX_t$.
A target remains in the never-detected population only if it is missed, hence the identity
\begin{equation}
    \lambda_k^+(\xi) = q^D(\xi) \lambda_k^-(\xi).
\end{equation}
This is the independent-thinning property of a PPP~\citep{lastpenrose2017poisson}, and an example is illustrated in Figure~\ref{fig:ppp-thinning} for a simple lawnmower sweep pattern over a uniform prior intensity $\lambda^- \sim U(\Omega)$.
\begin{figure}[H]
    \centering
    \includegraphics[width=\linewidth]{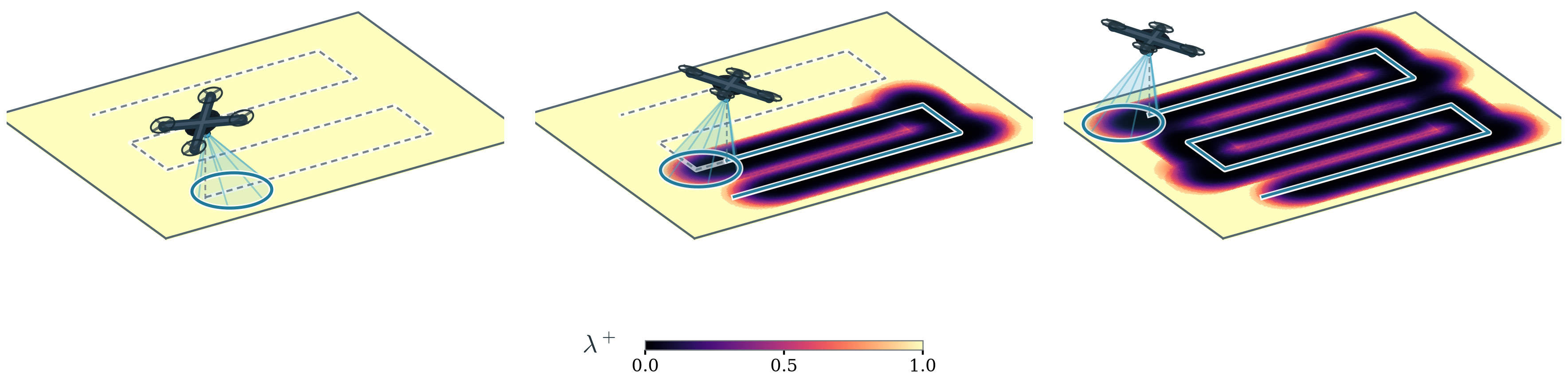}
    \caption{Undetected target intensity $\lambda^+(\xi)$ Poisson point process thinning during UAV execution and sensing.}
    \label{fig:ppp-thinning}
\end{figure}

\begingroup\color{black}
\hypersetup{citecolor=black,linkcolor=black,urlcolor=black}
For each existing track $\ell$, let $E_\ell$ be the event that the target exists and $D_\ell$ the event that it is detected by the current sensor.
Conditional on $E_\ell$, its position has density $\rho_\ell^-$, and $\P(E_\ell)=r_\ell^-$.
The detection model gives $\P(D_\ell\mid E_\ell,\rv x_\ell=\xi)=p^D(\xi)$.
All branch probabilities below are conditioned on the prior information available before this measurement update.

The local hypothesis $\mathcal A_{\ell,0}$ assigns no measurement to track $\ell$.
It is the disjoint union
\begin{equation*}
    \mathcal A_{\ell,0}=E_\ell^c\sqcup(E_\ell\cap D_\ell^c),
\end{equation*}
because the target is either absent or present but missed.

\sbox{\rfsBranchesImageBox}{\includegraphics[width=0.61\textwidth]{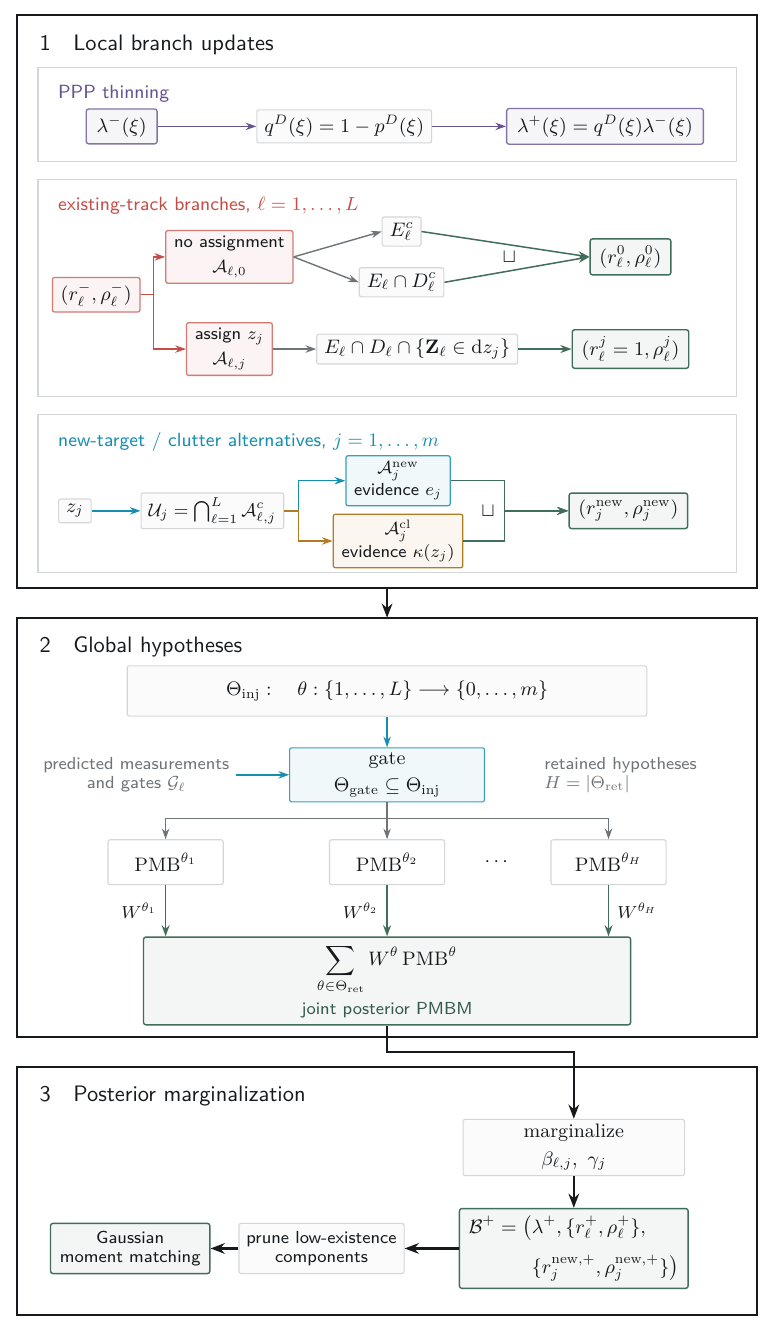}}
\Needspace{\dimexpr\ht\rfsBranchesImageBox+\dp\rfsBranchesImageBox+4\baselineskip\relax}
\noindent\begin{minipage}[t]{0.36\textwidth}
\vspace{0pt}
For $j\in\{1,\ldots,m\}$, the local hypothesis $\mathcal A_{\ell,j}$ assigns measurement $z_j$ to track $\ell$.
Let $\rv Z_\ell$ denote the target-generated measurement vector, defined when $E_\ell\cap D_\ell$ occurs.
In differential event notation, this hypothesis is
\begin{equation*}
    \mathcal A_{\ell,j}=E_\ell\cap D_\ell\cap\{\rv Z_\ell\in\d z_j\}.
\end{equation*}
The notation $\rv Z_\ell\in\d z_j$ represents a measurement in an infinitesimal neighborhood of $z_j$, so its evidence is a density with respect to $\d z_j$.
For an observed measurement $z_j$, let $\mathcal U_j$ be the hypothesis that no existing track generated it.
Let $\mathcal A_j^{\mathrm{new}}$ denote the hypothesis that it is the first detection of a previously undetected target, and let $\mathcal A_j^{\mathrm{cl}}$ denote the hypothesis that it is clutter.
These source hypotheses distinguish the two ways an unassigned measurement can arise.
Figure~\ref{fig:rfs-measurement-update-branches} summarizes the local event decompositions, global association, and posterior marginalization.
We next compute the evidence for each local hypothesis and apply Bayes' rule to obtain its posterior parameters.
\end{minipage}\hfill
\begin{minipage}[t]{0.61\textwidth}
    \vspace{0pt}\centering
    \includegraphics[width=\linewidth]{figures/illustrations/rfs-measurement-update-branches.pdf}
    \par
    \makeatletter
    \def\@captype{figure}
    \long\def\@makecaption#1#2{\vskip 4pt\nointerlineskip
        \vbox{\raggedright\noindent #1: #2\par}}
    \makeatother
    \caption{{RFS measurement update for one sensor.}}
    \label{fig:rfs-measurement-update-branches}
\end{minipage}
\par\medskip
The conditional probability of a missed detection is
\begin{equation}
    \begin{aligned}
        \bar q_\ell^D&=\P(D_\ell^c\mid E_\ell)\\
        &=\int_{\scriptX_t}q^D(\xi)\rho_\ell^-(\xi)\,\d\xi.
    \end{aligned}
\end{equation}
Adding the probabilities of the two disjoint events gives the branch evidence
\begin{equation}
    \begin{aligned}
        a_{\ell,0}&=\P(\mathcal A_{\ell,0})\\
        &=\P(E_\ell^c)+\P(E_\ell\cap D_\ell^c)\\
        &=(1-r_\ell^-)+r_\ell^-\bar q_\ell^D.
    \end{aligned}
\end{equation}
Since $E_\ell\cap\mathcal A_{\ell,0}=E_\ell\cap D_\ell^c$, the joint contribution from a present but missed target is
\begin{equation*}
    \P(\rv x_\ell\in\d\xi,D_\ell^c\mid E_\ell)
    =q^D(\xi)\rho_\ell^-(\xi)\,\d\xi.
\end{equation*}
Bayes' rule then gives
\begin{equation}\label{eq:no-assignment-track-update}
    r_\ell^0=\P(E_\ell\mid\mathcal A_{\ell,0})
    =\frac{r_\ell^-\bar q_\ell^D}{a_{\ell,0}},\qquad
    \rho_\ell^0(\xi)=\frac{q^D(\xi)\rho_\ell^-(\xi)}{\bar q_\ell^D}.
\end{equation}
Here $\rho_\ell^0$ is the position density conditional on $E_\ell\cap\mathcal A_{\ell,0}$.
Branches with zero evidence are omitted.
If $\bar q_\ell^D=0<a_{\ell,0}$, the branch has zero existence probability and does not require a conditional position density.

For the assignment hypothesis $\mathcal A_{\ell,j}$, the joint state and measurement contribution is
\begin{equation*}
    \P(\rv x_\ell\in\d\xi,\mathcal A_{\ell,j})
    =r_\ell^-\rho_\ell^-(\xi)p^D(\xi)g(z_j\mid\xi)\,\d\xi\,\d z_j.
\end{equation*}
Integrating out the target position gives $\P(\mathcal A_{\ell,j})=a_{\ell,j}\,\d z_j$, where
\begin{equation}
    L_{\ell j}=\int_{\scriptX_t}p^D(\xi)g(z_j\mid\xi)\rho_\ell^-(\xi)\,\d\xi,
    \qquad a_{\ell,j}=r_\ell^-L_{\ell j}.
\end{equation}
Since $\mathcal A_{\ell,j}\subseteq E_\ell$, the target exists under this hypothesis.
Dividing the joint contribution by the branch evidence gives
\begin{equation}\label{eq:assigned-track-update}
    r_\ell^j=\P(E_\ell\mid\mathcal A_{\ell,j})=1,\qquad
    \rho_\ell^j(\xi)=\frac{p^D(\xi)g(z_j\mid\xi)\rho_\ell^-(\xi)}{L_{\ell j}}.
\end{equation}
Thus, the posterior accounts for the intersection of target existence, detection, and the observed measurement.

The new-target and clutter hypotheses are disjoint and exhaust $\mathcal U_j$, so
\begin{equation*}
    \mathcal U_j=\bigcap_{\ell=1}^{L}\mathcal A_{\ell,j}^{c}
    =\mathcal A_j^{\mathrm{new}}\sqcup\mathcal A_j^{\mathrm{cl}},
    \qquad\mathcal A_j^{\mathrm{new}}\cap\mathcal A_j^{\mathrm{cl}}=\varnothing.
\end{equation*}
The undetected-target intensity $\lambda_k^-(\xi)$, detection probability $p^D(\xi)$, and measurement density $g(z_j\mid\xi)$ give the joint new-target evidence density
\begin{equation*}
    p^D(\xi)g(z_j\mid\xi)\lambda_k^-(\xi).
\end{equation*}
Integrating over the possible target positions gives the evidence for $\mathcal A_j^{\mathrm{new}}$,
\begin{equation}
    e_j=\int_{\scriptX_t}p^D(\xi)g(z_j\mid\xi)\lambda_k^-(\xi)\,\d\xi.
\end{equation}
The evidence for $\mathcal A_j^{\mathrm{cl}}$ is $\kappa(z_j)$.
By Poisson superposition, these source intensities add to give the local evidence $e_j+\kappa(z_j)$ for an unassigned measurement~\citep{garciafernandez2018pmbm}.
Bayes' rule therefore gives the source probabilities
\begin{equation*}
    \P(\mathcal A_j^{\mathrm{new}}\mid\mathcal U_j,z_j)
    =\frac{e_j}{e_j+\kappa(z_j)},\qquad
    \P(\mathcal A_j^{\mathrm{cl}}\mid\mathcal U_j,z_j)
    =\frac{\kappa(z_j)}{e_j+\kappa(z_j)}.
\end{equation*}
A new target exists only under $\mathcal A_j^{\mathrm{new}}$.
Its Bernoulli existence probability and its position density conditional on that hypothesis are therefore
\begin{equation}
    r_j^{\mathrm{new}}=\frac{e_j}{e_j+\kappa(z_j)},\qquad
    \rho_j^{\mathrm{new}}(\xi)=\frac{p^D(\xi)g(z_j\mid\xi)\lambda_k^-(\xi)}{e_j}.
\end{equation}
If $e_j=0<\kappa(z_j)$, only the clutter hypothesis remains and the conditional target density is immaterial.
If $e_j+\kappa(z_j)=0$, the unassigned-measurement branch is omitted.
Each measurement defines a candidate new Bernoulli component whose final existence probability also depends on whether the global association leaves it unassigned.
\endgroup

\begingroup\color{black}
\hypersetup{citecolor=black,linkcolor=black,urlcolor=black}
Before constructing global hypotheses, we discard track--measurement pairs that fail a statistical gate.
For a Gaussian track $\ell$, define the predicted measurement $\hat z_\ell$, measurement Jacobian $H_\ell$, and innovation covariance $S_\ell$ by
\begin{equation*}
    \hat z_\ell=h(\mu_\ell^-),\qquad
    H_\ell=\left.\frac{\partial h}{\partial\xi}\right|_{\mu_\ell^-},\qquad
    S_\ell=H_\ell P_\ell^-H_\ell^\intercal+R.
\end{equation*}
Here the sensor pose is fixed, and $R$ is evaluated at the predicted target position $\mu_\ell^-$.
Under the local Gaussian approximation, the squared normalized innovation
\begin{equation*}
    d_{\ell j}^2\triangleq(z_j-\hat z_\ell)^\intercal S_\ell^{-1}(z_j-\hat z_\ell)
\end{equation*}
is approximately $\chi^2_{n_z}$ distributed when $z_j$ is generated by track $\ell$, with $n_z=2$ for the two-angle measurement~\citep{barshalom2001estimation}.
For a gate probability $1-\alpha_{\mathrm{gate}}$, retain the measurement indices
\begin{equation*}
    \mathcal G_\ell\triangleq\bigl\{j\in\{1,\ldots,m\}:d_{\ell j}^2\leq\chi_{n_z}^2(1-\alpha_{\mathrm{gate}})\bigr\}.
\end{equation*}
The threshold $\chi_{n_z}^2(1-\alpha_{\mathrm{gate}})$ is the corresponding chi-square quantile.
Each track retains $\mathcal A_{\ell,0}$ and the assignment hypotheses $\mathcal A_{\ell,j}$ with $j\in\mathcal G_\ell$.
\endgroup

\subsection{Global Association Hypotheses}

A global hypothesis is a map 
\begin{equation*}
    \theta: \{1,\ldots, L\} \rightarrow \{0,1,\ldots, m\},
\end{equation*}
{which selects $\mathcal A_{\ell,0}$ when $\theta(\ell)=0$ and $\mathcal A_{\ell,j}$ when $\theta(\ell)=j$.}
A structurally admissible map has positive assignments that are \textit{injective}%
\footnote{Formally, if $\ell_1\neq\ell_2$ and $\theta(\ell_1),\theta(\ell_2)>0$, then $\theta(\ell_1)\neq\theta(\ell_2)$.}, meaning a measurement $z_j$ cannot be assigned to two existing tracks simultaneously.
{\color{black}
Let $\Theta_{\mathrm{inj}}$ denote the set of all such positive-injective maps and define its statistically gated subset as
\begin{equation*}
    \Theta_{\mathrm{gate}}\triangleq
    \bigl\{\theta\in\Theta_{\mathrm{inj}}:\theta(\ell)\in\{0\}\cup\mathcal G_\ell,\quad\forall\ell\in\{1,\ldots,L\}\bigr\}.
\end{equation*}
If computational truncation is required, let $\Theta_{\mathrm{ret}}\subseteq\Theta_{\mathrm{gate}}$ be the retained bounded set of high-weight hypotheses; otherwise, set $\Theta_{\mathrm{ret}}=\Theta_{\mathrm{gate}}$.
}
Define $U(\theta) \triangleq \{1,\ldots, m\} \setminus \{\theta(\ell) : \theta(\ell) > 0\}$ as the subset of measurements unassigned to existing tracks.
{For a fixed association $\theta$, independence of the prior components gives the product of the selected local evidences.
The common Poisson factor cancels when these weights are normalized, giving}
{\color{black}
\begin{equation}
    \begin{aligned}
        \tilde{W}^{\theta}
        &=\prod_{\ell=1}^{L}a_{\ell,\theta(\ell)}
        \prod_{j\in U(\theta)}\bigl(e_j+\kappa(z_j)\bigr),\\[3pt]
        W^\theta
        &=\frac{\tilde W^\theta}{\sum_{\theta'\in\Theta_{\mathrm{ret}}}\tilde W^{\theta'}},
        \qquad\theta\in\Theta_{\mathrm{ret}}.
    \end{aligned}
\end{equation}
}
{\color{black}
When $\Theta_{\mathrm{ret}}\subsetneq\Theta_{\mathrm{gate}}$, the normalized weights are conditional on the retained set.
Each retained association hypothesis $\theta$ defines one PMB posterior RFS; their weighted sum is generally a PMB \textit{mixture} (PMBM).
The final step marginalizes this retained PMBM to obtain the PMB posterior approximation.
}

\subsection{Marginalization}
For each existing track $\ell=1,\ldots, L$, define the weights
{\color{black}
\begin{equation*}
    \beta_{\ell,0} = \sum_{\substack{\theta\in\Theta_{\mathrm{ret}}:\, \theta(\ell) = 0}} W^{\theta}, \quad \beta_{\ell,j} = \sum_{\substack{\theta\in\Theta_{\mathrm{ret}}:\, \theta(\ell) = j}} W^{\theta},
\end{equation*}
}
as the aggregate weights corresponding to no measurement associations, and measurement-$j$ associations.
For each existing track, $\beta_{\ell,0}+\sum_{j=1}^{m}\beta_{\ell,j}=1$.
Likewise, for each measurement $j=1,\ldots,m$, define 
{\color{black}
\begin{equation*}
    \gamma_j = \sum_{\substack{\theta\in\Theta_{\mathrm{ret}}:\, j \in U(\theta)}} W^{\theta},
\end{equation*}
}
as the probability that it is unused by all existing tracks.
The next result gives the posterior marginalized PMB approximation of the resulting PMBM.
\begin{prop}
    The marginalized posterior MB RFS of existing and new tracks is
    \begin{align}
        r_{\ell}^+ &= \beta_{\ell,0} r_{\ell}^{0} + \sum_{j=1}^{m} \beta_{\ell,j}, \quad r_j^{\mathrm{new},+} = \gamma_j \frac{e_j}{e_j + \kappa(z_j)}, \label{eq:marginal-existence}\\
        \rho_{\ell}^+(\xi) &= \frac{\beta_{\ell,0}r_{\ell}^{0}\rho_{\ell}^0(\xi)+\sum_{j=1}^{m} \beta_{\ell,j} \rho_{\ell}^{j}(\xi)}{r_{\ell}^{+}}, \quad \rho_j^{\mathrm{new},+} = \rho_j^{\mathrm{new}}.\label{eq:marginal-position-density}
    \end{align}
\end{prop}
\begingroup\color{black}
\hypersetup{citecolor=black,linkcolor=black,urlcolor=black}
\begin{proof}
All probabilities in this proof refer to the posterior mixture over the retained hypotheses, whose weights satisfy $\sum_{\theta\in\Theta_{\mathrm{ret}}}W^\theta=1$.
Under a fixed hypothesis $\theta$, existing track $\ell$ has the Bernoulli parameters $(r_\ell^{\theta(\ell)},\rho_\ell^{\theta(\ell)})$.
The law of total probability therefore gives
\begin{equation*}
    r_\ell^+=\sum_{\theta\in\Theta_{\mathrm{ret}}}W^\theta r_\ell^{\theta(\ell)}
    =\beta_{\ell,0}r_\ell^0+\sum_{j=1}^{m}\beta_{\ell,j},
\end{equation*}
where we grouped hypotheses by $\theta(\ell)$ and used $r_\ell^j=1$ from \eqref{eq:assigned-track-update} for every retained assignment branch.
This proves the existing-track existence formula in \eqref{eq:marginal-existence}.

For a Bernoulli component, $r\rho(\xi)\,\d\xi$ is the joint probability that the target exists and lies in $\d\xi$.
A branch with zero existence probability contributes zero to this joint density.
Averaging this joint probability over the same hypotheses gives
\begin{equation*}
    \begin{aligned}
        r_\ell^+\rho_\ell^+(\xi)
        &=\sum_{\theta\in\Theta_{\mathrm{ret}}}W^\theta
          r_\ell^{\theta(\ell)}\rho_\ell^{\theta(\ell)}(\xi)\\
        &=\beta_{\ell,0}r_\ell^0\rho_\ell^0(\xi)
          +\sum_{j=1}^{m}\beta_{\ell,j}\rho_\ell^j(\xi).
    \end{aligned}
\end{equation*}
When $r_\ell^+>0$, division by $r_\ell^+$ gives the existing-track density in \eqref{eq:marginal-position-density}.
The remaining probability is $1-r_\ell^+$ and corresponds to the empty set, so this marginal is a Bernoulli RFS.

Consider a candidate new component with $e_j+\kappa(z_j)>0$; candidates with zero local evidence are omitted as specified in the local update.
The component associated with $z_j$ is present only in hypotheses with $j\in U(\theta)$.
In each such hypothesis it has the same parameters $(r_j^{\mathrm{new}},\rho_j^{\mathrm{new}})$, and in every other hypothesis it is empty.
Consequently,
\begin{equation*}
    \begin{aligned}
        r_j^{\mathrm{new},+}
        &=\sum_{\substack{\theta\in\Theta_{\mathrm{ret}}:\,j\in U(\theta)}}
          W^\theta r_j^{\mathrm{new}}
          =\gamma_j r_j^{\mathrm{new}},\\
        r_j^{\mathrm{new},+}\rho_j^{\mathrm{new},+}(\xi)
        &=\gamma_j r_j^{\mathrm{new}}\rho_j^{\mathrm{new}}(\xi).
    \end{aligned}
\end{equation*}
Substituting $r_j^{\mathrm{new}}=e_j/(e_j+\kappa(z_j))$ proves the new-track existence formula in \eqref{eq:marginal-existence}.
For $r_j^{\mathrm{new},+}>0$, division in the second identity gives $\rho_j^{\mathrm{new},+}=\rho_j^{\mathrm{new}}$, as stated in \eqref{eq:marginal-position-density}.
Components with zero existence probability are empty almost surely and need no conditional position density.
Taking the resulting Bernoulli marginals to be independent gives the stated MB approximation, preserving the individual component marginals while discarding their association dependence.
\end{proof}
\endgroup
If $r_\ell^+=0$, the corresponding Bernoulli component is omitted and its conditional density need not be defined.
\begin{cor}
    Assume $r_\ell^+>0$. Approximate each retained local branch with a Gaussian $(\mu_{\ell,b}, P_{\ell,b})$, \textcolor{black}{set the weights of branches absent from the retained hypotheses to zero}, and define the existence-conditioned branch weights
    \begin{equation*}
        \alpha_{\ell,0}=\frac{\beta_{\ell,0}r_\ell^0}{r_\ell^+}, \qquad \alpha_{\ell,j}=\frac{\beta_{\ell,j}}{r_\ell^+}, \quad j=1,\ldots,m.
    \end{equation*}
    The posterior Gaussian belief state under moment matching is
    \begin{align}
        \mu_{\ell}^+ &= \sum_{b\in\{0,1,\ldots,m\}} \alpha_{\ell,b} \mu_{\ell,b}, \label{eq:matched-track-mean}\\
        P_{\ell}^+ &= \sum_{b\in\{0,1,\ldots,m\}} \alpha_{\ell, b} \left[ P_{\ell,b} + (\mu_{\ell,b} - \mu_{\ell}^{+})(\mu_{\ell,b} - \mu_{\ell}^{+})^\intercal \right].\label{eq:matched-track-covariance}
    \end{align}
\end{cor}
\begingroup\color{black}
\hypersetup{citecolor=black,linkcolor=black,urlcolor=black}
\begin{proof}
Conditioning the branch probabilities on target existence gives the weights $\alpha_{\ell,b}$ by Bayes' rule.
They are nonnegative, and \eqref{eq:marginal-existence} gives
\begin{equation*}
    \sum_{b=0}^{m}\alpha_{\ell,b}
    =\frac{\beta_{\ell,0}r_\ell^0+\sum_{j=1}^{m}\beta_{\ell,j}}{r_\ell^+}=1.
\end{equation*}
Terms with zero weight are omitted from the following sums, so those branches require no moments.
Substituting the Gaussian branch approximations into \eqref{eq:marginal-position-density} therefore gives the conditional Gaussian mixture
\begin{equation*}
    \rho_\ell^+(\xi)\approx\sum_{b=0}^{m}\alpha_{\ell,b}
    \scriptN(\xi;\mu_{\ell,b},P_{\ell,b}).
\end{equation*}
Integrating $\xi$ against this mixture gives
\begin{equation*}
    \mu_\ell^+=\sum_{b=0}^{m}\alpha_{\ell,b}
    \int_{\scriptX_t}\xi\,\scriptN(\xi;\mu_{\ell,b},P_{\ell,b})\,\d\xi
    =\sum_{b=0}^{m}\alpha_{\ell,b}\mu_{\ell,b},
\end{equation*}
which is \eqref{eq:matched-track-mean}.

Within branch $b$, write $\rv x_\ell-\mu_\ell^+=(\rv x_\ell-\mu_{\ell,b})+(\mu_{\ell,b}-\mu_\ell^+)$.
The first term has zero conditional mean, so the cross terms vanish when the centered outer product is averaged within that branch.
Thus its contribution to the covariance about $\mu_\ell^+$ is
\begin{equation*}
    \int_{\scriptX_t}(\xi-\mu_\ell^+)(\xi-\mu_\ell^+)^\intercal
    \scriptN(\xi;\mu_{\ell,b},P_{\ell,b})\,\d\xi
    =P_{\ell,b}+(\mu_{\ell,b}-\mu_\ell^+)(\mu_{\ell,b}-\mu_\ell^+)^\intercal.
\end{equation*}
Averaging these contributions with weights $\alpha_{\ell,b}$ gives \eqref{eq:matched-track-covariance}.
The Gaussian $\scriptN(\mu_\ell^+,P_\ell^+)$ therefore matches the first two moments of the branch mixture.
\end{proof}
\endgroup
\josh{}

In practice, a small $r_j^{\mathrm{new},+}$ indicates clutter or a measurement likely used by an existing track, and such components are pruned and discarded from the set.
Similarly, a decaying $r_{\ell}^{+}$ indicates a confirmed track that most likely does not exist, and is also pruned below a configurable threshold.
In summary, the output target distribution arising from each measurement event is a new PMB posterior given by
\begin{equation}
    \rv X_k^+ = \rv X_k^{u,+} \sqcup \bigsqcup_{\ell=1}^{L_k^+} \hat{\rv B}_{k,\ell}^+ \sqcup \bigsqcup_{j=1}^{m_k^+} \hat{\rv B}_{k,j}^{\mathrm{new},+},
\end{equation}
where $L_k^+ \leq L_{k-1}$ denotes the remaining existing Bernoulli tracks, and $m_k^+ \leq m_k$ denotes the retained new Bernoulli tracks.
The above approximation preserves each Bernoulli first-order marginal of the retained PMBM approximation but replaces the global association dependence with independence among the output Bernoulli RFSs~\citep{williams2015marginal}.

\section{UGV Risk-Aware Planning}~\label{sec:UGV-traj-opt}
The ground vehicles are tasked with safely navigating an unknown planar environment filled with targets.
At each planning instant, the UGV fleet receives the latest posterior RFS $\rv X_k$, containing the distribution of the already-detected target tracks, in the form of an MB RFS $\sqcup_{j=1}^{L_k} \rv B_{k,j},\, \rv B_{k,j} \sim \mathrm{BernoulliRFS}(r_{k,j}, \rho_{k,j})$ with Gaussian belief states $\rho_{k,j}(\xi) = \scriptN(\xi; \mu_{k,j}, P_{k,j})$, along with the posterior distribution of undetected targets, in the form of a PPP $\rv X_k^u \sim \mathrm{PPP}(\lambda_k)$ with associated intensity function $\lambda_k(\xi),\, \xi \in \Omega$.
Given this information, the UGV team's responsibilities are three-fold: (i) to safely plan around the detected threats; (ii) to enforce inter-agent collision avoidance constraints with neighbors, and (iii) to navigate through regions of the domain with low residual threat of unknown targets.
In the following, we rigorously formulate this risk-sensitive joint trajectory optimization problem and subsequently present a fully decentralized method akin to that of the UAV team in Section~\ref{sec:UAV-traj-opt}, based on SCP and C-ADMM, to seek candidate local solutions.

\subsection{UGV Problem Formulation}\label{subsec:ugv-problem}
At each planning epoch, consider the graph $\scriptG = (\scriptV, \scriptE)$ composed of UGV nodes and connected by edges if agent $i$ and agent $j$ can communicate with each other.
As before, we assume $\scriptG$ is connected, and that the communication radius $\Delta_{\mathrm{com}} > 0$ is large enough so that agents detect each other well before the need for coordination.
For each UGV $j = 1,\ldots,N_g$, we use an acceleration-unicycle model, defined by the state $x_j \triangleq [p_j^\intercal, \psi_j, v_j]^\intercal \in \R^{4}$, control input $u_j = [a_j, \omega_j]^\intercal\in\R^{2}$, and dynamics
\begingroup\color{black}
\hypersetup{citecolor=black,linkcolor=black,urlcolor=black}
\begin{equation}\label{eq:ugv-dynamics}
    \dot x_j=f_j^g(x_j,u_j)\triangleq
    \begin{bmatrix}v_j\cos\psi_j\\v_j\sin\psi_j\\\omega_j\\a_j\end{bmatrix},
    \qquad j\in\scriptV,
\end{equation}
\endgroup

over a planning horizon $t \in [0,T]$.
{Here $p_j\in\R^2$ is position, $\psi_j$ is heading, $v_j$ is speed, and $a_j$ and $\omega_j$ are the commanded acceleration and turn rate.}
We constrain the velocities and accelerations of the vehicle within some maximum bounds, and further define an equivalent physical UGV radius $R_j^g > 0$ for all $j =1,\ldots,N_g$, which must stay in the search domain $\Omega$.
Assuming $\Omega$ is convex, aggregating all of these convex constraints gives the constraint set
\begin{equation}
    \Xi_j \triangleq \left\{(x,u)\in\R^4\times\R^{2} :
    \begin{aligned}
        &0 \leq v \leq v_{j,\max},\quad |a| \leq a_{j,\max},\\
        &|\omega| \leq \omega_{j,\max},\quad p \in \Omega \ominus \B(0,R_j^g)
    \end{aligned}\right\}.
\end{equation}
For brevity, we will also sometimes use the shorthand $\Omega_j \triangleq \Omega \ominus\B(0,R_j^g)$ as the domain containment set for the UGV's position, and write $\Xi_j^x\triangleq\operatorname{proj}_x\Xi_j$ for the admissible-state projection.

\paragraph{Threat avoidance}
Since each detected target track is a Gaussian belief, we can consider the associated confidence set as in the following proposition; confidence-set approximations of this kind are standard in chance constrained obstacle avoidance~\citep{blackmore2011chance}.
\begin{prop}
    Given $\rv\xi\sim\scriptN(\mu,\Sigma)$ with mean $\mu\in\R^{d}$, covariance matrix $\Sigma \succ 0$, and $\delta\in(0,1)$, let $\chi_d^2(1-\delta)$ denote the $(1-\delta)$ quantile of a chi-squared random variable with $d$ degrees of freedom.
    The associated uncertainty set
    \begin{equation}
        \scriptC_{\delta} \triangleq \left\{\xi \in \R^d: (\xi-\mu)^\intercal \Sigma^{-1} (\xi - \mu) \leq \chi_{d}^2(1-\delta)\right\},
    \end{equation}
    contains exactly probability $1-\delta$, that is, $\P(\rv\xi\in\scriptC_{\delta}) = 1 - \delta$.
    We refer to $\scriptC_{\delta}$ as the $1-\delta$ confidence set of $\scriptN(\mu,\Sigma)$.
\end{prop}
For each threat $\ell = 1,\ldots, L_k$, conditional on its existence, we construct the associated confidence set $\scriptC_{\delta_\ell}$.
{Let $d_{j\ell}^{\mathrm{clr}}=R_j^g+R_\ell^{\mathrm t}+\Delta_{j\ell}^{\mathrm t}>0$, where $R_\ell^{\mathrm t}\geq0$ is the target radius and $\Delta_{j\ell}^{\mathrm t}\geq0$ is an additional clearance margin.}
We then aim to enforce the continuous-time constraints
\begin{equation}
    p_j(t) \notin \bigcup_{\ell=1}^{L_k}\left(\scriptC_{\delta_{\ell}}\oplus\B(0,d_{j\ell}^{\mathrm{clr}})\right), \quad j \in \scriptV,\, t \in [0,T].
\end{equation}
{A simultaneous confidence budget $\epsilon_{\mathrm{th}}\in(0,1)$ can be allocated by requiring $\sum_{\ell=1}^{L_k}r_{k,\ell}\delta_\ell\leq\epsilon_{\mathrm{th}}$.
By the union bound, the probability that any existing track lies outside its confidence set is then at most $\epsilon_{\mathrm{th}}$.}

\paragraph{Collision avoidance}
{\color{black}
As for the UAVs, we impose geometric separation between UGV pairs:
\begin{equation*}
    \|p_i(t)-p_j(t)\|_2\geq d_{ij},
    \qquad \{i,j\}\in\scriptE,
    \qquad t\in[0,T],
\end{equation*}
where $d_{ij}=R_i^g+R_j^g+\Delta>0$ includes the UGV safety radii and an additional margin.
The planning graph must again include every pair that could violate separation during the horizon.
We use the shared isoperimetric construction in Section~\ref{subsec:ctcs-separation} with $s_i=p_i$.
With $\Pi_p=[I_2\ \ 0_{2\times2}]$, define
\begin{equation*}
    g_{ij}^g(x_i,x_j)\triangleq1-\frac{\|\Pi_p x_i-\Pi_p x_j\|_2^2}{d_{ij}^2},\qquad
    g_{c,ij}^g(x_i,x_j)\triangleq[g_{ij}^g(x_i,x_j)]_+^2.
\end{equation*}
The separation constraints are equivalently represented by
\begin{equation}\label{eq:ugv-pair-accumulator}
    \dot y_{ij}=g_{c,ij}^g(x_i,x_j),\qquad y_{ij}(0)=y_{ij}(T)=0,
    \qquad\forall\{i,j\}\in\scriptE.
\end{equation}
Stacking the physical states and pair accumulators gives
\begin{equation*}
    z^g\triangleq\operatorname{col}(\{x_i\}_{i\in\scriptV},\{y_{ij}\}_{\{i,j\}\in\scriptE})
    \in\R^{d_g},\qquad d_g=4N_g+|\scriptE|.
\end{equation*}
For $u=\operatorname{col}(\{u_i\}_i)\in\R^{2N_g}$, define the augmented vector field and initial state by
\begin{equation*}
    \begin{aligned}
        f_{\mathrm{aug}}^g(z^g,u)&\triangleq
        \operatorname{col}(\{f_i^g(x_i,u_i)\}_{i\in\scriptV},\{g_{c,ij}^g(x_i,x_j)\}_{\{i,j\}\in\scriptE}),\\
        z_0^g&\triangleq\operatorname{col}(\{x_{0,i}\}_{i\in\scriptV},0_{|\scriptE|\times1}).
    \end{aligned}
\end{equation*}
}

\paragraph{Void Probability}\phantomsection\label{par:void-probability}
For a given UGV path $p(t),\, t \in [0,T]$, let $\Gamma(p) \triangleq \{p(t) \in \R^2 : t \in [0,T]\}$ be its geometric trace. 
Consider the swept corridor around the path, defined as
\begin{equation}
    \scriptC_{\rho}(p) = \{\xi \in \Omega : \mathrm{dist}\big(\xi, \Gamma(p)\big) \leq \rho\},
\end{equation}
where $\rho > 0$ is the corridor radius.
Given the posterior PPP intensity $\lambda_k$ at time step $k$, the expected PPP count in the corridor is
\begin{equation}
    \Lambda_k (p) = \int_{\scriptC_{\rho}(p)} \lambda_k(\xi)\, \d\xi.
\end{equation}
Similarly, for posterior Bernoulli component $j$, define
\begin{equation}
    q_{k,j}(p) = \int_{\scriptC_{\rho}(p)} \rho_{k,j}(\xi)\, \d\xi,
\end{equation}
from which it follows that $r_{k,j} q_{k,j}(p)$ is the probability that track $j$ exists and lies in the corridor.
Under the assumed PMB independence approximation, the probability that the corridor contains \textit{no} targets is
\begin{equation}\label{eq:pmb-void-probability}
    P_{0,k}(p) = \exp\big(-\Lambda_k(p)\big)\prod_{j=1}^{L_k}\big(1 - r_{k,j}q_{k,j}(p)\big).
\end{equation}
We call the probability in \eqref{eq:pmb-void-probability} the \textit{void probability}, as it represents the probability that there are no targets in the corridor under the current target belief $\scriptB_k$.
As a result, it is natural to want the UGV path to maximize this void probability, so that it navigates through well-explored areas of the domain that have a low likelihood of containing targets.
For optimization, however, it is more convenient to use the negative log-void form, that is,
\begin{equation}\label{eq:negative-log-void}
    \scriptH_{k}(p) = -\log P_{0,k}(p) = \Lambda_k(p) - \sum_j \log\big(1 - r_{k,j} q_{k,j}(p)\big),
\end{equation}
which is additive in all the individual component terms and is understood as $+\infty$ if any factor $1-r_{k,j}q_{k,j}(p)$ vanishes.
We may choose to either minimize $\scriptH_k(p)$ in the objective function (as a soft constraint) or alternatively enforce the chance constraint 
\begin{equation}
    \P(|\rv X_k \cap \scriptC_{\rho}(p)| = 0) \geq 1 - \epsilon \iff \scriptH_k(p) \leq -\log(1-\epsilon),
\end{equation}
with respect to a risk level $\epsilon \in (0, 0.5]$.
In this work, we choose the soft constraint version to avoid issues of infeasibility for a given risk level $\epsilon$.
For agent $i$, let $\scriptH_{k,i}(p_i)$ denote \eqref{eq:negative-log-void} evaluated on its corridor $\scriptC_{\rho_i}(p_i)$. 
Let $\scriptX_{i,T}^{\mathrm{goal}}$ denote a prescribed nonempty closed convex terminal mission set, which may encode a goal region, progress condition, or terminal state, and assume the given initial states satisfy the threat and inter-agent safe sets.
{Let $\beta>0$ and $R_i\in\R^{2\times2}$, $R_i\succ0$, weight the control effort.}
{Define the UGV running cost $\ell_i^g(u_i)\triangleq\frac{\beta}{2}u_i^\intercal R_i u_i$.}
In total, the multi-agent UGV trajectory optimization problem becomes
\begingroup\color{black}
\hypersetup{citecolor=black,linkcolor=black,urlcolor=black}
\begin{subequations}\label{prob:ugv-centralized}
\begin{empheq}[box=\fbox]{align}
    \min_{z^g,u}\quad &J_g(z^g,u)=\sum_{i=1}^{N_g}
    \left[\int_0^T\ell_i^g(u_i)\,\d t+\scriptH_{k,i}(p_i)\right],\\
    \mathrm{s.t.}\quad &\dot z^g=f_{\mathrm{aug}}^g(z^g,u),\quad z^g(0)=z_0^g,\quad\forall t\in[0,T],\\
    &y_{ij}(T)=0,\quad\forall\{i,j\}\in\scriptE,\\
    &(x_i(t),u_i(t))\in\Xi_i,\quad\forall i\in\scriptV,\quad\forall t\in[0,T],\\
    &p_i(t)\notin\bigcup_{\ell=1}^{L_k}\left(\scriptC_{\delta_\ell}\oplus\B(0,d_{i\ell}^{\mathrm{clr}})\right),
      \quad\forall i\in\scriptV,\quad\forall t\in[0,T],\\
    &x_i(T)\in\scriptX_{i,T}^{\mathrm{goal}},\quad\forall i\in\scriptV.
\end{empheq}
\end{subequations}

\endgroup
\subsection{UGV Problem Solution}\label{subsec:ugv-solution}
To solve the nonlinear, coupled optimization problem \eqref{prob:ugv-centralized}, we proceed exactly as in Section~\ref{subsec:uav-solution} by (i) performing a direct transcription to obtain a finite-dimensional discrete-time program, (ii) using SCvx* to linearize the non-convex program to obtain a convex sub-problem around a current reference $\{\hat{x}_i, \hat{u}_i\}_{i=1}^{N_g}$, and (iii) reformulating the convex program into a form amenable to distributed optimization through C-ADMM.

\begingroup\color{black}
\hypersetup{citecolor=black,linkcolor=black,urlcolor=black}
First, introduce epigraph variables $\kappa_i\in\R$ for the individual void costs.
Fix $N\geq1$, $\Delta t=T/N$, and $t_n=n\Delta t$, and hold $u_i(t)=u_{i,n}$ on each interval $[t_n,t_{n+1})$.
Assume the augmented dynamics have a unique solution on each interval for the chosen initial state and held control.
Their exact flow is
\begin{equation*}
    F_n^g(z_n^g,u_n)\triangleq\zeta_n^g(t_{n+1}),\qquad
    \dot\zeta_n^g(t)=f_{\mathrm{aug}}^g(\zeta_n^g(t),u_n),\quad\zeta_n^g(t_n)=z_n^g.
\end{equation*}
Multiple shooting treats the node states as independent variables and imposes the matching equations $z_{n+1}^g=F_n^g(z_n^g,u_n)$.
We approximate the flow by numerical integration and compute its Jacobians through the variational equations in Appendix~\ref{app:flow-sensitivities}.
The corridor $\scriptC_\rho$ in the void cost $\scriptH_{k,i}$ is parameterized by $P_i\triangleq\operatorname{col}(p_{i,0},\ldots,p_{i,N})\in\R^{2(N+1)}$, which gives the interpolated continuous path $p_i^P(t)=\sum_{n=0}^{N}B_n(t)p_{i,n}$, where $B_n:[0,T]\to\R$ are scalar, piecewise-linear hat basis functions.
Multiple shooting and the relaxed increment bounds \eqref{eq:pair-ctcs-relaxed}, with $\varepsilon_{ij,n}>0$, give the NLP
\endgroup
\begin{subequations}\label{prob:ugv-transcription}
    \begin{align}
        &\min_{x,u,y,\kappa} &&J_g(x,u,\kappa) = \sum_{i=1}^{N_g}
        \left[
        \frac{\beta \Delta t}{2}
        \sum_{n=0}^{N-1}
        u_{i,n}^\intercal R_i u_{i,n}
        +
        \kappa_i
        \right] \\
        &\quad\mathrm{s.t.} &&z^g_{n+1}=F_n^g(z^g_n,u_n), \quad n=0,\ldots,N-1,\\
        &&&x_{i,0}\ \mathrm{given}, \quad i\in\scriptV, \qquad y_{ij,0}=0, \quad \{i,j\}\in\scriptE,\\
        &&&(x_{i,n},u_{i,n})\in\Xi_i, \quad x_{i,N}\in\Xi_i^x, \quad i\in\scriptV, \quad n=0,\ldots,N-1,\\
        &&&\mathrm{dist}\left(p_{i,n}, \scriptC_{\delta_{\ell}}\right) \geq d_{i\ell}^{\mathrm{clr}}, \quad i\in\scriptV, \quad \ell = 1,\ldots, L_k, \quad  n = 1,\ldots, N, \label{eq:ugv-distance-constraint}\\ 
        &&&\textcolor{black}{y_{ij,n+1}-y_{ij,n}\leq\varepsilon_{ij,n}}, \quad \{i,j\}\in\scriptE, \quad n=0,\ldots,N-1,
        \\
        &&&\scriptH_{k,i}(P_i)\leq\kappa_i, \quad i\in\scriptV,\label{eq:ugv-void-epigraph}
        \\
        &&&x_{i,N}\in\scriptX_{i,T}^{\mathrm{goal}}, \quad i\in\scriptV,
    \end{align}
\end{subequations}
{\color{black}
Here $x$, $u$, and $y$ collect the physical-state, control, and separation-accumulator shooting sequences.
The fixed initial states satisfy the geometric threat and inter-agent separation constraints.
Pairwise separation is enforced through the relaxed CTCS formulation, while threat-distance and admissibility constraints remain imposed at nodes.
}
We have decomposed avoidance of the union of all threat confidence sets into individual constraints with respect to each threat itself%
\footnote{%
    Avoiding $\bigcup_\ell(\scriptC_{\delta_\ell}\oplus\B(0,d_{i\ell}^{\mathrm{clr}}))$ is equivalent to imposing $\mathrm{dist}(p,\scriptC_{\delta_\ell})\geq d_{i\ell}^{\mathrm{clr}}$ for every $\ell$.
}.
\begingroup\color{black}
\hypersetup{citecolor=black,linkcolor=black,urlcolor=black}
Let $X_i\triangleq\operatorname{col}(x_{i,0},\ldots,x_{i,N})\in\R^{4(N+1)}$ and $U_i\triangleq\operatorname{col}(u_{i,0},\ldots,u_{i,N-1})\in\R^{2N}$.
As in the UAV decomposition, introduce symmetric endpoint copies
\begin{equation*}
    Y_{ij}^{i}=\operatorname{col}(y_{ij,0}^{i},\ldots,y_{ij,N}^{i}),\qquad
    Y_{ij}^{j}=\operatorname{col}(y_{ij,0}^{j},\ldots,y_{ij,N}^{j}),\qquad
    Y_{ij}^{i}=Y_{ij}^{j},\quad\forall\{i,j\}\in\scriptE.
\end{equation*}
Define agent $i$'s decision vector by
\begin{equation*}
    \vartheta_i\triangleq\operatorname{col}(X_i,U_i,\{Y_{ij}^{i}\}_{j\in\scriptN_i},\kappa_i)
    \in\R^{d_{\vartheta,i}},\qquad
    d_{\vartheta,i}=4(N+1)+2N+(N+1)|\scriptN_i|+1,
\end{equation*}
and let $J_i(\vartheta_i)\triangleq\frac{\beta\Delta t}{2}\sum_{n=0}^{N-1}u_{i,n}^\intercal R_i u_{i,n}+\kappa_i$.
Let $\Theta_i$ collect the convex state, input, domain, and terminal constraints, the prescribed initial conditions, and the affine accumulator increment bounds for each endpoint copy at agent $i$.
The epigraph constraint in \eqref{eq:ugv-void-epigraph} relegates each non-convex individual void probability surrogate to a local inequality.
Together with the local threat-distance constraints, these inequalities form
\begin{equation*}
    \gamma_i(\vartheta_i)\triangleq\operatorname{col}\!\left(
    \scriptH_{k,i}(P_i)-\kappa_i,
    \{d_{i\ell}^{\mathrm{clr}}-\operatorname{dist}(p_{i,n},\scriptC_{\delta_\ell})\}_{\ell=1,\ldots,L_k;\,n=1,\ldots,N}
    \right)\in\R^{1+NL_k}.
\end{equation*}
Let $\gamma_i^{\mathrm{cvx}}$ denote their affine approximation at the current reference.
For the pairwise shooting defects, set $\bar y_{ij,n}=(y_{ij,n}^{i}+y_{ij,n}^{j})/2$ and let $\bar z_n^g$ stack the physical states and these averages.
Define
\begin{equation*}
    \pi_{ij,n}(\vartheta_i,\vartheta_j)\triangleq
    \bar y_{ij,n+1}-[F_n^g(\bar z_n^g,u_n)]_{y_{ij}},\qquad n=0,\ldots,N-1,
\end{equation*}
and stack the interval defects into $\pi_{ij}\in\R^N$.
Let $F_{i,n}^g$ denote the physical-state component of the flow for agent $i$, and define its local shooting defects by
\begin{equation*}
    \Pi_i(\vartheta_i)\triangleq
    \operatorname{col}(\{x_{i,n+1}-F_{i,n}^g(x_{i,n},u_{i,n})\}_{n=0}^{N-1})\in\R^{4N}.
\end{equation*}
Their affine approximations $\pi_{ij}^{\mathrm{cvx}}$ and $\Pi_i^{\mathrm{cvx}}$ follow from the augmented variational equations in Appendix~\ref{app:flow-sensitivities}, composed with the averaging map for the pairwise defects.
With virtual controls $\xi_i\in\R^{4N}$, $\zeta_{ij}\in\R^N$, and inequality slacks $s_i\in\R_+^{1+NL_k}$, performing the steps for SCvx* results in the convex sub-problem
\endgroup
\begingroup\color{black}
\hypersetup{citecolor=black,linkcolor=black,urlcolor=black}
\begin{subequations}\label{prob:ugv-scvx-subproblem}
    \begin{align}
        &\min_{\substack{\forall i\in\scriptV: \vartheta_i,\xi_i,\,s_i \\ \forall \{i,j\}\in\scriptE: \zeta_{ij}}} &&\sum_{i=1}^{N_g} J_i(\vartheta_i) + \sum_{i=1}^{N_g} \left( \lambda_i^\intercal \xi_i + \frac{w}{2}\|\xi_i\|^2 +\mu_i^\intercal s_i+\frac{w}{2}\|s_i\|^2\right) \nonumber\\[-1mm]
        &&&\hspace{4cm}+ \sum_{\{i,j\}\in\scriptE} \left( \mu_{ij}^\intercal \zeta_{ij} + \frac{w}{2}\|\zeta_{ij}\|^2 \right) \\
        &\quad\quad \mathrm{s.t.} &&\Pi_i^{\mathrm{cvx}}(\vartheta_i) = \xi_i, \quad i \in\scriptV, \\
        &&&\gamma_i^{\mathrm{cvx}}(\vartheta_i)\leq s_i,\quad s_i\geq0, \quad \vartheta_i\in\Theta_i, \quad i\in\scriptV,\\
        &&&\pi_{ij}^{\mathrm{cvx}}(\vartheta_i,\vartheta_j) = \zeta_{ij}, \quad \{i,j\} \in \scriptE,\label{eq:ugv-convex-pair} \\
        &&&Y_{ij}^{i}=Y_{ij}^{j},\quad\{i,j\}\in\scriptE,\label{eq:ugv-convex-copy}\\
        &&&\|\vartheta_i - \hat{\vartheta}_i\|_{\infty} \leq \Delta_{\mathrm{tr}}, \quad i \in \scriptV.
    \end{align}
\end{subequations}
\endgroup
The shooting virtual controls $\xi_i$ and $\zeta_{ij}$ and their multipliers $\lambda_i$ and $\mu_{ij}$ are unrestricted in sign; $s_i$ and its multiplier $\mu_i$ are nonnegative because they correspond to local inequalities.
Expressions for the distance function in \eqref{eq:ugv-distance-constraint} and negative log void probability linearization in \eqref{eq:negative-log-void} are given in Appendices~\ref{app:distance} and~\ref{app:void}, respectively, for completeness.
In particular, the distance function gradient $\nabla_p\,\mathrm{dist}(p, \scriptC_{\delta})$ is {the unit normal vector from the projection $\Pi_{\scriptC_{\delta}}(p)$ for $p\notin\scriptC_\delta$}, and the void probability gradient $\nabla_{P_i} \scriptH_{k,i}(P_i)$ comes from using shape calculus to analyze variations of the set $\scriptC_{\rho}(p)$.
{Only \eqref{eq:ugv-convex-pair} and \eqref{eq:ugv-convex-copy} couple neighboring agents.
As in Section~\ref{subsec:uav-solution}, introduce endpoint copies of $\zeta_{ij}$, split its penalty equally between the endpoints, and include the pairwise shooting and copy-agreement constraints in the unified coupling map $\widehat H_{ij}$.
The UGV objective is already separable, so no global coefficient copies or flow variables are needed.}
The remaining steps needed to turn the convex program \eqref{prob:ugv-scvx-subproblem} into a fully separable form amenable to C-ADMM are largely the same as those in Section~\ref{subsec:uav-solution}, so they are omitted here for brevity.
Thus, as for the UAV team, every outer SCvx* iteration produces a coupled convex subproblem that is solved by the inner graph C-ADMM iterations before the outer reference is updated.

\section{Vehicle Coordination and Re-Planning}~\label{sec:coordination}
It would not be prudent for the ground vehicle fleet to start moving before the aerial vehicles have sufficiently searched the corridors $\scriptC_{\rho_i}$ for all $i = 1,\ldots, N_g$.
The void probability construction above gives the probability that there are no represented threats in the union of the remaining UGV corridors.
At planning epoch $k$, let $p_{i,k}^{\mathrm{rem}}$ denote the unexecuted portion of the route committed to UGV $i$.
\begin{samepage}
We define $\scriptC_k^{\mathrm g}$ as the remaining fleet corridor and $q_{k,j}^{\mathrm g}$ as the conditional probability mass of Bernoulli track $j$ inside that corridor.
\begin{equation}
    \scriptC_k^{\mathrm g}
    \triangleq
    \bigcup_{i=1}^{N_g}
    \scriptC_{\rho_i}\!\left(p_{i,k}^{\mathrm{rem}}\right),
    \qquad
    q_{k,j}^{\mathrm g}
    \triangleq
    \int_{\scriptC_k^{\mathrm g}}\rho_{k,j}(y)\,\d y.
    \label{eq:remaining-corridor-track-overlap}
\end{equation}
\end{samepage}
In particular, applying \eqref{eq:pmb-void-probability} to this union gives the joint void probability $P_{0,k}^{\mathrm g}=\exp(-\scriptH_k^{\mathrm g})$, where
\begin{equation}\label{eq:fleet-log-void}
    \scriptH_k^{\mathrm g}
    =\int_{\scriptC_k^{\mathrm g}}\lambda_k(y)\,\d y
    -\sum_{j=1}^{L_k}\log\big(1-r_{k,j}q_{k,j}^{\mathrm g}\big).
\end{equation}
A simple yet effective motion gate is to allow the fleet to move only when
\begin{equation}\label{eq:fleet-motion-gate}
    P_{0,k}^{\mathrm g}\geq1-\epsilon_{\mathrm{go}}
    \quad\Longleftrightarrow\quad
    \scriptH_k^{\mathrm g}\leq-\log(1-\epsilon_{\mathrm{go}}),
    \qquad\epsilon_{\mathrm{go}}\in(0,0.5].
\end{equation}
This is one joint fleet gate evaluated under the current PMB belief.
All UGVs begin each committed plan together and execute on the common time base used by the joint optimization and its inter-agent collision-avoidance constraints.
The motion gate is evaluated on the currently committed routes and is distinct from the event that requests a new UGV trajectory optimization.

A newly formed Bernoulli component is declared confirmed once its existence probability satisfies $r_{k,j}\geq r_{\mathrm{conf}}$, where $r_{\mathrm{conf}}\in(0,1)$ is a configurable confirmation threshold.
The UGV fleet requests a new trajectory solve when there exists a newly confirmed track $j$ such that
\begin{equation}
    r_{k,j}\geq r_{\mathrm{conf}},
    \qquad
    r_{k,j}q_{k,j}^{\mathrm g}>\epsilon_{\mathrm{trigger}},
    \label{eq:bernoulli-replan-trigger}
\end{equation}
where $\epsilon_{\mathrm{trigger}}\in(0,1)$ specifies the minimum existence-weighted route-overlap probability considered operationally significant.
A confirmed track with negligible probability mass inside the remaining corridor therefore does not cause a route change, whereas a spatially uncertain track may trigger re-planning even when its posterior mean lies outside the corridor.

Upon a trigger, the UGV trajectory optimization problem is re-solved from the current UGV states using the latest PMB posterior.
The candidate routes replace only the unexecuted portions of the previously committed routes after the SCvx*--C-ADMM iterations return a candidate satisfying the required dynamics, terminal, threat-clearance, and inter-agent safety conditions.
Following commitment, the remaining-route map, posterior-dependent base density, and UAV trajectories are updated, while UGV execution remains governed by the joint motion gate \eqref{eq:fleet-motion-gate}.
Figure~\ref{fig:bernoulli-triggered-replanning} summarizes this trigger, route-update, and execution sequence.
\begin{figure}[H]
    \centering
    \includegraphics[width=\linewidth]{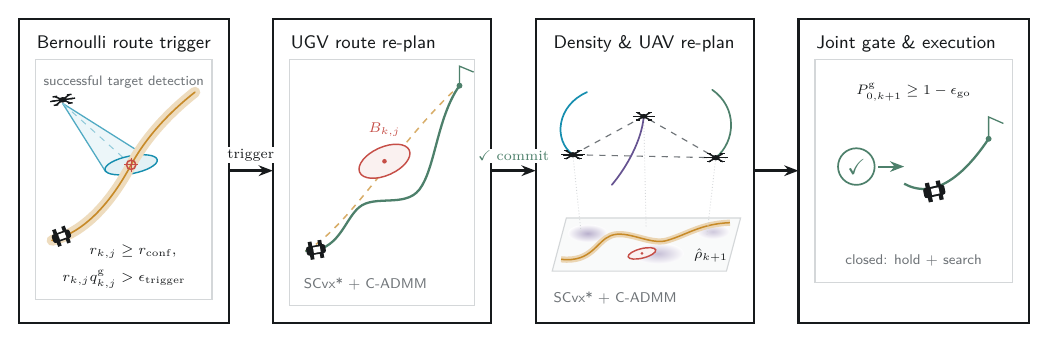}
    \caption{Closed-loop Bernoulli-triggered re-planning flow: route-overlap evaluation, UGV route re-optimization, UAV search-density and trajectory updates, and joint PMB-gated UGV execution.}
    \label{fig:bernoulli-triggered-replanning}
\end{figure}

Next, we describe how we construct and update the core base distribution $\hat\mu$ that drives the UAV fleet ergodic search.
We begin with the UGV route: let $p(t), \ t \in [0,T]$ denote the most recently committed UGV route, or only its unexecuted suffix if the UGV has already begun moving.
Define the un-normalized Gaussian tube
\begin{equation}
    \tilde{\Phi}(\xi) = \int_{0}^{T} w_k(t) \exp\left(-\frac{\|\xi - p(t)\|^2}{2\sigma_r^2} \right)\, \d t,
\end{equation}
where \joshrevision{$\sigma_r>0$ is the tube width and} $w_k(t)\geq0$ is a temporal weighting function (e.g., \joshrevision{$w_k(t)=e^{-\lambda_r t}$ with $\lambda_r>0$} biasing the density towards the beginning of the UGV path).
\joshrevision{Assume $0<\int_0^T w_k(t)\,\d t<\infty$.}
\joshrevision{Aggregating each UGV's route, we obtain the normalized route map $\Phi=\operatorname{Norm}_{\Omega}[\sum_{i=1}^{N_g}\tilde\Phi_i]$.}
Concentrating the base density along the UGV routes alone may waste sensing effort in portions of the routes that have already been sufficiently explored.
In particular, the route map depends only on geometry and does not incorporate the current posterior probability of an undetected target.
To emphasize its dependence on the route available at planning epoch $k$, we henceforth denote the normalized route map by $\Phi_k$.
We define the \textit{route-conditioned} PPP map as
\begin{equation}
    U_k(\xi)
    \triangleq
    \operatorname{Norm}_{\Omega}\!\left[\Phi_k(\cdot)\lambda_k(\cdot)\right](\xi)
    =
    \frac{\Phi_k(\xi)\lambda_k(\xi)}
    {\displaystyle\int_{\Omega}\Phi_k(z)\lambda_k(z)\,\d z}.
    \label{eq:route_conditioned_ppp_map}
\end{equation}
This map assigns sensing effort to locations that are simultaneously relevant to the current UGV routes and likely, according to the residual PPP intensity, to contain an as-yet-undetected target.
Note that if the denominator in \eqref{eq:route_conditioned_ppp_map} is zero, this component is omitted and the remaining mixture weights are renormalized.

It remains to incorporate targets that have already generated Bernoulli components in the PMB posterior.
Simply using $T_{k,j}(\xi)=\operatorname{Norm}_{\Omega}[\rho_{k,j}](\xi)$ is generally undesirable.
First, this is merely the in-domain normalization of the conditional track density and hence does not account for the existence probability $r_{k,j}$.
Second, every track would contribute unit mass after normalization, irrespective of its localization uncertainty.
Thus, a highly uncertain track and an accurately localized track could receive the same total sensing attention.
Finally, $\rho_{k,j}(\xi)$ describes where the target is believed to be, whereas the desired base density must describe where the UAV sensor should be placed.
For an angles-only sensor, these locations need not coincide because informative viewpoints depend on the sensor-to-target geometry.

We therefore construct a Bayesian expected information density over the planar search domain, building on information-driven ergodic exploration~\joshrevision{\citep{miller2016information,mavrommati2018realtime,abraham2018localization}}.
Let $s\in\Omega$ denote a candidate planar sensor coordinate, which we take to be the center of the UAV ground footprint.
Importantly, $s$ is not the complete UAV state.
Let $x=\mathfrak c(s,\zeta)$ denote a full sensing configuration whose footprint center is $s$, where the set $\mathcal Z$ is the admissible nuisance-configuration space and $\zeta\in\mathcal Z$ collects the remaining sensing variables, including altitude and attitude.
At planning epoch $k$, let $\varpi_k(\d\zeta\mid s)$ denote a prescribed conditional probability measure over these nuisance variables.
For example, the simplest construction fixes a nominal altitude and a nominal downward-looking attitude,
\begin{equation}
    \varpi_k(\d\zeta\mid s)
    =
    \delta_{\bar{\zeta}_k(s)}(\d\zeta).
\end{equation}
Alternatively, $\varpi_k$ may represent the distribution of altitude and attitude induced by the current reference UAV trajectories.
In either case, $\varpi_k$ is frozen while solving the UAV trajectory optimization problem at epoch $k$.

Suppose that the angles-only measurement generated by a target at $y\in\Omega$ from sensing configuration $\mathfrak c(s,\zeta)$ follows the model in \eqref{eq:meas-model}, written here as
\begin{equation}
    z
    =
    h\big(y;\mathfrak c(s,\zeta)\big)+v,
    \qquad
    v\sim\mathcal N\!\left(0,R\big(y;s,\zeta\big)\right),
\end{equation}
\joshrevision{where $R(y;s,\zeta)\triangleq R(\mathfrak c(s,\zeta)\mid y)\succ0$ is the covariance from \eqref{eq:meas-model}.}
Define the measurement Jacobian with respect to the planar target position as
\begin{equation}
    H_y(y;s,\zeta)
    \triangleq
    \frac{\partial h\big(y;\mathfrak c(s,\zeta)\big)}{\partial y} \in \R^{2\times 2}.
\end{equation}
\joshrevision{Let $p_k^D(y\mid\mathfrak c(s,\zeta))$ denote the detection probability from Section~\ref{sec:rfs-inference}, evaluated at this sensing configuration.}
The detection-weighted information surrogate associated with this potential measurement is defined as
\begin{equation}
    \mathsf I(y;s,\zeta)
    \triangleq
    p_k^{D}\big(y\mid\mathfrak c(s,\zeta)\big)
    H_y^\intercal(y;s,\zeta)
    R^{-1}(y;s,\zeta)
    H_y(y;s,\zeta).
    \label{eq:conditional_fim}
\end{equation}
The expected information surrogate for Bernoulli component $j$, conditional on that target existing, is then
\begin{equation}
    \overline{\mathsf I}_{k,j}(s)
    \triangleq
    \int_{\mathcal Z}
    \int_{\Omega}
    \mathsf I(y;s,\zeta)\rho_{k,j}(y)
    \,\d y\,
    \varpi_k(\d\zeta\mid s).
    \label{eq:expected_track_fim}
\end{equation}
Consequently, $\overline{\mathsf I}_{k,j}:\Omega\rightarrow\mathbb S_+^2$ is a matrix-valued field indexed by the planar footprint center $s$.
When $R$ depends on $y$, \eqref{eq:conditional_fim} omits covariance derivative and detection event information terms and is therefore an information surrogate rather than the complete Fisher information~\citep{wilson2014trajectory,chaloner1995bayesian}.
The determinant is taken after computing the expectation in \eqref{eq:expected_track_fim}.
The classical expected information density (EID) construction would use
\begin{equation}
    E^{\mathrm{FIM}}_{k,j}(s)
    =
    \det\!\left(\overline{\mathsf I}_{k,j}(s)\right).
\end{equation}
\joshrevision{D-optimality maximizes the determinant, or equivalently the log-determinant, of the information matrix.
Under a Gaussian approximation, this minimizes the volume of the corresponding confidence ellipsoid~\citep{boyd2004convex,chaloner1995bayesian}.}
For the present application, we instead use the Bayesian incremental D-optimal \joshrevision{surrogate}
\begin{align}
    E_{k,j}(s)
    &\triangleq
    \log\det\!\left(
        P_{k,j}^{-1}
        +
        \overline{\mathsf I}_{k,j}(s)
    \right)
    -
    \log\det\!\left(P_{k,j}^{-1}\right)
    \nonumber\\
    &=
    \log\det\!\left(
        I_2+
        P_{k,j}^{1/2}\overline{\mathsf I}_{k,j}(s)P_{k,j}^{1/2}
    \right),
    \label{eq:bayesian_track_eid}
\end{align}
where $P_{k,j}\succ0$ is the covariance of the planar position density $\rho_{k,j}$, or the planar position marginal if the Bernoulli state contains additional variables.
\joshrevision{The second equality follows by factoring $P_{k,j}^{-1/2}$ from both sides of $P_{k,j}^{-1}+\overline{\mathsf I}_{k,j}(s)$ and taking determinants.}
The quantity in \eqref{eq:bayesian_track_eid} approximates the improvement in the log-determinant of the track information matrix produced by a measurement from footprint center $s$.
In particular, $E_{k,j}(s)\geq 0$ and tends to zero as $P_{k,j}\to0$ provided $\overline{\mathsf I}_{k,j}(s)$ remains bounded.
Thus, unlike a separately normalized copy of $\rho_{k,j}$, an accurately localized track naturally receives less sensing attention under this bounded information condition.

The conditional information density in \eqref{eq:bayesian_track_eid} does not, by itself, account for whether track $j$ exists or whether it is relevant to a UGV route.
{\color{red}
Recall from \eqref{eq:remaining-corridor-track-overlap} that $\scriptC_k^{\mathrm g}$ is the union of the remaining UGV-route corridors and $q_{k,j}^{\mathrm g}$ is the conditional probability mass of track $j$ inside that union.
}
The corresponding Bernoulli route-risk weight is
\begin{equation}
    a_{k,j}^{\mathrm B}
    \triangleq
    -\log\!\left(
        1-r_{k,j}q_{k,j}^{\mathrm g}
    \right).
    \label{eq:bernoulli_route_risk_weight}
\end{equation}
This is precisely the contribution of track $j$ to the negative log-void probability of the remaining fleet corridor.
It is small when the track has low existence probability or negligible probability mass inside the relevant corridors, and becomes large when an existing track is likely to lie along a UGV route.
\joshrevision{For the density construction, assume $r_{k,j}q_{k,j}^{\mathrm g}<1$ and finite, integrable information fields so the weights and normalization below are well defined.}
The aggregate, unnormalized Bernoulli information map is therefore
\begin{equation}
    \widetilde T_k(s)
    \triangleq
    \sum_{j=1}^{L_k}
    a_{k,j}^{\mathrm B}E_{k,j}(s),
    \qquad
    Z_k^{\mathrm B}
    \triangleq
    \int_{\Omega}\widetilde T_k(s)\,\d s.
    \label{eq:aggregate_track_eid}
\end{equation}
When $Z_k^{\mathrm B}>0$, we define
\begin{equation}
    T_k(s)
    \triangleq
    \operatorname{Norm}_{\Omega}[\widetilde T_k](s).
    \label{eq:normalized_track_eid}
\end{equation}
\joshrevision{If $Z_k^{\mathrm B}=0$, this component is omitted, its mixture weight is set to zero, and the remaining weights are renormalized.}
Notice that the individual fields $E_{k,j}$ are not normalized before being combined.
Consequently, both the weight $a_{k,j}^{\mathrm B}$ and the magnitude of the available localization information affect the relative contribution of track $j$.

Writing $\hat\mu_k(\d\xi)=\hat\rho_k(\xi)\,\d\xi$, the updated base density is finally constructed as the convex mixture
\begin{equation}
    \hat{\rho}_k(\xi)
    =
    \omega_{k}^{\mathrm r}\Phi_k(\xi)
    +
    \omega_{k}^{\mathrm u}U_k(\xi)
    +
    \omega_{k}^{\mathrm B}T_k(\xi),
    \label{eq:updated_base_density}
\end{equation}
where
$
    \omega_{k}^{\mathrm r},
    \omega_{k}^{\mathrm u},
    \omega_{k}^{\mathrm B}\geq 0,
    \
    \omega_{k}^{\mathrm r}
    +
    \omega_{k}^{\mathrm u}
    +
    \omega_{k}^{\mathrm B}
    =1.
$
Because each active component is a probability density on $\Omega$, \eqref{eq:updated_base_density} defines a valid probability measure $\hat\mu_k$ on $\Omega$.

The mixture weights may be fixed design parameters.
Alternatively, they may be adapted using the current posterior RFS.
One possible construction is to define
\begin{equation}
    \Lambda_k^{\mathrm g}
    \triangleq
    \int_{\scriptC_k^{\mathrm g}}\lambda_k(y)\,\d y,
    \qquad
    \eta_k^{\mathrm u}
    \triangleq
    \frac{\Lambda_k^{\mathrm g}}
    {\Lambda_k^{\mathrm g}+\Lambda_{\mathrm{ref}}},
    \qquad
    \eta_k^{\mathrm B}
    \triangleq
    \frac{Z_k^{\mathrm B}}
    {Z_k^{\mathrm B}+Z_{\mathrm{ref}}},
\end{equation}
where $\Lambda_{\mathrm{ref}}>0$ and $Z_{\mathrm{ref}}>0$ are reference scales.
Given nominal preferences
$\bar\omega^{\mathrm r}>0$,
$\bar\omega^{\mathrm u}\geq0$, and
$\bar\omega^{\mathrm B}\geq0$, let
\begin{equation}
    s_k^{\mathrm r}=\bar\omega^{\mathrm r},
    \qquad
    s_k^{\mathrm u}=\bar\omega^{\mathrm u}\eta_k^{\mathrm u},
    \qquad
    s_k^{\mathrm B}=\bar\omega^{\mathrm B}\eta_k^{\mathrm B},
\end{equation}
and set
\begin{equation}
    \omega_k^{m}
    =
    \frac{s_k^{m}}
    {s_k^{\mathrm r}+s_k^{\mathrm u}+s_k^{\mathrm B}},
    \qquad
    m\in\{\mathrm r,\mathrm u,\mathrm B\}.
    \label{eq:adaptive_base_density_weights}
\end{equation}
This rule retains a baseline emphasis on the committed UGV routes, increases the residual-PPP contribution when appreciable undetected target intensity remains inside the corridors, and decreases the track contribution when the route-relevant tracks are either unlikely or already accurately localized.

At every planning epoch, the PMB posterior, remaining UGV routes, nuisance-state distribution $\varpi_k$, and all component maps are updated once and then held fixed while the ensuing UAV trajectory optimization problem is solved.

\section{Case Studies}~\label{sec:examples}
\begingroup
\color{black}
\raggedbottom
\setlength{\textfloatsep}{12pt plus 2pt minus 2pt}
\setlength{\floatsep}{12pt plus 2pt minus 2pt}
\setlength{\intextsep}{10pt plus 2pt minus 2pt}
\makeatletter
\setlength{\@fptop}{0pt}
\setlength{\@fpsep}{12pt plus 2pt minus 2pt}
\setlength{\@fpbot}{0pt plus 1fil}
\long\def\@makecaption#1#2{\vskip6pt{\small\raggedright #1: #2\par}\vskip4pt}
\makeatother

\caseedit{The case studies evaluate aerial coverage, risk-sensitive ground planning, and closed-loop collaboration.}
\caseedit{After specifying the simulation settings, we examine multi-UAV ergodic search and solver performance, followed by the UGV risk--effort tradeoff.}
\caseedit{Missions with one and two UGVs then demonstrate how aerial sensing, ground replanning, and risk-gated motion interact during execution.}
\josh{}

\color{black}
\Needspace{7\baselineskip}
\subsection{Simulation Setup}\label{subsec:case-setup}

The simulations are run on an Intel Core Ultra 9 275HX processor ($24$ cores).
The examples use the workspace $\Omega=[0,1]^2$ and the quadrotor and accelerated-unicycle models in \eqref{eq:uav-dynamics} and \eqref{eq:ugv-dynamics}, respectively.
Distances are in normalized workspace units, time is in seconds, and angles are in radians unless degrees are indicated.
Tables~\ref{tab:case-sensor}--\ref{tab:case-setup} list the sensor, vehicle, boundary, and numerical parameters.
The sensing mission uses three UAVs, one UGV, and three stationary targets.
The initial belief is a uniform PPP with intensity $\lambda_0^-(\xi)=3/|\Omega|$ and no Bernoulli tracks.
\begingroup
\color{black}
\makeatletter
\def\@captype{table}
\long\def\@makecaption#1#2{\vskip6pt{\small\color{black}\raggedright #1: #2\par}\vskip4pt}
\makeatother
\newsavebox{\caseSensorTable}
\begin{lrbox}{\caseSensorTable}
\begin{minipage}[t]{0.47\textwidth}
        \vspace{0pt}
        \caption{\color{black}Sensor model.}
        \label{tab:case-sensor}
        \color{black}\footnotesize
        \setlength{\tabcolsep}{2pt}
        \renewcommand{\arraystretch}{1.08}
        \begin{tabular*}{\linewidth}{@{\extracolsep{\fill}}lr@{}}
            \toprule
            Parameter & Value \\
            \midrule
            \multicolumn{2}{@{}l}{\textit{Footprint and sensing quality}} \\
            FoV half-angle, $\alpha$ & $15^\circ$ \\
            Attitude exponent, $\gamma$ & $2.2$ \\
            Altitude scale, $\bar h$ & $0.42$ \\
            Footprint scale, $\kappa_{\mathrm f}$ & $0.50$ \\
            Covariance floor, $\varepsilon_\Sigma$ & $10^{-6}$ \\
            \addlinespace[3pt]
            \multicolumn{2}{@{}l}{\textit{Angular observations}} \\
            Noise floor, $\sigma_{\min}$ & $0.8^\circ$ \\
            Noise gain, $k_\sigma$ & $2^\circ$ \\
            Noise ceiling, $\sigma_{\max}$ & $4^\circ$ \\
            Sensor update rate & $5$ Hz \\
            Mean clutter count, $\bar\kappa$ & $0.01$/scan \\
            \bottomrule
        \end{tabular*}
    \end{minipage}
\end{lrbox}
\newsavebox{\caseVehicleTable}
\begin{lrbox}{\caseVehicleTable}
\begin{minipage}[t]{0.47\textwidth}
        \vspace{0pt}
        \caption{\color{black}Vehicle models and limits.}
        \label{tab:case-vehicles}
        \color{black}\footnotesize
        \setlength{\tabcolsep}{2pt}
        \renewcommand{\arraystretch}{1.08}
        \begin{tabular*}{\linewidth}{@{\extracolsep{\fill}}lr@{}}
            \toprule
            Parameter & Value \\
            \midrule
            \multicolumn{2}{@{}l}{\textit{UAVs: thrust and body-rate control}} \\
            Mass, $m_i$ & $1$ kg \\
            Altitude range & $[0.15,1.0]$ \\
            Initial altitude & $0.45$ \\
            Velocity bound, $\|v_i\|_\infty$ & $1.25$ \\
            Roll / pitch bound & $0.42$ rad \\
            Yaw bound & $0.75$ rad \\
            Thrust range & $[7,13.5]$ \\
            Body-rate bound, $\|\omega_i\|_\infty$ & $1.0$ rad/s \\
            Inter-UAV separation & $0.22$ \\
            \addlinespace[3pt]
            \multicolumn{2}{@{}l}{\textit{UGV: acceleration and turn-rate control}} \\
            Speed range & $[0,0.50]$ \\
            Acceleration limit & $0.80$ \\
            Turn-rate limit & $1.5$ rad/s \\
            UGV / target radius & $0.035,\ 0.015$ \\
            Additional safety margin, $\Delta$ & $0.010$ \\
            Swept-corridor radius, $\rho$ & $0.060$ \\
            Terminal speed, $v_f$ & $0$ \\
            \bottomrule
        \end{tabular*}
    \end{minipage}
\end{lrbox}
\newsavebox{\caseInitialTable}
\begin{lrbox}{\caseInitialTable}
\begin{minipage}[t]{0.47\textwidth}
        \vspace{0pt}
        \caption{\caseedit{Initial positions and goals.}}
        \label{tab:case-initial}
        \color{black}\footnotesize
        \setlength{\tabcolsep}{2pt}
        \renewcommand{\arraystretch}{1.08}
        \begin{tabular*}{\linewidth}{@{\extracolsep{\fill}}lr@{}}
            \toprule
            Parameter & Value \\
            \midrule
            UAV 1 position, $r_{1,0}$ & $(0.1,0.9,0.45)$ \\
            UAV 2 position, $r_{2,0}$ & $(0.9,0.9,0.45)$ \\
            UAV 3 position, $r_{3,0}$ & $(0.5,0.5,0.45)$ \\
            Ground-route start, $p_0$ & $(0.1,0.1)$ \\
            Ground-route goal, $p_f$ & $(0.9,0.9)$ \\
            \bottomrule
        \end{tabular*}
    \end{minipage}
\end{lrbox}
\newsavebox{\caseNumericsTable}
\begin{lrbox}{\caseNumericsTable}
\begin{minipage}[t]{0.47\textwidth}
        \vspace{0pt}
        \caption{\color{black}Discretization and cost weights.}
        \label{tab:case-setup}
        \color{black}\footnotesize
        \setlength{\tabcolsep}{2pt}
        \renewcommand{\arraystretch}{1.08}
        \begin{tabular*}{\linewidth}{@{\extracolsep{\fill}}lr@{}}
            \toprule
            Parameter & Value \\
            \midrule
            \multicolumn{2}{@{}l}{\textit{Horizon and sensing-density parameters}} \\
            Planning horizon & $10$ s \\
            UAV / UGV, $N:\Delta t$ & $20/40:0.5/0.25$ s \\
            Maximum Fourier index, $K$ & $3$ \\
            Route-tube width, $\sigma_r$ & $0.08$ \\
            \addlinespace[3pt]
            \multicolumn{2}{@{}l}{\textit{Objective and risk weights}} \\
            UAV input, $R^a$ & $0.015I_4$ \\
            UAV velocity / attitude, $Q^a$ & $0.0025I_6$ \\
            UGV input, $R^g$ & $I_2$ \\
            Full-PMB objective, $w_H$ & $10$ \\
            \caseedit{Motion-gate budget, $\epsilon_{\mathrm{go}}$} & \caseedit{$0.05$} \\
            \caseedit{Track confirmation, $r_{\mathrm{conf}}$} & \caseedit{$0.80$} \\
            \caseedit{Replan trigger, $\epsilon_{\mathrm{trigger}}$} & \caseedit{$0.05$} \\
            \bottomrule
        \end{tabular*}
    \end{minipage}
\end{lrbox}

\newsavebox{\caseLeftTables}
\newsavebox{\caseRightTables}
\begin{lrbox}{\caseLeftTables}
    \begin{minipage}[t]{0.47\textwidth}
        \vspace{0pt}
        \usebox{\caseSensorTable}\par\vspace{8pt}
        \usebox{\caseInitialTable}
    \end{minipage}
\end{lrbox}
\begin{lrbox}{\caseRightTables}
    \begin{minipage}[t]{0.47\textwidth}
        \vspace{0pt}
        \usebox{\caseVehicleTable}\par\vspace{8pt}
        \usebox{\caseNumericsTable}
    \end{minipage}
\end{lrbox}
\makeatletter
\dimen@=\dimexpr\ht\caseRightTables+\dp\caseRightTables-\ht\caseLeftTables-\dp\caseLeftTables-\baselineskip\relax
\count@=\dimen@
\divide\count@ by \baselineskip
\advance\count@ by 2
\edef\caseTableWrapLines{\the\count@}
\makeatother
\Needspace{\dimexpr\ht\caseRightTables+\dp\caseRightTables+2\baselineskip\relax}
\setlength{\columnsep}{0.06\textwidth}
\setlength{\intextsep}{0pt}
\begin{wraptable}[\caseTableWrapLines]{r}{0.47\textwidth}
    \vspace{1.4ex}
    \usebox{\caseRightTables}
\end{wraptable}
\noindent\usebox{\caseLeftTables}\par\vspace{\dimexpr8pt+\baselineskip\relax}

\color{black}
Table~\ref{tab:case-vehicles} lists the vehicle parameters and bounds.
The Fourier indices satisfy $k_1,k_2\in\{0,\ldots,K\}$.
The ground objective is $E_u+w_H\scriptH$, where $E_u=\int_0^T\ell_i^g\,\d t$ and $w_H\geq0$ weights the PMB negative log-void cost in \eqref{eq:negative-log-void}.
The hover input is $\bar u_i=(m_i g,0,0,0)^\intercal$.
\par
\makeatletter
\ifnum\c@WF@wrappedlines>1
  \vspace{\dimexpr\c@WF@wrappedlines\baselineskip\relax}
\fi
\makeatother
\WFclear
The simulation running costs use the matrices in Table~\ref{tab:case-setup} and are given by the quadratic functions
\begin{equation*}
    \begin{aligned}
        \ell_i^a &=(u_i-\bar u_i)^\intercal R^a(u_i-\bar u_i)
            +[v_i^\intercal\ q_i^\intercal]Q^a[v_i^\intercal\ q_i^\intercal]^\intercal,\\
        \ell_i^g &=[a_i\ \omega_i]R^g[a_i\ \omega_i]^\intercal.
    \end{aligned}
\end{equation*}

\par
\endgroup %

We assume two disparate UAV graph topologies for the Fourier coefficient consensus and the collision avoidance consensus rounds.
At each planning epoch, a minimum-length spanning tree is constructed from the current UAV positions for coefficient consensus and held fixed during that solve.
Collision avoidance, on the other hand, assumes a complete graph so that constraints are imposed for every UAV pair.

We use the quality and detection functions $\beta_s$ and $G_i$ as defined in Section~\ref{subsec:sensor-model}, and the angles-only measurement model \eqref{eq:meas-model} with angular measurement covariance $R(x_i\mid\xi)=\sigma^2(\chi_{i,k}(\xi))I_2$, with the clipped noise law from \eqref{eq:angular-noise-law} and the values in Table~\ref{tab:case-sensor}.
Measurements are fused (sequentially) in UAV-index order, with one prediction step per fleet scan.
Clutter is uniform on the angular image window $[-\alpha,\alpha]^2$.

\color{black}
\FloatBarrier

\subsection{Multi-UAV Ergodic Trajectory Optimization}\label{subsec:case-uav}

The UAVs start from the positions in Table~\ref{tab:case-initial} and optimize coverage of a Gaussian tube around the ground route from $p_0$ to $p_f$.
This initial base density appears in Figure~\ref{fig:case-base-density}(a); its later updates are discussed in Section~\ref{subsec:case-feedback}.
Figure~\ref{fig:case-uav} shows how the team distributes its footprints along that tube through both horizontal motion and altitude changes.
The vehicles cover different portions of the route while maintaining separation; the minimum inter-agent distance during execution is $0.2216$, above the required $0.22$.
For each executed prefix, we average the footprint coefficients over its elapsed duration and evaluate the ergodic mismatch against the fixed desired density (Figure~\ref{fig:case-uav}(d)), reaching $2.16\times10^{-3}$ at $10$~s.

\begin{figure}[htbp]
    \centering
    \includegraphics[width=\textwidth,clip]{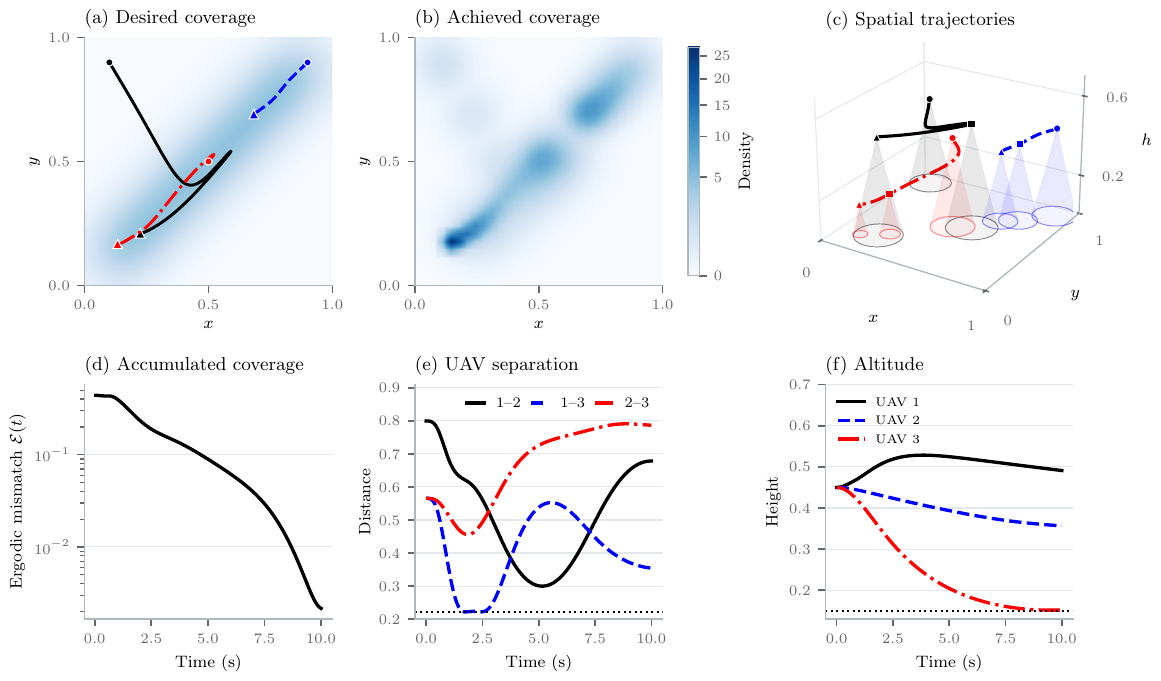}
    \caption{\color{black}Three-UAV coverage and trajectories.
Circles, squares, and triangles mark $0$, $5$, and $10$~s; translucent cones show the FoV, and dotted lines mark bounds.}
    \label{fig:case-uav}
\end{figure}

\color{black}
The SCvx*--C-ADMM solve takes $64$ outer attempts, with $46$ accepted updates, as shown in Figure~\ref{fig:case-uav-convergence}.
The improvement plots record the magnitude of the nonlinear penalty function and nonlinear infeasibility changes, including rejected trials.
The first inner C-ADMM solve converges after $56$ iterations.
\begin{figure}[htbp]
    \centering
    \includegraphics[width=\textwidth,clip]{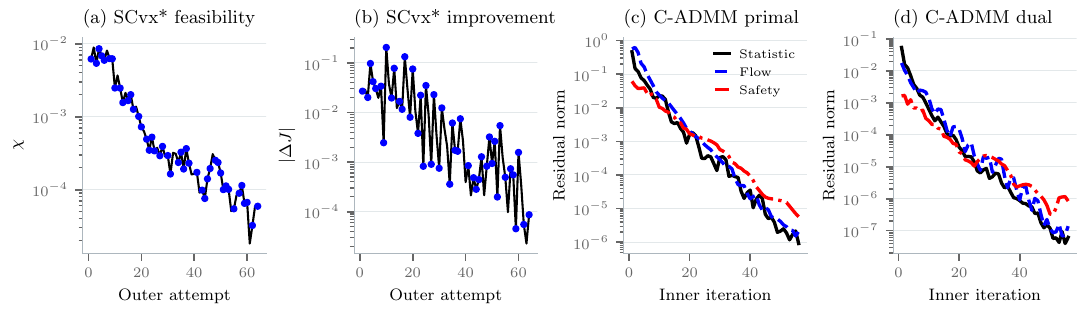}
    \caption{UAV solver convergence.
    Panels (c)--(d) show the C-ADMM primal and dual residuals through convergence for the first outer attempt of the same solve.
    Dots mark accepted steps.}
    \label{fig:case-uav-convergence}
\end{figure}
\color{black}
\Needspace{20\baselineskip}
\begin{wrapfigure}{r}{0.56\textwidth}
    \centering
    \includegraphics[width=\linewidth,clip]{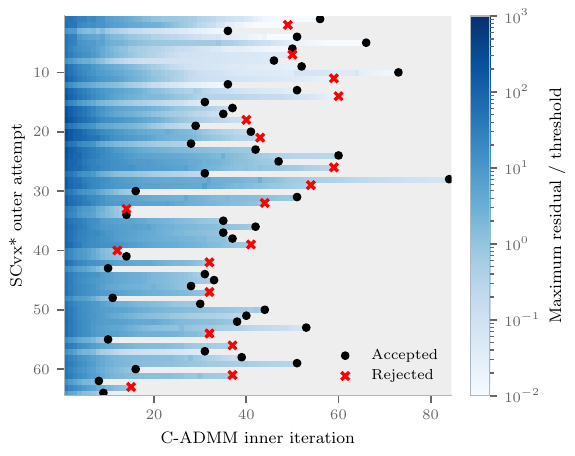}
    \makeatletter
    \long\def\@makecaption#1#2{\vskip6pt{\small\raggedright\noindent #1: #2\par}\vskip4pt}
    \makeatother
    \caption{Nested SCvx*--C-ADMM convergence for the UAV fleet problem.
    Colors show the largest residual-to-tolerance ratio; row-end markers show inner termination and outer acceptance.
    }
    \label{fig:case-nested-convergence}
\end{wrapfigure}
Figure~\ref{fig:case-nested-convergence} shows the inner C-ADMM solves across the outer SCvx* attempts for the same three-UAV, one-route planning instance.
Each row records one outer attempt, including rejected trials.

\color{black}

\leavevmode Next, we compare complete planning times for centralized IPOPT and the proposed decentralized SCvx*--C-ADMM framework.
For each fleet size and horizon, the initial states, desired density, and initial control trajectories are fixed across methods.
Table~\ref{tab:case-solvers} gives the solver settings for the computational comparison.
The C-ADMM penalties act on the statistic, flow, and affine safety residuals, respectively.
\color{black}
Figure~\ref{fig:case-timing} reports the comparison for $2$, $3$, $5$, and $8$ UAVs at $2.5$, $5$, and $10$~s horizons, with $\Delta t=0.5$~s.
The centralized solves use IPOPT~\citep{wachter2006ipopt}, and local C-ADMM QPs use concurrent PIQP~\citep{schwan2023piqp} processes.
The consensus tolerances follow $\epsilon_{\min}\max\{1,c_0 d^n\}$ at outer attempt $n=0,1,\ldots$, so that earlier iterations can converge quicker as they are less accurate, and we only need tighter tolerances near outer convergence.
The repeated trials in Figure~\ref{fig:case-timing} run sequentially on identical inputs, so the error bars measure runtime variability rather than variation across randomized scenarios.
Elapsed time includes model construction, solver setup, nonlinear evaluations, and distributed coordination.

\Needspace{13\baselineskip}
\begin{wraptable}{r}{0.53\textwidth}
    \centering
    \color{black}
    \makeatletter
    \long\def\@makecaption#1#2{\vskip4pt{\small\color{black}\raggedright #1: #2\par}\vskip6pt}
    \makeatother
    \caption{\color{black}UAV solver settings for the centralized versus decentralized computational comparisons.}
    \label{tab:case-solvers}
    \footnotesize
    \setlength{\tabcolsep}{2pt}
    \begin{tabular*}{\linewidth}{@{\extracolsep{\fill}}lr@{}}
        \toprule
        Parameter & Value \\
        \midrule
        \multicolumn{2}{l}{\textit{Inner C-ADMM}} \\
        Initial penalties, $(\rho_\tau,\rho_f,\rho_p)$ & $(0.04,0.05,0.02)$ \\
        Minimum tolerances, $\epsilon_{\min}$ & $10^{-5}$ \\
        Timing schedule, $(c_0,d)$ & $(100,0.8)$ \\
        \addlinespace[3pt]
        \multicolumn{2}{l}{\textit{Outer SCvx*}} \\
        Initial AL penalty, $w$ & $25$ \\
        Feasibility / optimality tolerance & $10^{-4},\ 3\times10^{-3}$ \\
        \bottomrule
    \end{tabular*}
\end{wraptable}
C-ADMM replaces each fleet-wide QP with smaller local QPs solved in parallel, but requires many consensus rounds.
For example, for eight UAVs and a $10$~s horizon, the slowest local QP averages $35$~ms per round.
The additional rounds and coordination yield comparable total planning times in this case: $242.6$~s distributed and $238.0$~s centralized.
Thus, the proposed decentralized framework achieves comparable planning time to a centralized solver in this case.

\begin{figure}[htbp]
    \centering
    \includegraphics[width=\textwidth,clip]{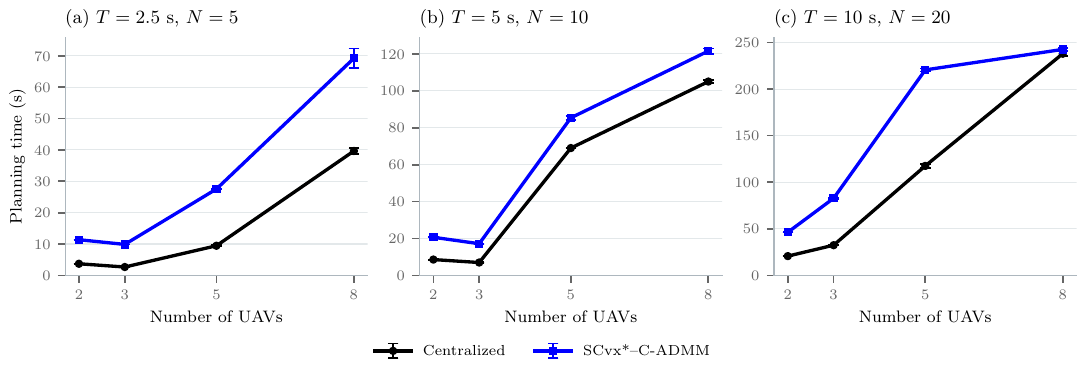}
    \caption{Comparison of centralized and decentralized SCvx*--C-ADMM planning times: mean $\pm$ one standard deviation over three cold solves per configuration.}
    \label{fig:case-timing}
\end{figure}

\color{black}
\Needspace{7\baselineskip}
\begin{rmk}[Planning under a time budget]
The results above use convergence-based termination.
For online execution, each cycle can be warm-started and both iteration counts capped to fit the available time.
For example, \citet{stomberg2026realtime} demonstrate a decentralized SQP--ADMM scheme with one outer SQP step and six inner ADMM iterations per sampling instant.
\end{rmk}

\FloatBarrier
\color{black}
\Needspace{11\baselineskip}
\subsection{Risk-Sensitive UGV Trajectory Optimization}\label{subsec:case-risk}

\begingroup
\setlength{\intextsep}{0pt}
\begin{wraptable}{r}{0.47\textwidth}
    \centering
    \makeatletter
    \long\def\@makecaption#1#2{\vskip6pt{\small\color{black}\raggedright #1: #2\par}\vskip4pt}
    \makeatother
        \caption{\color{black}Target belief used by UGV planner comparisons.}
        \label{tab:case-belief}
        \color{black}\footnotesize
        \setlength{\tabcolsep}{2pt}
        \renewcommand{\arraystretch}{1.08}
        \begin{tabular*}{\linewidth}{@{\extracolsep{\fill}}lr@{}}
            \toprule
            Parameter & Value \\
            \midrule
            \multicolumn{2}{@{}l}{\textit{One Bernoulli track}} \\
            Track existence probability, $r_\ell$ & $0.95$ \\
            Track mean, $\mu_\ell$ & $(0.54,0.48)$ \\
            Principal std., $\sigma_1,\sigma_2$ & $0.10,\ 0.03$ \\
            Ellipse orientation & $135^\circ$ \\
            Confidence level, $1-\delta_\ell$ & $95\%$ \\
            \bottomrule
        \end{tabular*}
\end{wraptable}
For the UGV trajectory optimization, we compare three ground planners between the endpoints in Table~\ref{tab:case-initial}.
The planners share the fixed Gaussian Bernoulli belief in Table~\ref{tab:case-belief}.
Here $\sigma_1,\sigma_2$ describe the major and minor axes, the orientation is measured counterclockwise from the positive horizontal axis.
\par
\makeatletter
\ifnum\c@WF@wrappedlines>1
  \vspace{\dimexpr\c@WF@wrappedlines\baselineskip-\baselineskip\relax}
\fi
\makeatother
\WFclear
\endgroup

The undetected-target PPP intensity is $0.70$ in the left half of the workspace ($\xi_1\leq0.5$) and $0.03$ in the right half, in expected targets per unit area.
Figure~\ref{fig:case-risk} shows the resulting optimal position trajectories of the UGV for various problem configurations.
The first planner avoids the track mean, the second avoids its confidence set, and the third retains that constraint while adding PMB risk (i.e., the void probability) to the objective.
All three selected solutions converge and satisfy the terminal constraints.

\begin{figure}[htbp]
    \centering
    \includegraphics[width=\textwidth,clip]{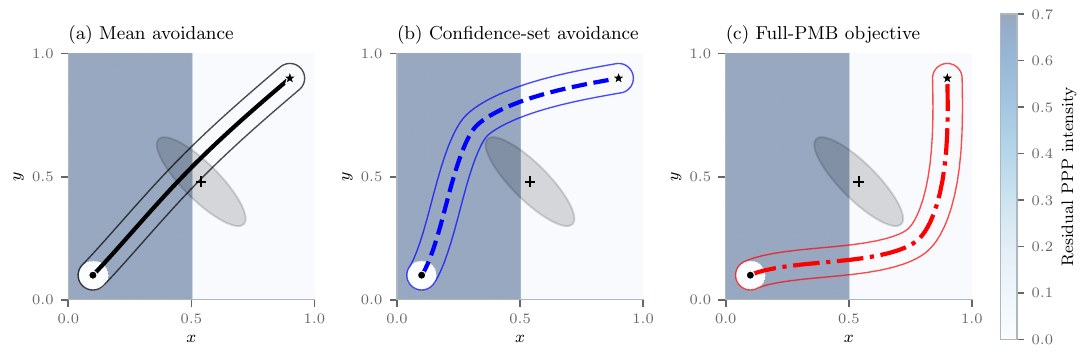}
    \caption{\color{black}UGV planner comparison.
    The ellipse shows track uncertainty; thin outlines show swept corridors, and circles/stars mark start/goal.}
    \label{fig:case-risk}
\end{figure}

\color{black}
From Figure~\ref{fig:case-risk}(a), we see that naive mean avoidance passes through appreciable track probability mass.
Confidence-set avoidance in Figure~\ref{fig:case-risk}(b) takes the upper detour but remains longer in the region of high undetected-target intensity, which is undesirable.
The full-PMB objective used in Figure~\ref{fig:case-risk}(c) instead selects the lower detour and enters the lower-intensity region earlier.
Table~\ref{tab:case-risk} quantifies the resulting risk--effort tradeoff, namely, optimizing the fully risk-aware cost yields the lowest threat probability, at the expense of slightly more control effort.

\begin{table}[htbp]
    \centering
    \color{black}
    \caption{\color{black}Risk and control effort for the routes in Figure~\ref{fig:case-risk}.}
    \label{tab:case-risk}
    \color{black}
    \begin{tabular}{lrr}
        \toprule
        Planner & PMB corridor probability & Control effort $E_u$ \\
        \midrule
        Mean avoidance & $37.99\%$ & $0.0214$ \\
        Confidence-set avoidance & $7.27\%$ & $0.4516$ \\
        Full-PMB objective & $3.64\%$ & $0.7366$ \\
        \bottomrule
    \end{tabular}
\end{table}

\FloatBarrier
\color{black}
\subsection{Closed-Loop Collaboration with One UGV}\label{subsec:case-feedback}

We now combine multi-UAV trajectory optimization, PMB inference, and risk-aware UGV planning in a closed-loop mission.
Figure~\ref{fig:case-base-density} shows the UAV base density updates, and Figure~\ref{fig:case-mission} shows the PMB intensity and ground execution during the closed-loop mission.
Initially, the UAVs search along the straight ground route while the UGV waits.
Every initial plan, ground replan, and UAV horizon refresh uses the three-component mixture in \eqref{eq:updated_base_density}, evaluated from the current posterior and remaining UGV route.
The posterior mixture changes as sensing reduces uncertainty along the remaining ground corridor.

We use equal nominal weights for route coverage, undetected-target search, and track information, renormalized over the nonzero components.
Thus, the weights are $(1/3,1/3,1/3)$ when all three components are active, and $(1/2,1/2,0)$ when the track-information component is absent.
The Bayesian track-information maps in \eqref{eq:bayesian_track_eid} use a point-mass nuisance-state distribution concentrated at altitude $0.45$ and zero attitude angles.
With the uniform initial PPP and no tracks, $U_0=\Phi_0$, so the initial mixture is the Gaussian tube in Figure~\ref{fig:case-base-density}(a).
At $3.0$~s, a confirmed target triggers a ground replan and a new posterior-dependent UAV plan (Figure~\ref{fig:case-base-density}(b)).
Ground planning minimizes $E_u+w_{H}\scriptH$ subject to track-avoidance constraints, using the current PPP intensity and Bernoulli tracks.

Finally, the PMB motion gate governs permission for the UGV to move and the UGV is released when the corridor risk bound falls below the motion-gate budget $\epsilon_{\mathrm{go}}=0.05$, in accordance with the motion gate in \eqref{eq:fleet-motion-gate}.
Here, we observe that continued sensing opens the gate at $31.0$~s, after which the UGV traverses the route and reaches its stopped goal at $41.0$~s.
The mission uses five UAV planning episodes and one ground replan.
Reported times are simulated execution times and exclude optimization wall time.
Figure~\ref{fig:case-mission}(a) shows the posterior and vehicle configurations at the principal mission events.
Panel (b) shows how the undetected-target and track contributions determine the motion gate, while panel (c) shows the UGV waiting and then progressing to its goal.
The first track-risk spike, at $1.4$~s, follows a clutter-generated tentative track, whereas the spike at $2.2$~s follows a true-target detection.
The route update at $3$~s produces the sharp risk reduction.
The later gradual increase reflects growing track uncertainty between informative observations; renewed target observations and the rejection of an unsupported tentative track subsequently reduce the corridor risk without triggering another ground replan.

\begin{figure}[htbp]
    \centering
    \includegraphics[width=\textwidth,clip]{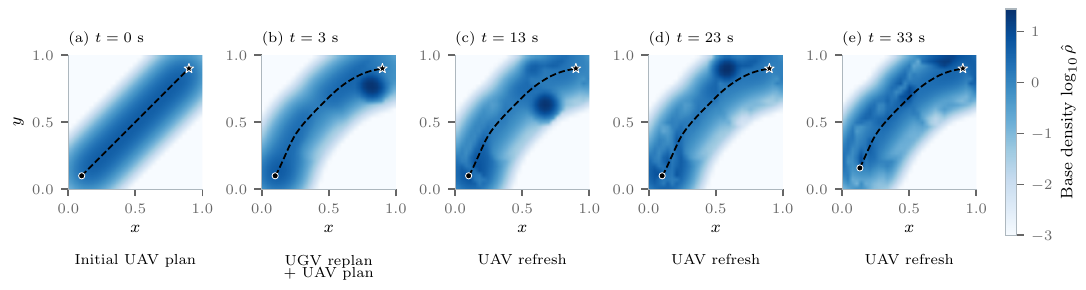}
    \caption{UAV base density updates. Dashed lines show the remaining ground routes.}
    \label{fig:case-base-density}
\end{figure}

\Needspace{25\baselineskip}
\begin{figure}[htbp]
    \centering
    \includegraphics[width=\textwidth,clip]{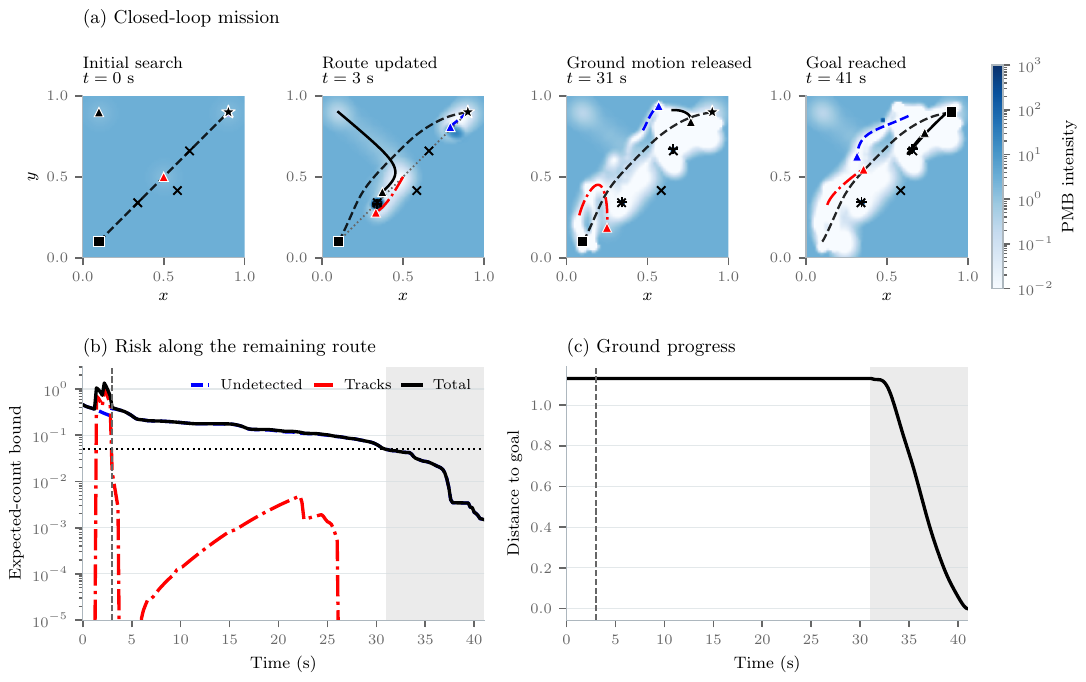}
    \caption{Closed-loop mission execution: (a) PMB intensity and vehicle configurations; (b) corridor risk bounds; (c) UGV progress.
    Crosses mark targets; the vertical dashed line marks replanning, and shading marks authorized motion.}
    \label{fig:case-mission}
\end{figure}
\FloatBarrier

\color{black}
\subsection{Full-System Collaboration with Two UGVs}\label{subsec:case-full-system}
Finally, we consider three UAVs supporting two UGVs whose initial routes intersect.
This example combines posterior-dependent aerial sensing with coupled ground trajectory optimization, inter-UGV separation constraints, and a joint motion gate evaluated over the union of the remaining ground corridors.
The UGVs start at $(0.1,0.1)$ and $(0.1,0.9)$ and have goals $(0.9,0.9)$ and $(0.9,0.1)$, respectively.

\color{black}
\begin{figure}[htbp]
    \centering
    \includegraphics[width=\textwidth,clip]{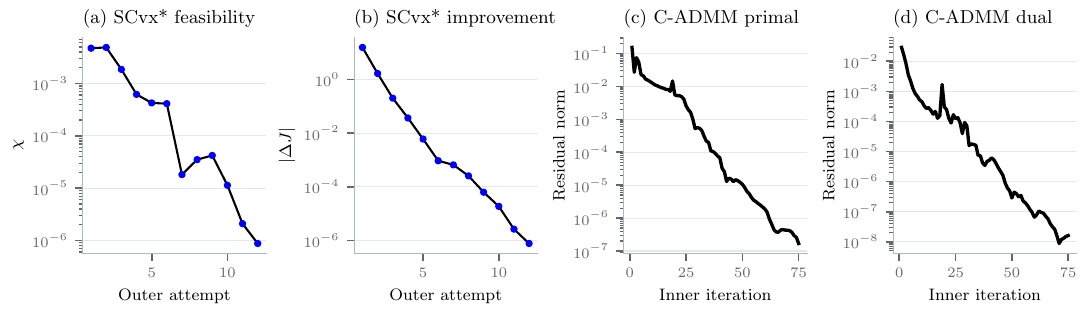}
    \caption{Two-UGV SCvx*--C-ADMM solver convergence. Inner residuals correspond to the first outer subproblem and terminate at convergence.}
    \label{fig:case-ugv-convergence}
\end{figure}
\FloatBarrier
\begingroup
\setlength{\intextsep}{4pt plus 1pt minus 1pt}
\emergencystretch=1em
\Needspace{20\baselineskip}
\begin{wrapfigure}{r}{0.56\textwidth}
    \centering
    \includegraphics[width=\linewidth,clip]{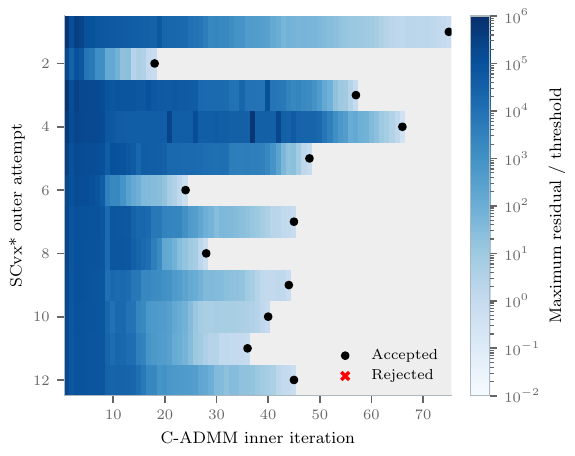}
    \makeatletter
    \long\def\@makecaption#1#2{\vskip6pt{\small\raggedright\noindent #1: #2\par}\vskip4pt}
    \makeatother
    \caption{\color{caseeditcolor}Nested SCvx*--C-ADMM convergence for the UGV fleet problem.
    Colors show the largest residual-to-tolerance ratio; row-end markers show inner termination and outer acceptance.
    }
    \label{fig:case-ugv-nested}
\end{wrapfigure}
First, we illustrate two-UGV SCvx*--C-ADMM convergence using the control-effort objective ($w_H=0$).
The solver converges in 12 outer iterations and Figure~\ref{fig:case-ugv-convergence} shows the SCvx* feasibility and merit improvement, together with the standard C-ADMM residuals for the first convex subproblem.
Figure~\ref{fig:case-ugv-nested} shows the nested iteration history with the same stopping tolerances.

\noindent\caseedit{The ground fleet replans at $2$, $22$, $42$~s as new route-relevant tracks are confirmed.}
\caseedit{The joint gate releases both UGVs at $57.6$~s, and both reach their stopped goals at $67.6$~s (Figure~\ref{fig:case-two-ugv-mission}).}
\caseedit{The mission uses $8$ UAV planning episodes and $3$ joint ground replans.}
\caseedit{The minimum inter-UGV separation during execution is $0.0995$, above the required $0.08$.}

\par
\makeatletter
\ifnum\c@WF@wrappedlines>1
  \vspace{\dimexpr\c@WF@wrappedlines\baselineskip\relax}
\fi
\makeatother
\WFclear

\begin{figure}[htbp]
    \centering
    \includegraphics[width=\textwidth,clip]{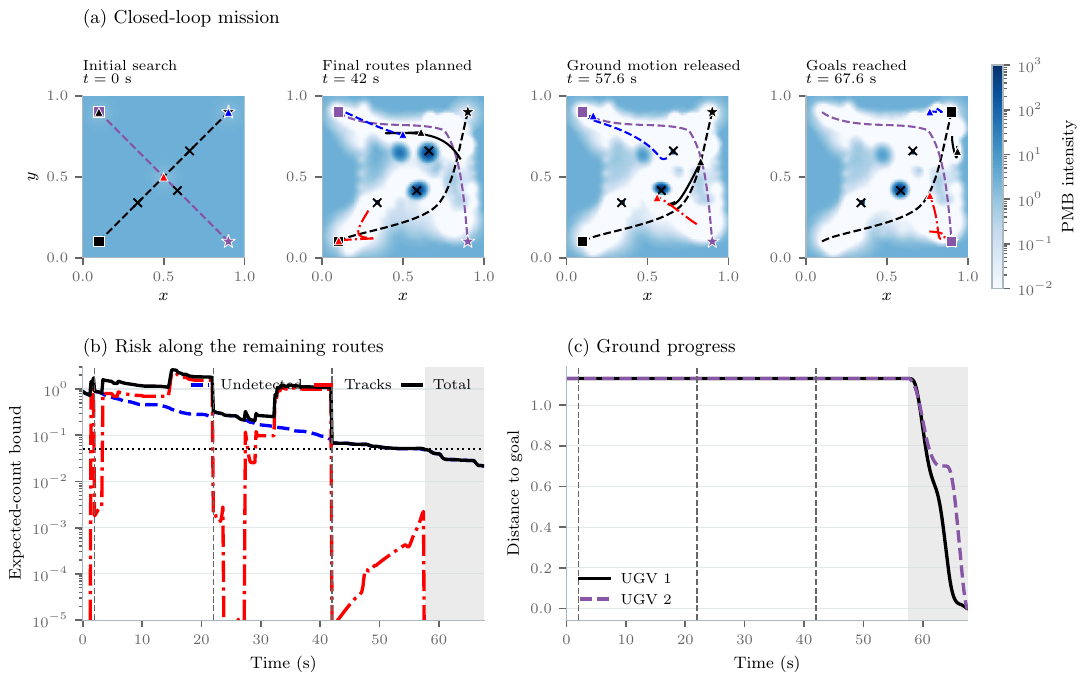}
    \caption{\color{caseeditcolor}Closed-loop collaboration with two UGVs: (a) PMB intensity and vehicle configurations; (b) risk along the union of the remaining ground corridors; (c) ground progress.
\caseedit{Vertical dashed lines mark ground replans, and shading marks authorized motion.}}
    \label{fig:case-two-ugv-mission}
\end{figure}
\par\endgroup

\FloatBarrier
\endgroup

\section{Conclusion}~\label{sec:conclusion}
\begingroup\color{red}
\hypersetup{citecolor=red,linkcolor=red,urlcolor=red}
In this work, we have developed a framework for information-driven collaboration between UAV and UGV fleets in a domain containing an unknown number of targets.
We formulated aerial search using an ergodic objective whose occupation measure accounts for the ground area observed by pose-dependent camera footprints.
Bayesian inference from angles-only measurements provides a PMB belief over detected and undetected targets, linking aerial search priorities to ground-route risk.
The resulting coordination scheme updates both fleets' plans as information becomes available and permits ground motion once the remaining route risk satisfies a prescribed threshold.

We developed a nested SCvx*--C-ADMM method for the nonconvex fleet trajectory problems, accounting for nonlinear dynamics, nonseparable objectives, and continuous-time inter-agent separation.
For the UAV problem, the graph-flow formulation replaces the global coefficient-sum constraint with local balance equations and neighbor agreement.
Its equivalence at consensus to the centralized convex subproblem allows the coupled ergodic objective to be optimized without a central coefficient aggregator.

The numerical studies demonstrate how the UAVs distribute camera coverage along the ground routes and how accounting for undetected threats changes the UGV risk--effort tradeoff.
The solver comparisons assess convergence, solution quality, and computational scaling relative to centralized optimization.
Closed-loop missions with one and two UGVs further demonstrate the interaction between aerial sensing, target confirmation, ground replanning, and motion authorization as the ground vehicles progress to their goals.

Future work will consider UGVs with onboard sensing and decentralized RFS fusion, allowing both fleets to contribute observations through local belief updates and inter-agent communication.
Further investigation of the nested optimization structure may improve convergence and reduce the computation and communication required by the consensus iterations.
Lastly, control-barrier-function safety filters adapted to the relative degree of the vehicle dynamics could address execution errors and disturbances~\citep{ames2017cbf,xiao2022hocbf}.
\endgroup

\iffalse
\begin{itemize}
    \item Extensions to scenarios where the UGVs themselves have on-board sensing capabilities, which leads to heterogeneous UAV/UGV fusion
    \item Distributed/decentralized RFS sensor fusion
    \item A more thorough investigation into non-convex distributed optimization methods, and exploiting the rich structure of SCP-ADMM
    \item \textcolor{red}{A control-barrier-function safety filter for execution errors and disturbances, with a relative-degree-appropriate construction for the UAV and UGV dynamics~\citep{ames2017cbf,xiao2022hocbf}.}
\end{itemize}
\fi

\newpage

\appendix
\section{Proof of Proposition~\ref{prop:FoV-ellipse}}
\label{app:A}

A point $s = [\xi^\intercal, 0]^\intercal\in\R^{3}$ lies in the sensor FoV cone exactly when
\begin{equation}\label{eq:cone-membership}
    o^\intercal (s - r) > 0, \quad \frac{o^\intercal (s - r)}{\|s - r\|} \geq \cos\alpha.
\end{equation}
Letting $d = \xi - p$ be the planar relative position vector, the boundary of the cone-plane intersection satisfies
\begin{equation}\label{eq:cone-plane-boundary}
    (b^\intercal d - o_z h)^2 = \cos^2 \alpha (\|d\|^2 + h^2).
\end{equation}
Together with the forward-cone condition in \eqref{eq:cone-membership}, the downward-facing intersection whose boundary satisfies \eqref{eq:cone-plane-boundary} is a bounded ellipse precisely when $D=o_z^2-\sin^2\alpha>0$ and $o_z<-\sin\alpha$. At the limiting and opposite cases the plane-cone intersection may instead be parabolic, hyperbolic, or lie on the opposite cone nappe.
For $b\neq0$, use the principal directions $(u,v)$, write $\zeta=u^\intercal d$ and $\eta=v^\intercal d$, expand \eqref{eq:cone-plane-boundary}, and complete the square to obtain
\begin{equation}
    D(\zeta - \zeta_c)^2 + \cos^2\alpha\eta^2 = \frac{h^2\sin^2\alpha\cos^2\alpha}{D}, \quad \zeta_c = -\frac{o_z h \|b\|}{D}.
\end{equation}
Consequently, transforming back to the original variable $\xi$ and collecting terms yields the desired result. If $b=0$, \eqref{eq:cone-plane-boundary} reduces directly to the stated circular nadir footprint.

\section{Void Probability Linearization}
\label{app:void}

To compute the derivative of the negative log void probability, we need to compute the variation in the \textit{integration domain} $\scriptC_{\rho}(p)$. We assume locally that the swept corridor is compactly contained in the interior of $\Omega$, has a sufficiently regular boundary, has a unique closest path point almost everywhere on that boundary, and undergoes no topology change under the perturbation; under these assumptions the following is a local shape derivative. If the corridor meets $\partial\Omega$, the fixed part of the clipped boundary must instead be excluded from the moving-boundary integral. We fix one agent throughout this appendix and suppress its index.
Let $\hat{P}$ be the current SCvx reference and perturb it in the direction $\Delta P$, i.e., $P^{\varepsilon} = \hat{P} + \varepsilon \Delta P$, which implies the position variation $\delta p(t) = \sum_{n=0}^{N} B_n(t) \Delta p_n$.
Denote the reference corridor by $\hat{\scriptC} = \scriptC_{\rho}(\hat{P})$.
Define the generic spatial integral over the corridor $\scriptC_{\rho}(P)$ with respect to a sufficiently regular function $f: \Omega \rightarrow \R_{+}$ as
\begin{equation*}
    \scriptI_f(P) \triangleq \int_{\scriptC_{\rho}(P)} f(\xi)\, \d\xi,
\end{equation*}
which implies $\scriptH(P) = \scriptI_{\lambda}(P) - \sum_j \log\big(1 - r_{j} \scriptI_{\rho_j}(P)\big)$.
Thus, it suffices to compute $\nabla_{p_n} \scriptI_f(\hat{P})$ to perform the objective linearization.
Consider a point $\xi \in \partial\hat{\scriptC}$ and the closest point on the corresponding path $p\big(\tau^\star(\xi)\big)$, where $\tau^\star(\xi) \triangleq \mathrm{argmin}_{t\in [0,T]} \, \|p(t) - \xi\|$ is the corresponding time.
It follows that $\|\xi - p(\tau^\star(\xi))\| = \rho$ and that the normal vector at $\xi$ is $\nu(\xi) = \big(\xi - p(\tau^\star(\xi))\big) / \rho$.
Next, consider the corridor $\scriptC_{\rho}^{\varepsilon}$ induced by the perturbed path $P^\varepsilon$.
Similarly, take a point on the boundary $\xi^\varepsilon$ with normal vector $\nu^\varepsilon(\xi^\varepsilon)$.
The \textit{velocity field} on the boundary $\partial\hat{\scriptC}$ is then given by
\begin{equation}\label{eq:tube-boundary-velocity}
    V(\xi)^\intercal\nu(\xi)=\delta p\big(\tau^\star(\xi)\big)^\intercal\nu(\xi),
\end{equation}
See Figure~\ref{fig:shape-calc} for an illustration of the various deformations and vectors used throughout the analysis.
Suppose that the perturbation of the path induces a boundary velocity field $V(\xi)$ on $\partial\hat{\scriptC}$.
It follows that
\begin{align*}
    D\scriptI_{f}(\hat{P})[\Delta P] \triangleq \frac{\d}{\d\varepsilon} \scriptI_{f}(P^{\varepsilon})\bigg|_{\varepsilon=0} = \frac{\d}{\d\varepsilon}\int_{\scriptC_{\rho}(P^{\varepsilon})} f(\xi)\, \d\xi\bigg|_{\varepsilon=0} &= \int_{\partial\hat{\scriptC}} f(\xi) V^\intercal (\xi) \nu(\xi)\, \d s \\
    &= \int_{\partial\hat{\scriptC}} f(\xi) \delta p^\intercal\big(\tau^\star(\xi)\big) \nu(\xi)\, \d s,
\end{align*}
where the second equality follows from the Hadamard formula~\citep{hiptmair2013shape}, and the third follows from \eqref{eq:tube-boundary-velocity}.
\begin{figure}[!htb]
    \centering
    \includegraphics[width=\linewidth]{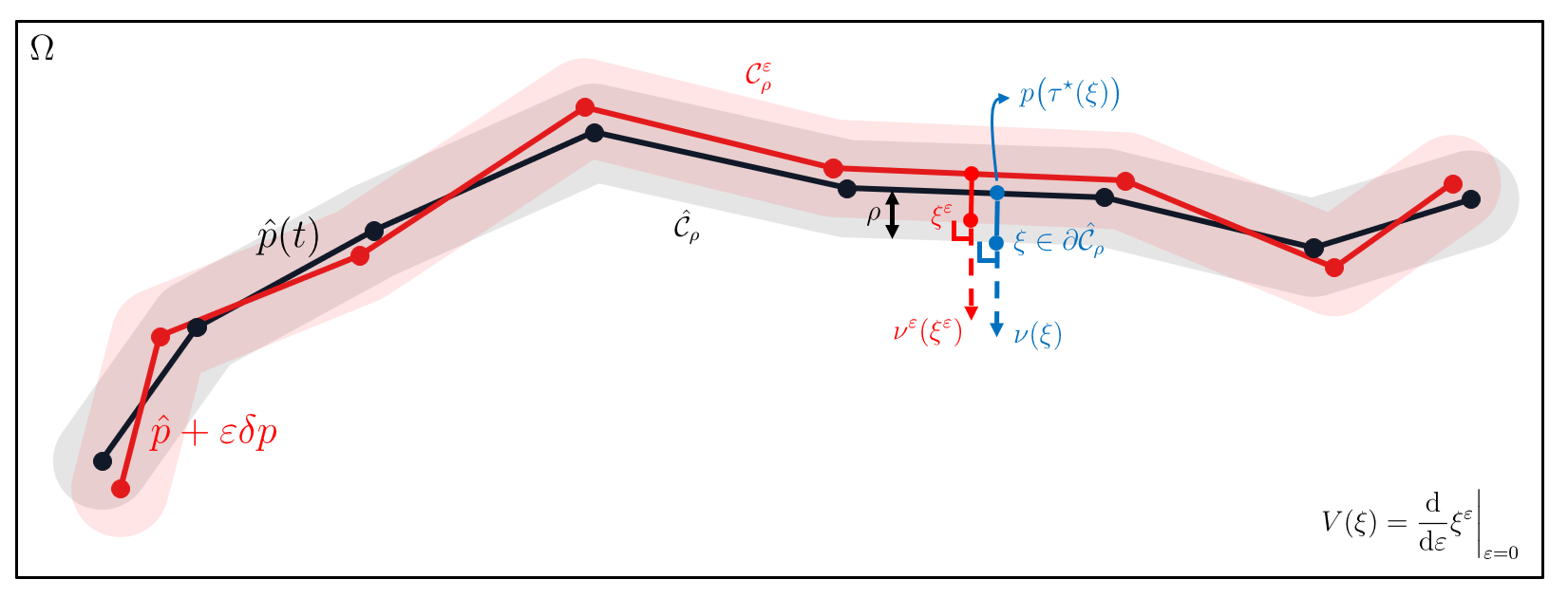}
    \caption{Illustration of various geometric objects, properties, and deformation of path-centered corridor in shape calculus.}
    \label{fig:shape-calc}
\end{figure}
Substituting the piecewise-linear path expression into the shape derivative yields
\begin{align*}
    D\scriptI_{f}(\hat{P})[\Delta P] &= \int_{\partial\hat{\scriptC}} f(\xi) \left[ \sum_{n=0}^{N} B_n\big(\tau^\star(\xi)\big) \Delta p_n \right]^\intercal \nu(\xi)\, \d s \\
    &= \sum_{n=0}^{N} \Delta p_n^\intercal \int_{\partial\hat{\scriptC}} f(\xi) B_n \big(\tau^\star(\xi)\big) \nu(\xi)\, \d s.
\end{align*}
Define the integral
\begin{equation}
    \sigma_{f,n}(\hat{P}) \triangleq \int_{\partial\hat{\scriptC}} f(\xi) B_n\big(\tau^\star(\xi)\big) \nu(\xi)\, \d s \in \R^2,
\end{equation}
from which it follows that $\nabla_{p_n}\scriptI_f(\hat{P}) = \sigma_{f,n}(\hat{P})$.
Consequently, provided $1-r_j\scriptI_{\rho_j}(\hat P)>0$ for every $j$, the void-probability linearization becomes
\begin{align}
    \scriptH(P) \approx \scriptH(\hat P)+\sum_{n=0}^{N}\left[\sigma_{\lambda,n}(\hat P)+\sum_{j=1}^{L}\frac{r_j\,\sigma_{\rho_j,n}(\hat P)}{1-r_j\scriptI_{\rho_j}(\hat P)}\right]^\intercal(p_n-\hat p_n).
\end{align}

\section{Distance Function Linearization}
\label{app:distance}
Let $d_{\scriptA}(p) \triangleq \mathrm{dist}(p,\scriptA)$ denote the minimum distance from point $p$ to the nonempty, closed, convex set $\scriptA \subseteq\Omega$.
Let $\hat{p} \notin\scriptA$ be the reference point and compute the projection
\begin{equation}
    \hat{q} \triangleq \Pi_{\scriptA}(\hat p) = \underset{q \in \scriptA}{\operatorname{argmin}}\, \|\hat p - q\|.
\end{equation}
Then it follows that $d_{\scriptA}(\hat p) = \|\hat p - \hat q\|$.
Define the outward unit vector $\hat{n} = (\hat p - \hat q) / \|\hat p - \hat q\|$.
The distance function then has the gradient~\citep{bauschke2017convex}
\begin{equation}
    \nabla d_{\scriptA}(p) = \frac{p - \Pi_{\scriptA}(p)}{\|p - \Pi_{\scriptA}(p)\|}, \quad p \notin \scriptA.
\end{equation}
Since $d_{\scriptA}$ is a convex function, we have the global lower bound
\begin{equation}
    d_{\scriptA}(p) \geq d_{\scriptA}(\hat p) + \nabla d_{\scriptA}^\intercal(\hat p) (p - \hat{p}),
\end{equation}
thus the linearized constraint is a conservative inner approximation of the feasible set, that is,
\begin{equation*}
    \left\{p : d_{\scriptA}(\hat p) + \hat{n}^\intercal (p - \hat p) \geq \rho \right\} \subseteq \left\{ p : d_{\scriptA}(p) \geq \rho \right\}.
\end{equation*}

\begingroup\color{red}
\hypersetup{citecolor=red,linkcolor=red,urlcolor=red}
\section{Flow Sensitivities and Shooting Linearization}\label{app:flow-sensitivities}
For a reference shooting state $\hat z_n^a$ and held control $\hat u_n$, let $\hat\zeta_n(t)$ solve
\begin{equation*}
    \dot{\hat\zeta}_n(t)=f_{\mathrm{aug}}(\hat\zeta_n(t),\hat u_n),\qquad
    \hat\zeta_n(t_n)=\hat z_n^a,\qquad t\in[t_n,t_{n+1}].
\end{equation*}
Where $f_{\mathrm{aug}}$ is continuously differentiable along this rollout, define
\begin{equation*}
    \mathsf A_n(t)=D_{z^a}f_{\mathrm{aug}}(\hat\zeta_n(t),\hat u_n),\qquad
    \mathsf B_n(t)=D_u f_{\mathrm{aug}}(\hat\zeta_n(t),\hat u_n).
\end{equation*}
The sensitivities to the initial state and held control satisfy the variational equations
\begin{equation}\label{eq:flow-variational-odes}
    \begin{aligned}
        \dot\Phi_{z,n}(t)&=\mathsf A_n(t)\Phi_{z,n}(t),
        &\Phi_{z,n}(t_n)&=I_{d_a},\\
        \dot\Phi_{u,n}(t)&=\mathsf A_n(t)\Phi_{u,n}(t)+\mathsf B_n(t),
        &\Phi_{u,n}(t_n)&=0_{d_a\times m_a},
    \end{aligned}
\end{equation}
where $d_a=\dim z^a$ and $m_a=\dim u_n$.
Their endpoint values give the flow Jacobians
\begin{equation*}
    A_n=\Phi_{z,n}(t_{n+1})=D_{z^a}F_n^a(\hat z_n^a,\hat u_n),\qquad
    B_n=\Phi_{u,n}(t_{n+1})=D_u F_n^a(\hat z_n^a,\hat u_n).
\end{equation*}
Thus, with $c_n=F_n^a(\hat z_n^a,\hat u_n)-A_n\hat z_n^a-B_n\hat u_n$, the linearized shooting equation is
\begin{equation}\label{eq:ctcs-affine-dynamics}
    z_{n+1}^a=A_nz_n^a+B_nu_n+c_n.
\end{equation}
We integrate the reference trajectory and variational equations together to obtain numerical approximations of the flow and its Jacobians.
For the symmetric accumulator representation, the shared pair defect uses $\bar z_n^a$, formed from the agent states and the averages of the endpoint copies.
Its derivatives follow by composing the flow Jacobians with this linear averaging map.
The copy-agreement constraints are already affine.
For the UGV formulation, the same equations apply with $z_n^a$, $F_n^a$, and $f_{\mathrm{aug}}$ replaced by $z_n^g$, $F_n^g$, and $f_{\mathrm{aug}}^g$.
The augmented state and held-control dimensions are then $d_g=4N_g+|\scriptE|$ and $2N_g$, respectively.
\endgroup

\vskip 0.2in
\begingroup
\color{black}
\hypersetup{citecolor=black,linkcolor=black,urlcolor=black}
\bibliography{refs}

@inproceedings{ayvali2017constrained,
  author    = {Elif Ayvali and Hadi Salman and Howie Choset},
  title     = {Ergodic Coverage in Constrained Environments Using Stochastic Trajectory Optimization},
  booktitle = {Proceedings of the IEEE/RSJ International Conference on Intelligent Robots and Systems},
  year      = {2017},
  pages     = {5204--5210},
  doi       = {10.1109/IROS.2017.8206410}
}

@inproceedings{dewan2013coordinated,
  author    = {Ayush Dewan and Aravindh Mahendran and Nikhil Soni and K. Madhava Krishna},
  title     = {Optimization Based Coordinated {UGV-MAV} Exploration for {2D} Augmented Mapping},
  booktitle = {Proceedings of the International Conference on Autonomous Agents and Multiagent Systems},
  year      = {2013},
  pages     = {1125--1126},
  note      = {Extended abstract},
  url       = {https://aamas.csc.liv.ac.uk/Proceedings/aamas2013/docs/p1125.pdf}
}

@inproceedings{butzke2016collaborative,
  author    = {Jonathan Butzke and Kalin Gochev and Benjamin Holden and Eui-jung Jung and Maxim Likhachev},
  title     = {Planning for a Ground-Air Robotic System with Collaborative Localization},
  booktitle = {Proceedings of the IEEE International Conference on Robotics and Automation},
  year      = {2016},
  pages     = {284--291},
  doi       = {10.1109/ICRA.2016.7487146}
}

@article{naveed2025adaptive,
  author  = {Kaleb Ben Naveed and Devansh R. Agrawal and Rahul Kumar and Dimitra Panagou},
  title   = {Adaptive Ergodic Search with Energy-Aware Scheduling for Persistent Multi-Robot Missions},
  journal = {Autonomous Robots},
  year    = {2025},
  volume  = {49},
  number  = {4},
  pages   = {27},
  doi     = {10.1007/s10514-025-10215-6}
}

@inproceedings{hughes2025ergodic,
  author    = {Christian Hughes and Houston Warren and Darrick Lee and Fabio Ramos and Ian Abraham},
  title     = {Ergodic Trajectory Optimization on Generalized Domains Using Maximum Mean Discrepancy},
  booktitle = {Proceedings of the IEEE International Conference on Robotics and Automation},
  year      = {2025},
  pages     = {1--7},
  doi       = {10.1109/ICRA55743.2025.11127573}
}

@article{nguyen2020multiobjective,
  author  = {Hoa Van Nguyen and Hamid Rezatofighi and Ba-Ngu Vo and Damith C. Ranasinghe},
  title   = {Multi-Objective Multi-Agent Planning for Jointly Discovering and Tracking Mobile Objects},
  journal = {Proceedings of the AAAI Conference on Artificial Intelligence},
  year    = {2020},
  volume  = {34},
  number  = {5},
  pages   = {7227--7235},
  doi     = {10.1609/aaai.v34i05.6213}
}

@article{beard2017void,
  author  = {Michael Beard and Ba-Tuong Vo and Ba-Ngu Vo and Sanjeev Arulampalam},
  title   = {Void Probabilities and {Cauchy--Schwarz} Divergence for Generalized Labeled Multi-{Bernoulli} Models},
  journal = {IEEE Transactions on Signal Processing},
  year    = {2017},
  volume  = {65},
  number  = {19},
  pages   = {5047--5061},
  doi     = {10.1109/TSP.2017.2723355}
}

@misc{han2026risk,
  author       = {Shaohang Han and Gang Chen and Yixi Cai and Ignacio Torroba and Ivan Stenius and Patric Jensfelt and Javier Alonso-Mora and Jana Tumova},
  title        = {Risk-Aware Belief Control Barrier Functions over Random Finite Sets},
  year         = {2026},
  howpublished = {arXiv:2607.15016},
  url          = {https://arxiv.org/abs/2607.15016}
}

@inproceedings{rey2018decentralized,
  author    = {Felix Rey and Zhoudan Pan and Adrian Hauswirth and John Lygeros},
  title     = {Fully Decentralized {ADMM} for Coordination and Collision Avoidance},
  booktitle = {Proceedings of the European Control Conference},
  year      = {2018},
  pages     = {825--830},
  doi       = {10.23919/ECC.2018.8550245}
}

@article{stomberg2026realtime,
  author  = {G{\"o}sta Stomberg and Alexander Engelmann and Moritz Diehl and Timm Faulwasser},
  title   = {Decentralized Real-Time Iterations for Distributed {NMPC}},
  journal = {IEEE Transactions on Automatic Control},
  year    = {2026},
  volume  = {71},
  number  = {3},
  pages   = {1584--1599},
  doi     = {10.1109/TAC.2025.3622000}
}

@article{mathew2011metrics,
  author  = {George Mathew and Igor Mezi{\'c}},
  title   = {Metrics for Ergodicity and Design of Ergodic Dynamics for Multi-Agent Systems},
  journal = {Physica D: Nonlinear Phenomena},
  year    = {2011},
  volume  = {240},
  number  = {4--5},
  pages   = {432--442},
  doi     = {10.1016/j.physd.2010.10.010}
}

@inproceedings{miller2013trajectory,
  author    = {Lauren M. Miller and Todd D. Murphey},
  title     = {Trajectory Optimization for Continuous Ergodic Exploration},
  booktitle = {Proceedings of the American Control Conference},
  year      = {2013},
  pages     = {4196--4201},
  doi       = {10.1109/ACC.2013.6580484}
}

@article{abraham2018measurement,
  author  = {Ian Abraham and Todd D. Murphey},
  title   = {Decentralized Ergodic Control: Distribution-Driven Sensing and Exploration for Multiagent Systems},
  journal = {IEEE Robotics and Automation Letters},
  year    = {2018},
  volume  = {3},
  number  = {4},
  pages   = {2987--2994},
  doi     = {10.1109/LRA.2018.2849588}
}

@article{miller2016information,
  author  = {Lauren M. Miller and Yonatan Silverman and Malcolm A. MacIver and Todd D. Murphey},
  title   = {Ergodic Exploration of Distributed Information},
  journal = {IEEE Transactions on Robotics},
  year    = {2016},
  volume  = {32},
  number  = {1},
  pages   = {36--52},
  doi     = {10.1109/TRO.2015.2500441}
}

@article{wilson2014trajectory,
  author  = {Andrew D. Wilson and Jarvis A. Schultz and Todd D. Murphey},
  title   = {Trajectory Synthesis for {Fisher} Information Maximization},
  journal = {IEEE Transactions on Robotics},
  year    = {2014},
  volume  = {30},
  number  = {6},
  pages   = {1358--1370},
  doi     = {10.1109/TRO.2014.2345918}
}

@book{chavel1984eigenvalues,
  author    = {Isaac Chavel},
  title     = {Eigenvalues in Riemannian Geometry},
  publisher = {Academic Press},
  address   = {Orlando, Florida},
  year      = {1984},
  isbn      = {978-0-12-170640-1}
}

@inproceedings{mao2016successive,
  author    = {Yuanqi Mao and Michael Szmuk and Beh{\c c}et A{\c c}{\i}kme{\c s}e},
  title     = {Successive Convexification of Non-Convex Optimal Control Problems and Its Convergence Properties},
  booktitle = {Proceedings of the IEEE Conference on Decision and Control},
  year      = {2016},
  pages     = {3636--3641},
  doi       = {10.1109/CDC.2016.7798816}
}

@inproceedings{oguri2023scvxstar,
  author    = {Kenshiro Oguri},
  title     = {Successive Convexification with Feasibility Guarantee via {Augmented Lagrangian} for Non-Convex Optimal Control Problems},
  booktitle = {Proceedings of the IEEE Conference on Decision and Control},
  year      = {2023},
  pages     = {3296--3302},
  doi       = {10.1109/CDC49753.2023.10383462}
}

@article{boyd2011admm,
  author  = {Stephen Boyd and Neal Parikh and Eric Chu and Borja Peleato and Jonathan Eckstein},
  title   = {Distributed Optimization and Statistical Learning via the Alternating Direction Method of Multipliers},
  journal = {Foundations and Trends in Machine Learning},
  year    = {2011},
  volume  = {3},
  number  = {1},
  pages   = {1--122},
  doi     = {10.1561/2200000016}
}

@article{wang2010fastmpc,
  author  = {Yang Wang and Stephen Boyd},
  title   = {Fast Model Predictive Control Using Online Optimization},
  journal = {IEEE Transactions on Control Systems Technology},
  year    = {2010},
  volume  = {18},
  number  = {2},
  pages   = {267--278},
  doi     = {10.1109/TCST.2009.2017934}
}

@inproceedings{bock1984multiple,
  author    = {Hans Georg Bock and Karl-Josef Plitt},
  title     = {A Multiple Shooting Algorithm for Direct Solution of Optimal Control Problems},
  booktitle = {Proceedings of the 9th IFAC World Congress},
  year      = {1984},
  pages     = {1603--1608},
  doi       = {10.1016/S1474-6670(17)61205-9}
}

@article{ames2017cbf,
  author  = {Aaron D. Ames and Xiangru Xu and Jessy W. Grizzle and Paulo Tabuada},
  title   = {Control Barrier Function Based Quadratic Programs for Safety Critical Systems},
  journal = {IEEE Transactions on Automatic Control},
  year    = {2017},
  volume  = {62},
  number  = {8},
  pages   = {3861--3876},
  doi     = {10.1109/TAC.2016.2638961}
}

@article{xiao2022hocbf,
  author  = {Wei Xiao and Calin Belta},
  title   = {High-Order Control Barrier Functions},
  journal = {IEEE Transactions on Automatic Control},
  year    = {2022},
  volume  = {67},
  number  = {7},
  pages   = {3655--3662},
  doi     = {10.1109/TAC.2021.3105491}
}

@book{mahler2007statistical,
  author    = {Ronald P. S. Mahler},
  title     = {Statistical Multisource-Multitarget Information Fusion},
  publisher = {Artech House},
  address   = {Norwood, Massachusetts},
  year      = {2007},
  isbn      = {978-1-59693-092-6}
}

@article{garciafernandez2018pmbm,
  author  = {{\'A}ngel F. Garc{\'i}a-Fern{\'a}ndez and Jason L. Williams and Karl Granstr{\"o}m and Lennart Svensson},
  title   = {Poisson Multi-{Bernoulli} Mixture Filter: Direct Derivation and Implementation},
  journal = {IEEE Transactions on Aerospace and Electronic Systems},
  year    = {2018},
  volume  = {54},
  number  = {4},
  pages   = {1883--1901},
  doi     = {10.1109/TAES.2018.2805153}
}

@article{williams2015marginal,
  author  = {Jason L. Williams},
  title   = {Marginal Multi-Bernoulli Filters: {RFS} Derivation of {MHT}, {JIPDA}, and Association-Based {MeMBer}},
  journal = {IEEE Transactions on Aerospace and Electronic Systems},
  year    = {2015},
  volume  = {51},
  number  = {3},
  pages   = {1664--1687},
  doi     = {10.1109/TAES.2015.130550}
}

@book{lastpenrose2017poisson,
  author    = {G{\"u}nter Last and Mathew Penrose},
  title     = {Lectures on the Poisson Process},
  publisher = {Cambridge University Press},
  address   = {Cambridge},
  year      = {2017},
  doi       = {10.1017/9781316104477},
  isbn      = {978-1-107-45843-7}
}

@book{barshalom2001estimation,
  author    = {Yaakov Bar-Shalom and X. Rong Li and Thiagalingam Kirubarajan},
  title     = {Estimation with Applications to Tracking and Navigation: Theory, Algorithms, and Software},
  publisher = {John Wiley \& Sons},
  address   = {New York},
  year      = {2001},
  doi       = {10.1002/0471221279},
  isbn      = {978-0-471-41655-5}
}

@article{chaloner1995bayesian,
  author  = {Kathryn Chaloner and Isabella Verdinelli},
  title   = {Bayesian Experimental Design: A Review},
  journal = {Statistical Science},
  year    = {1995},
  volume  = {10},
  number  = {3},
  pages   = {273--304},
  doi     = {10.1214/ss/1177009939}
}

@article{hiptmair2013shape,
  author  = {Ralf Hiptmair and Jingzhi Li},
  title   = {Shape Derivatives in Differential Forms {I}: An Intrinsic Perspective},
  journal = {Annali di Matematica Pura ed Applicata},
  year    = {2013},
  volume  = {192},
  number  = {6},
  pages   = {1077--1098},
  doi     = {10.1007/s10231-012-0259-9}
}

@book{bauschke2017convex,
  author    = {Heinz H. Bauschke and Patrick L. Combettes},
  title     = {Convex Analysis and Monotone Operator Theory in Hilbert Spaces},
  edition   = {Second},
  publisher = {Springer},
  address   = {Cham},
  year      = {2017},
  doi       = {10.1007/978-3-319-48311-5}
}

@article{blackmore2011chance,
  author  = {Lars Blackmore and Masahiro Ono and Brian C. Williams},
  title   = {Chance-Constrained Optimal Path Planning with Obstacles},
  journal = {IEEE Transactions on Robotics},
  year    = {2011},
  volume  = {27},
  number  = {6},
  pages   = {1080--1094},
  doi     = {10.1109/TRO.2011.2161160}
}

@misc{elango2024ctcs,
  author = {Purnanand Elango and Dayou Luo and Abhinav G. Kamath and Samet Uzun and Taewan Kim and Beh{\c{c}}et A{\c{c}}{\i}kme{\c{s}}e},
  title = {Successive Convexification for Trajectory Optimization with Continuous-Time Constraint Satisfaction},
  year = {2024},
  howpublished = {arXiv:2404.16826},
  url = {https://arxiv.org/abs/2404.16826}
}

@inproceedings{delatorre2016stochastic,
  author = {Gerardo {De La Torre} and Kathrin Fla{\ss}kamp and Ahalya Prabhakar and Todd D. Murphey},
  title = {Ergodic Exploration with Stochastic Sensor Dynamics},
  booktitle = {Proceedings of the American Control Conference},
  year = {2016},
  pages = {2971--2976},
  doi = {10.1109/ACC.2016.7525371}
}

@misc{berger2024range,
  author = {Henry Berger and Ian Abraham},
  title = {{RAnGE}: Reachability Analysis for Guaranteed Ergodicity},
  year = {2024},
  howpublished = {arXiv:2404.03186},
  doi = {10.48550/arXiv.2404.03186}
}

@inproceedings{abraham2018localization,
  author    = {Ian Abraham and Anastasia Mavrommati and Todd D. Murphey},
  title     = {Data-Driven Measurement Models for Active Localization in Sparse Environments},
  booktitle = {Proceedings of Robotics: Science and Systems},
  year      = {2018},
  address   = {Pittsburgh, Pennsylvania},
  doi       = {10.15607/RSS.2018.XIV.045}
}

@article{mavrommati2018realtime,
  author  = {Anastasia Mavrommati and Emmanouil Tzorakoleftherakis and Ian Abraham and Todd D. Murphey},
  title   = {Real-Time Area Coverage and Target Localization Using Receding-Horizon Ergodic Exploration},
  journal = {IEEE Transactions on Robotics},
  year    = {2018},
  volume  = {34},
  number  = {1},
  pages   = {62--80},
  doi     = {10.1109/TRO.2017.2766265}
}

@book{boyd2004convex,
  author    = {Stephen Boyd and Lieven Vandenberghe},
  title     = {Convex Optimization},
  publisher = {Cambridge University Press},
  year      = {2004},
  doi       = {10.1017/CBO9780511804441}
}

@article{wachter2006ipopt,
  author = {W{\"a}chter, Andreas and Biegler, Lorenz T.},
  title = {On the implementation of an interior-point filter line-search algorithm for large-scale nonlinear programming},
  journal = {Mathematical Programming},
  volume = {106},
  number = {1},
  pages = {25--57},
  year = {2006},
  doi = {10.1007/s10107-004-0559-y}
}

@article{schwan2023piqp,
  author = {Roland Schwan and Yuning Jiang and Daniel Kuhn and Colin N. Jones},
  title = {{PIQP}: A Proximal Interior-Point Quadratic Programming Solver},
  journal = {arXiv preprint arXiv:2304.00290},
  year = {2023}
}
\endgroup

\end{document}